\documentclass[11pt]{article}

\usepackage[margin=1.2in]{geometry}
\usepackage{changepage}

\usepackage{color}

\usepackage[utf8]{inputenc} 
\usepackage[T1]{fontenc}    
\usepackage{hyperref}       
\usepackage{url}            
\usepackage{booktabs}       
\usepackage{amsfonts}       
\usepackage{nicefrac}       
\usepackage{microtype}      

\usepackage{multirow}
\usepackage{algorithm}
\usepackage{algorithmic}
\usepackage{wrapfig}

\usepackage{amsthm}
\usepackage{amsmath, amssymb,amsfonts, xspace, verbatim, bm}
\usepackage{alltt}
\usepackage{lmodern}
\usepackage{graphicx} 
\usepackage{subcaption}
\usepackage{xr}

\usepackage{tikz}
\usetikzlibrary{backgrounds}

\usepackage{changepage}

\usepackage{pdfpages}
\usepackage{pgfplots}
\pgfplotsset{compat=1.7}

\DeclareMathAlphabet\mathbfcal{OMS}{cmsy}{b}{n}

\theoremstyle{plain}
\newtheorem{thm}{Theorem}[section]
\newtheorem{lem}[thm]{Lemma}

\newtheorem{thmCond}{Theorem}
\newtheorem{cond}[thmCond]{Condition}
\newtheorem{problem}{Problem}

\theoremstyle{definition}
\newtheorem{thmDef}{Theorem}[section]

\newtheorem{assumption}[thmDef]{Assumption}

\theoremstyle{remark}
\newtheorem{thmrem}{Theorem}[section]
\newtheorem{rem}[thmrem]{Remark}

\newcommand{\bas}[1]{\begin{align*}#1\end{align*}}
\newcommand{\ba}[1]{\begin{align}#1\end{align}}

\newcommand{\trace}{\text{trace}}
\newcommand{\txtred}[1]{}

\newcommand{\var}{\mbox{var}}

\newcommand{\bbR}{\mathbb{R}}
\newcommand{\bzero}{{\bf 0}}
\newcommand{\bone}{{\bf 1}}

\newcommand{\bA}{{\bf A}}
\newcommand{\bB}{{\bf B}}
\newcommand{\bc}{{\bf c}}

\newcommand{\bD}{{\bf D}}
\newcommand{\be}{{\bf e}}
\newcommand{\bE}{{\bf E}}
\newcommand{\bF}{{\bf F}}
\newcommand{\bG}{{\bf G}}
\newcommand{\bH}{{\bf H}}
\newcommand{\bI}{{\bf I}}

\newcommand{\bL}{{\bf L}}
\newcommand{\bM}{{\bf M}}

\newcommand{\bO}{{\bf O}}
\newcommand{\bP}{{\bf P}}

\newcommand{\bR}{{\bf R}}
\newcommand{\bT}{{\bf T}}

\newcommand{\bu}{{\bf u}}
\newcommand{\bU}{{\bf U}}
\newcommand{\bv}{{\bf v}}
\newcommand{\bV}{{\bf V}}

\newcommand{\bw}{{\bf w}}
\newcommand{\bS}{{\bf S}}
\newcommand{\bx}{{\bf x}}
\newcommand{\bX}{{\bf X}}
\newcommand{\by}{{\bf y}}
\newcommand{\bY}{{\bf Y}}
\newcommand{\bz}{{\bf z}}
\newcommand{\bZ}{{\bf Z}}
\newcommand{\bbm}{\bm{m}}

\newcommand{\balpha}{{\boldsymbol \alpha}}
\newcommand{\bbeta}{{\boldsymbol \beta}}
\newcommand{\hbeta}{{\hat{\beta}}}
\newcommand{\hbbeta}{{\hat{\bbeta}}}
\newcommand{\btheta}{{\boldsymbol \theta}}
\newcommand{\bTheta}{{\boldsymbol \Theta}}
\newcommand{\bvarphi}{{\boldsymbol \varphi}}

\newcommand{\bSigma}{{\boldsymbol \Sigma}}
\newcommand{\bGamma}{{\boldsymbol \Gamma}}

\newcommand{\hbGamma}{{\boldsymbol\hat{\bGamma}}}
\newcommand{\hbTheta}{{\boldsymbol\hat{\bTheta}}}

\newcommand{\gikD}{\max(V,D)}
\newcommand{\supp}{the Appendix}
\newcommand{\suppNoThe}{Appendix}

\newcommand{\rowP}{\bz}
\newcommand{\mtxP}{\bZ}
\newcommand{\mtxPp}{\bZ_P}
\newcommand{\trowP}{{\by}}
\newcommand{\tmtxP}{{\bY}}
\newcommand{\tmtxPp}{{\bY}_P}

\newcommand{\rowS}{\hat{\rowP}}
\newcommand{\mtxS}{\hat{\mtxP}}

\newcommand{\trowS}{\hat{\trowP}}
\newcommand{\tmtxS}{\hat{\tmtxP}}
\newcommand{\tmtxSp}{\hat{\tmtxP}_P}

\newcommand{\elecvxP}{m}
\newcommand{\rowcvxP}{{\bf \elecvxP}}
\newcommand{\cvxP}{{\bf M}}
\newcommand{\cvxS}{{\hat{\cvxP}}}

\newcommand{\elecvxY}{\phi}
\newcommand{\rowcvxY}{{\boldsymbol \phi}}

\newcommand{\tmDiagP}{{\bf D}}
\newcommand{\tmDiagS}{{\hat{\tmDiagP}}}

\newcommand{\diag}{\text{diag}}

\newcommand{\bbb}[1]{\left(#1\right)}

\newcommand{\cS}{\mathcal{S}}

\newcommand{\R}{\mathbb{R}}

\newcommand{\rank}{\mathrm{rank}}

\newcommand{\cov}{\mathrm{Cov}}
\newcommand{\conv}{\mathrm{Conv}}

\newcommand{\proj}{\mathrm{Proj}}

\newcommand{\uE}{\mathrm{E}}

\newcommand{\uP}{\mathrm{P}}

\newcommand{\w}{\bA}
\newcommand{\W}{\mathbfcal{A}}

\newcommand{\hv}{\hat{\bf V}\xspace}
\newcommand{\hvv}{\hat{\bf v}\xspace}
\newcommand{\hww}{\hat{\bf w}\xspace}
\newcommand{\hb}{\hat{b}\xspace}
\newcommand{\hu}{\hat{\bf U}\xspace}
\newcommand{\hhu}{\hat{\bf u}\xspace}

\newcommand{\hV}{\hat{\bf V}\xspace}

\newcommand{\hE}{\hat{\bf E}\xspace}

\newcommand{\hT}{\hat{\bf T}\xspace}

\newcommand{\bN}{{\bf N}}

\newcommand{\bk}{\color{black}}
\newcommand{\bal}{\nu\xspace}
\newcommand{\bpi}{{\bf \Pi}}

\newcommand{\mypsi}{}

\newcommand{\ocsvm}{{One-class SVM}\xspace}

\newcommand{\tablefontsize}{\footnotesize}
\newcommand{\svmcone}{SVM-cone\xspace}

\newcommand{\yError}{\tilde{O}\bbb{\frac{\mypsi\gamma_{\max}\min\{K^2,(\kappa(\bP))^2\}\sqrt{\kappa(\bTheta^T\bGamma^2 \bTheta)}\sqrt{Kn}}{\gamma_{\min}\lambda^*(\bB)\lambda_K(\bTheta^T\bGamma^2\bTheta)\sqrt{\rho}}}}
\newcommand{\ThetaErrorDCMMSBtype}{\tilde{O}\bbb{\frac{\gamma_{\max} K^{2.5}\min\{K^2,(\kappa(\bP))^2\} n^{3/2}}{\gamma_{\min}\eta\lambda^*(\bB)\lambda_K^2(\bTheta^T\bGamma^2\bTheta)\sqrt{\rho}}}}
\newcommand{\topicTbounds}{O_P\left(\frac{K^4\max_j \|\be_j^T \bT\|_1}{\eta(\min_j \|\be_j^T\bT\|_1)^2}\sqrt{\frac{\log \gikD}{DN}}\right)}
\newcommand{\epsilonMMSB}{\tilde{O}\bbb{\frac{\mypsi\gamma_{\max}\min\{K^2,(\kappa(\bP))^2\}K^{0.5}\nu(1+\alpha_0)}{\gamma_{\min}^3\lambda^*(\bB)\sqrt{n\rho}}}}
\newcommand{\topicRowwiseEigenspaceBound}{\frac{\kappa(\bH\bH^T)\kappa(\bT^T\bT)}{\lambda_K(\bT^T\bT)}O_P\left(\sqrt{\frac{K\log \gikD}{DN}}\right)}
\newcommand{\topicRowwiseEigenspaceBoundSimplified}{\frac{ 1}{\min_j\|\be_i^T\bT\|_1\sqrt{\lambda_K(\bT^T\bT)}}O_P\left(\sqrt{\frac{K\log \gikD}{DN}}\right)}
\usepackage[numbers]{natbib}

\begin{document}
	\title{Overlapping Clustering Models, and One (class) SVM to Bind Them All
	}
	\author{Xueyu Mao\thanks{Department of Computer Science. Email: \href{mailto:xmao@cs.utexas.edu}{xmao@cs.utexas.edu}}, Purnamrita Sarkar\thanks{Department of Statistics and Data Sciences. Email: \href{mailto:purna.sarkar@austin.utexas.edu}{purna.sarkar@austin.utexas.edu}}, Deepayan Chakrabarti\thanks{Department of Information, Risk, and Operations Management. Email: \href{mailto:deepay@utexas.edu}{deepay@utexas.edu}}\\ The University of Texas at Austin}
	\date{}
	
	\maketitle

\begin{abstract}
	People belong to multiple communities, words belong to multiple topics, and books cover multiple genres; overlapping clusters are commonplace.
	Many existing overlapping clustering methods model each person (or word, or book) as a non-negative weighted combination of ``exemplars'' who belong solely to one community, with some small noise.
	Geometrically, each person is a point on a cone whose corners are these exemplars.
	This basic form encompasses the widely used Mixed Membership Stochastic Blockmodel of networks~\cite{airoldi2008mixed} and its degree-corrected variants~\cite{jin2017estimating}, as well as topic models such as LDA~\cite{Blei:2003:LDA}.
	We show that a simple one-class SVM yields provably consistent parameter inference for all such models, and scales to large datasets.
	Experimental results on several simulated and real datasets show our algorithm (called \svmcone) is both accurate and scalable.
\end{abstract}
\section{Introduction}
\label{sec:intro}
Clustering has many real-world applications: market segmentation, product recommendation, document clustering, finding protein complexes in gene networks, among others.
The simplest form of a clustering model assumes that every record or entity belongs to exactly one cluster. 
More general forms allow for overlapping clusters, where each entity may
belong to different clusters or communities to different degrees. 
For example, George Orwell's {\em 1984} belongs to both the dystopian fiction and political fiction genres, and {\em Pink Floyd's} music is both progressive and psychedelic.
In this paper, we show that many existing overlapping clustering models can be written in a general form, whose parameters can then be inferred using a one-class SVM.

In many clustering problems, overlapping or otherwise, we have access to a data matrix $\hat{\bZ}\in \bbR^{n\times m}$, which is a noisy version of an ideal matrix ${\bZ}$, i.e.  $\hat{\bZ}=\bZ+\bR$ where the norm of the rows of $\bR$ is small.
Also, $\bZ=\bG\bZ_P$, where 
$\bZ_P$ are ideal ``exemplars'' of the various communities,
and $\bG\in \bbR_{\geq 0}^{n\times K}$ gives the community memberships of each entity. 
We will now give some examples.

Consider the Stochastic Blockmodel (SBM)~\cite{holland_stochastic_1983} for networks. 
In this model, each node belongs to one of $K$ communities, and the probability $\bP_{ij}$ of an edge between nodes $i$ and $j$ is a function of their respective communities.
Recent results~\cite{mao2017estimating} show that the eigenvectors $\hat{\bV}$ of the adjacency matrix concentrate row-wise around the eigenvectors $\bV$ of $\bP$. 
The matrix $\bV$ is also blockwise constant, mapping all nodes in one cluster to one point~\cite{rohe2011spectral}.
Hence, $\hat{\bV}=\bG\bV_{P}+\bR$, where $\bG\in\{0,1\}^{n\times K}$ is a binary membership matrix where each row sums to one.
The Mixed Membership Stochastic Blockmodel (MMSB)~\cite{airoldi2008mixed} relaxes this by allowing the entries of $\bG$ to be in $[0, 1]$.
Since the rows of $\bG$ sum to one, the ideal matrix $\bZ$ arranges points in a simplex. The corners of this simplex represent the ``pure'' nodes, i.e. nodes belonging to exactly one community. 
Most algorithms first find the corners, and then estimate model parameters via regression~\cite{mao2017,mao2017estimating,jin2017estimating,panov2017consistent,rubin2017statistical}.  
Other notable methods include tensor based approaches ~\cite{hopkins2017bayesian, MMSBAnandkumar2014}, Bayesian inference~\cite{gopalan2013efficient}, etc. Related models and inference methods for overlapping networks have been presented in~\cite{BNMF2011,kaufmann2016spectral,ray2014overlap,latouche2011}, etc.

The MMSB model does not allow for degree heterogeneity, which can be achieved via the Degree-corrected Mixed Membership Stochastic Blockmodel (DCMMSB)
~\cite{jin2017estimating}. 
In DCMMSB, each node has an extra degree parameter, with a high parameter value leading to more edges for that node. 
Now, $\bG$ is non-negative, but its rows do not sum to one. 
Thus the points lie inside a cone, and the pure nodes lie on the corner rays of this cone. 
Other network models also give rise to such cones~\cite{zhang2014detecting,kaufmann2016spectral}.

Existing algorithms for degree corrected overlapping models use a range of different techniques. 
OCCAM~\cite{zhang2014detecting} uses  
a $k$-medians step on the {regularized eigenvectors of the}
adjacency matrix to get the corners. While the algorithm is computationally efficient, a key assumption is that the $k$-medians loss function attains its minimum at the locations of the pure nodes and there is a curvature around this minimum. This condition is typically hard to check.
In~\cite{jin2017estimating}, the authors show an interesting result that the second to $K$ eigenvectors, element-wise divided by the first eigenvector entries form a simplex. The authors provide an algorithm for finding this simplex with $K$ corners in $K-1$ dimensions. The algorithm requires a combinatorial search step, which is prohibitive for large $K$.

Topic models~\citep{Blei:2003:LDA} are another example of overlapping clustering models. Here the documents can be generated from a mixture of topics, which are the analog of communities in networks. 
The normalized word co-occurrence matrix forms a simplex structure, with the corners representing anchor words, i.e. words that belong to exactly one topic. 
While there are many existing inference methods, the ones that provide consistency guarantees are typically based on analyzing tensors or finding corners in simplexes~\cite{arora12computing,ke2017topic,ding2013topic,anandkumar2014tensor,huang2016anchor,bansal2014provable,pmlr-v84-awasthi18a,bing2018fast}.

In this paper, we provide an overarching framework which incorporates all the above problems, from Mixed membership models (with or without degree correction) to topic models. As discussed before, in all the above models, the ideal data matrix lies inside a cone (a simplex is a special type of a cone). The goal is to infer $\bG$, which depends on finding the correct corner rays.

\begin{figure}
	\centering
	\begin{subfigure}[c]{0.4\textwidth}
		\fbox{\includegraphics[width=\textwidth]{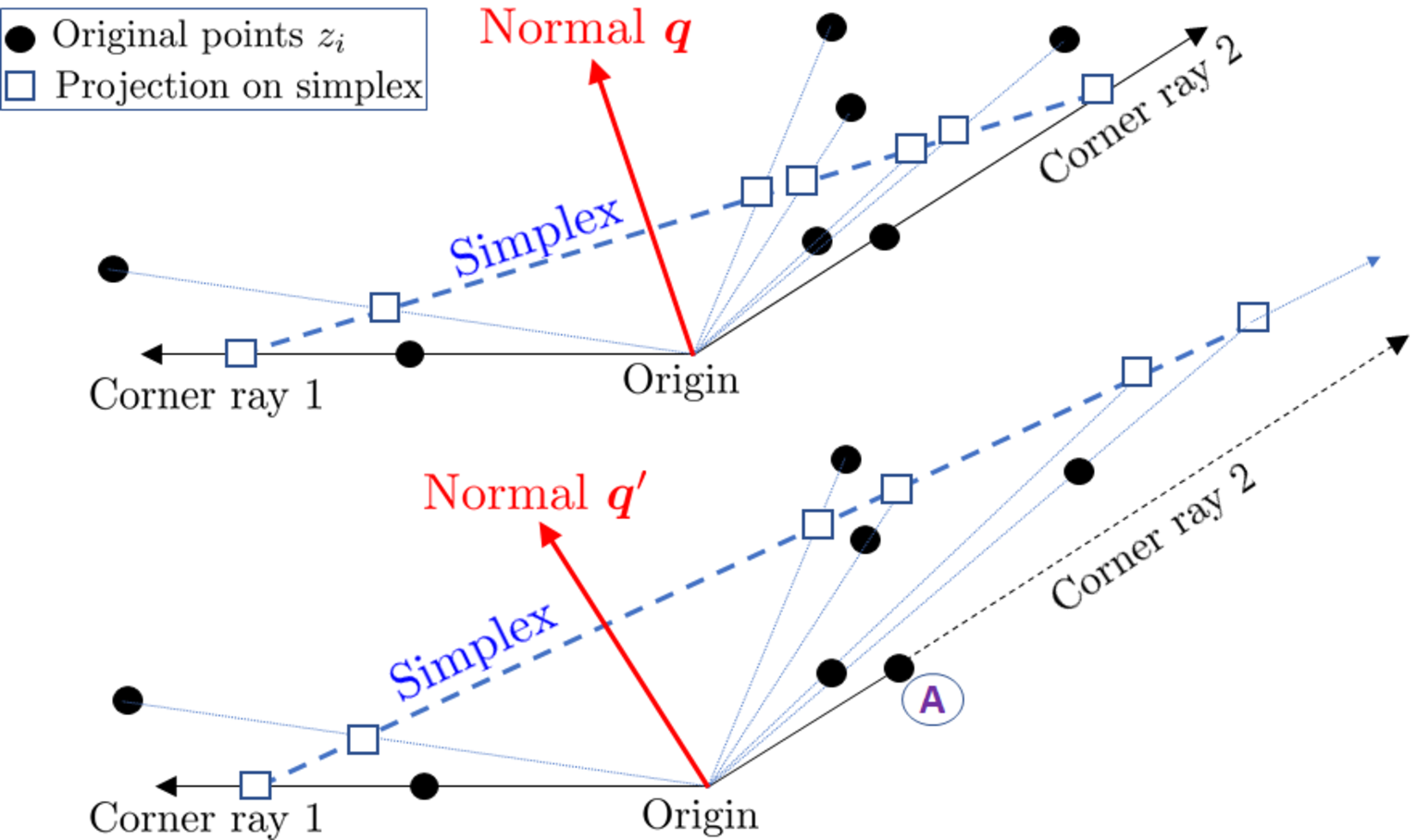}}
		\caption{Simplex methods can fail}
		\label{fig:simplexBad}
	\end{subfigure}
~
	\begin{subfigure}[c]{0.2625\textwidth}
		\fbox{\includegraphics[width=\textwidth]{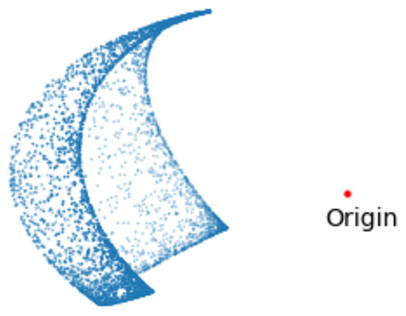}}
		\caption{Normalized points}
		\label{fig:pointsNorm}
	\end{subfigure}
~
	\begin{subfigure}[c]{0.25\textwidth}
		\fbox{\includegraphics[width=\textwidth]{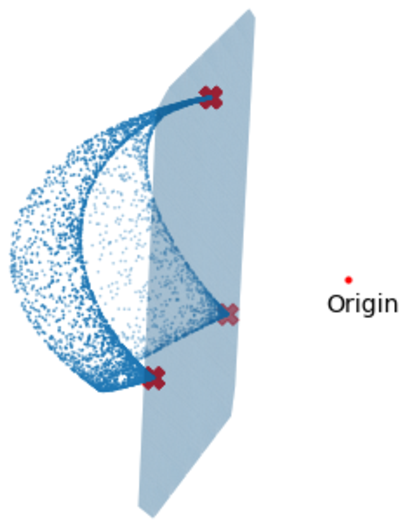}}
		\caption{Supporting hyperplane}
		\label{fig:svm}
	\end{subfigure}
	\caption{(a) Simplex-based corner-finding methods require points on a simplex, with uniformly small errors. Projecting points to a simplex with normal vector $\bm{q}$ works well, but a very similar $\bm{q}'$ does not. Some points (such as \textcircled{A}) get projected to far-off points, amplifying errors in their positions. (b) Instead, we normalize points to the unit sphere, and (c) find corners from the support vectors of a one-class SVM.}
	\label{fig1}
\end{figure}

Let us illustrate why seemingly obvious methods fail to obtain the corner rays. 
The simplest idea would be to generate a random plane (the ``simplex''), and project points to the intersection of the line  joining these points to this random plane as shown in Figure~\ref{fig:simplexBad}. 
Corners of this simplex correspond to the corner rays.
However, extending the idea to the sample or empirical cone is difficult, because if the simplex is not good, some points can get projected to arbitrarily far points, which will amplify their error.
As Figure~\ref{fig:simplexBad} shows, the set of  good simplexes may be quite limited, and finding a good simplex is difficult.

We will illustrate our idea with the ``ideal cone.''  
First, we row-normalize the ideal data matrix to have unit $\ell_2$ norm (similar to~\citet{ng2002spectral,qin2013regularized}). This projects all points inside the cone to the surface of the sphere, with the points on the corner rays being projected to the corners (Figure~\ref{fig:pointsNorm}).

Then, we show a rather fascinating result, namely, for all the above models, the corners can be obtained via the support vectors of a one class SVM~\cite{scholkopf2001estimating}, where all normalized points are in the positive class, and the origin is in the negative class (Figure~\ref{fig:svm}). Observe that a hyperplane through the corners separates all the points from the origin. 
We also show that if the row-wise error of $\bR$ is small, the SVM approach can be used to infer $\bG$ from empirical cones. Finally, we show that since the row-wise error of $\bR$ is indeed small for different degree-corrected overlapping network models and topic models, we can use our algorithm to infer the parameters consistently.
We provide error bounds for parameter estimates at the {\em per-node} and {\em per-word} level, in contrast to typical bounds for the entire parameter matrix.
We conclude with experimental results on simulated and real datasets.

\section{Proposed work}
\label{sec:proposed}

Consider a population matrix $\bP$ of the form $\bP=\rho\bGamma\bTheta \bB\bTheta^T\bGamma$, with $\bGamma\in\mathbb{R}_{>0}^{n\times n}$ being a positive diagonal matrix, $\bTheta\in\mathbb{R}_{\geq 0}^{n\times K}$ a community-membership matrix, and $\bB\in\mathbb{R}_{\geq0}^{K\times K}$ a cross-community connection matrix.
We will make the following assumptions which are common in the literature:
\begin{assumption}
(a) {\em Pure nodes:} Each community has at least one ``pure'' node, which belongs solely to that community.
(b) {\em Non-zero rows:} No row of $\bTheta$ is identically $0$.
(c) $\bB$ is full rank.
\label{assume1}
\end{assumption}
The form of the population matrix $\bP$, alongside Assumption~\ref{assume1}, induces a conic structure on the rows of the eigenvectors of $\bP$.
\begin{lem}\label{lem:vtheta}
  Let there be $K$ communities ($\rank(\bP)=K$), and let $\mathcal{I}$ be indices of $K$ pure nodes, one from each community.
	Let $\bP=\bV\bE\bV^T$ be the top-$K$ eigen-decomposition of $\bP${, where columns of $\bV\in\mathbb{R}^{n\times K}$ are the $K$ principal eigenvectors and $\bE\in\mathbb{R}^{K\times K}$ is a diagonal matrix of the $K$ principal eigenvalues}.
  Then,
	$\bV = \bGamma \bTheta \bGamma_P^{-1}\bV_P$, where $\bV_P = \bV(\mathcal{I}, :)$ is full rank and $\bGamma_P = \bGamma(\mathcal{I}, \mathcal{I})$.
\end{lem}
Since $(\bGamma\bTheta\bGamma_P^{-1})_{ij}\geq 0$ for all $(i,j)$, the rows of $\bV$ fall within a cone with corners $\bV_P$.
This suggests the following idealized problem:
\begin{problem}[Ideal cone problem]
We are given a matrix $\bZ \in \mathbb{R}^{n\times m}$ such that $\bZ = \bM \bY_P$, where $\bM\in\mathbb{R}_{\geq 0}^{n\times K}$, no row of $\bM$ is $0$, and 
$\bY_P\in\mathbb{R}^{K \times m}$
corresponds to $K$ (unknown) rows of $\bZ$, each scaled to unit $\ell_2$ norm.
Infer $\bM$ from $\bZ$.
\end{problem}
The rows of $\bY_P$ are unit vectors representing the corner rays of the cone.
Each row of $\bZ$ is constructed from a non-negative weighted combination of these unit vectors, with the weights being given by the corresponding rows of $\bM$.
Rows of $\bZ$ that lie on some corner correspond to rows of $\bM$ that have zero in all but one component.
Observe that $\bM$ is invariant to the choice of $K$ corner rows of $\bZ$ used to construct $\bY_P$.

Now consider solving the ideal cone problem with the eigenvector matrix, i.e., $\bZ=\bV$.
From Lemma~\ref{lem:vtheta}, the corner rows correspond to the pure nodes.
Choosing one such row from each corner gives us a set of pure node indices $\mathcal{I}$.
Hence, $\bM=\bGamma\bTheta\bGamma_P^{-1}\bN_P^{-1}$, where $\bN$ is a diagonal matrix with $\bN_{ii}=1/\|\be_i^T\bZ\|$ and $\bN_P=\bN(\mathcal{I},:)$.
We also have the identity $\rho\bGamma_P \bB \bGamma_P = \bV_P \bE \bV_P^T$.
Coupled with model-specific identifiability conditions ({details are provided in \supp}), these can be used to infer $\bTheta$ and $\rho\bB$ ($\bGamma$ are typically considered nuisance parameters).

In practice, we only have an observation matrix $\bA$ that is stochastically generated from the population matrix $\bP$.
Hence, we must actually solve:
\sloppy
\begin{problem}[Empirical cone problem]
We are given a matrix $\hat{\bZ} \in \mathbb{R}^{n\times m}$ such that $\max_{i\in[n]} \|\be_i^T(\bY - \hat{\bY}) \|_2 \leq \epsilon$, where $\bY = \bN\bZ$ is the row-normalized version of $\bZ$, and $\hat{\bY}$ is constructed similarly from $\hat{\bZ}$. 
Again, $\bZ = \bM \bY_P$, where $\bM\geq 0$, no row of $\bM$ is $0$, and $\bY_P=\bY(\mathcal{I},:)$ corresponds to $K$ (unknown) rows of $\bY$ with indices $\mathcal{I}$.
Infer $\bM$ from $\hat{\bZ}$.
\end{problem}
\fussy

We will first present the solution to the ideal cone problem.
We will then show that the same algorithm with some post-processing solves the empirical cone problem up to $O(\epsilon)$ error.
Finally, we apply our algorithm to infer parameters for a variety of models, and present error bounds for each.

\noindent
{\bf Notation:} We shall refer to the $i^{th}$ row of $\bZ$ as $\bz_i^T$ {\em expressed using a column vector}, i.e., $\bz_i = \bZ^T\be_i$.
The same pattern will be used for rows $\bbm_i^T$ of $\bM$, and other matrices as well.

\subsection{The Ideal Cone Problem}
Observe that given the corner indices $\mathcal{I}$ (i.e., given $\bY_P$), finding $\bM$ such that $\bZ=\bM \bY_P$ is a simple regression problem.
Thus, the only difficulty is in finding the corner indices.

Our key insight is that under certain conditions, the ideal cone problem can be solved easily by a {\bf one-class SVM applied to the rows of $\bY$}.
Figure~\ref{fig1} plots the normalized rows $\by_i$ of $\bY$ for an example cone.
Observe that a hyperplane through the corners separates all the points from the origin.
This suggests that the normalized corners are the support vectors found by a one-class SVM:
\begin{equation}
\text{maximize }\  b\quad \text{ s.t. } \bw^T \by_i \geq b \;(\text{for } i=1, \ldots, n) \text{ and } \|\bw\|_2 \leq 1.
\label{eq:SVM}
\end{equation}
We show next that this intuition is correct.
Define the following condition.
\begin{cond}
The matrix $\bY_P$ satisfies $(\bY_P\bY_P^T)^{-1}\bone > 0$.
\label{cond:pop}
\end{cond}
\begin{thm}
Each support vector selected by the one-class SVM (Eq.~\ref{eq:SVM}) is a corner of \bZ.
Also, if Condition~\ref{cond:pop} holds, there is at least one support vector for every corner.
\end{thm}
Thus, under Condition~\ref{cond:pop}, we can get all the corners from the support vectors, and then find $\bM$ via regression of $\bZ$ on $\bY_P$.
Condition~\ref{cond:pop} is always satisfied for our problem setting, as shown next.
\begin{thm}\label{thm:dcmmsb_condition_true}
Let $\bP$ be a population matrix satisfying Assumption~\ref{assume1}.
Let $\bZ=\bV$, where $\bV$ is the rank-$K$ eigenvector matrix. 
Let $\bY=\bN\bZ$ as defined above.
Then, Condition~\ref{cond:pop} is true.
\label{thm:condpopHolds}
\end{thm}
Thus, the ideal cone problem is easily solved by a one-class SVM.
Next, we show that the same method suffices for the empirical cone problem too.

\subsection{The Empirical Cone Problem}
Now, instead of the normalized eigenvector rows $\bY$, we are given the empirical matrix $\hat{\bY}$ with rows $\hat{\bz}_i^T/\|\hat{\bz}_i\|$, where $\max_i \|\be_i^T (\hat{\bY} - \bY)\|\leq \epsilon$.
Once again, we focus on finding the corner indices, using which $\bM$ can be inferred by regression.
We will show that running a one-class SVM on the rows of $\hat{\bY}$ yields ``near-corners,'' after some post-processing.
We will need a stronger form of Condition~\ref{cond:pop}:
\begin{cond}
The matrix $\bY_P$ satisfies {${(\bY_P\bY_P^T)^{-1}\bone} \geq \eta\bone$ }
 for some constant $\eta>0$.
\label{cond:emp}
\end{cond}
It is easy to show that the solution $(\bw, b)$ of the population SVM under Condition~\ref{cond:pop} is given by 
\begin{align}
\bw &= b^{-1}\cdot \bY_P^T \bbeta &
b&=\left(\bone^T (\bY_P\bY_P^T)^{-1}\bone \right)^{-1/2} &
\bbeta &= \dfrac{(\bY_P\bY_P^T)^{-1}\bone}{\bone^T (\bY_P\bY_P^T)^{-1}\bone}.
\label{eq:bbeta}
\end{align}
Thus, Condition~\ref{cond:pop} implies that $\bw$ is a convex combination of the corners, while Condition~\ref{cond:emp} additionally requires a minimum contribution from each corner.

\begin{lem}[SVM solution is nearly ideal]
Let $(\hat{\bw}, \hat{b})$ be the solution for the one-class SVM (Eq.~\ref{eq:SVM}) applied to the rows of $\hat{\bY}$.
Under Condition~\ref{cond:emp},
we have $|\hat{b}-b|\leq\epsilon$ and $\|\hat{\bw}-\bw\|\leq \zeta\epsilon$, for $\zeta=\frac{4}{\eta b^2\sqrt{\lambda_K(\tmtxP_P\tmtxP_P^T)}}\leq \frac{4K}{\eta(\lambda_K(\tmtxP_P\tmtxP_P^T))^{1.5}}$.
\label{lem:bw}
\end{lem}

Unlike the ideal cone scenario, the rows $\hat{\bY}_P$ corresponding to the corners need not be support vectors for the empirical cone.
However, they are not far off.
\begin{lem}[Corners are nearly support vectors]
The corners of the population cone are close to the supporting hyperplane:
$\mbox{$\hb\bone\leq\tmtxSp\hww\leq\hb\bone+(\zeta+2)\epsilon\bone$}$.
\label{lem:cornerNearSV}
\end{lem}
This suggests that we should consider all points that are up to $(\zeta+2)\epsilon$ away from the supporting hyperplane when searching for corners.
The next Lemma shows that each such point is a ``near-corner.''

Recall that each row $\hat{\by}_i^T$ is a noisy version of a population row $\by_i^T = \bbm_i^T \bY_P/\|\bbm_i^T \bY_P\|$, which can be rewritten as a scaled convex combination of the normalized corners: ${\by_i^T = r_i \rowcvxY_i^T \bY_P}$, where $\rowcvxY_i^T\bone=1$.
Specifically, $r_i =  \frac{\bbm_i^T \bone}{\|\bbm_i^T\bY_P\|}$ and $\rowcvxY_i = \frac{\bbm_i}{\bbm_i^T \bone}$. 
For a corner, $r_i=1$ and $\rowcvxY_i=\be_j$ for some $j$.
We now show that every point $i$ that is close to the supporting hyperplane is nearly a corner of the ideal cone.
\begin{lem}[Points close to support vectors are near-corners]
	If $\hww^T\trowS_i\leq \hb + (\zeta+2)\epsilon$ for some point $i\in [n]$, then
  $\mbox{$1\leq r_i\leq 1+ \frac{(\zeta+4)\epsilon}{b-\epsilon}$}$
  and
  $\elecvxY_{ij}\geq 1- \frac{2\zeta\epsilon}{b\lambda_K(\tmtxPp\tmtxPp^T)}$ for some $j\in[K]$.
  \label{lem:nearCorner}
\end{lem}

Consider the set of points $S_c=\{i\mid \hww^T\trowS_i\leq \hb + (\zeta+2)\epsilon\}$ that are close to the supporting hyperplane.
Lemmas~\ref{lem:cornerNearSV} and~\ref{lem:nearCorner} show that $S_c$ contains all corners, and possibly other points that are all near-corners.
This suggests that we can cluster the vectors $\{\trowS_i\mid i\in S\}$ into $K$ clusters, each corresponding to one corner and possibly extra near-corners close to that corner.
Randomly selecting one point from each cluster gives us the set of inferred corners.
\begin{lem}[Each corner has its own cluster]
	There exist exactly $K$ clusters in $S_c$, as long as
  $\epsilon\leq c_{\epsilon}\frac{\eta(\lambda_K(\tmtxPp\tmtxPp^T))^3}{K^{1.5}\sqrt{\kappa(\tmtxPp\tmtxPp^T)}}$,
  for some global constant $c_{\epsilon}$.
\label{lem:Kclusters}
\end{lem}

Let $C$ be the indices of the near-corners picked by this clustering step.
Since $\bZ = \bM\bY_P$, this suggests $\bM$ can be obtained via regression: $\bM \approx \mtxS\tmtxS_C^T(\tmtxS_C\tmtxS_C^T)^{-1}\bpi$,
where ${\tmtxS_C:=\tmtxS(C,:)}$ and $\bpi$ is a permutation matrix that matches the ordering of ideal corners and the empirical near-corners.
\begin{thm}\label{thm:M_row_bound}
If Condition~\ref{cond:emp} and the condition on $\epsilon$ in Lemma~\ref{lem:Kclusters} holds, then for any $i\in[n]$,
	${
		\|\be_i^T(\cvxP - \mtxS\tmtxS_C^T(\tmtxS_C\tmtxS_C^T)^{-1}\bpi)\| \leq \frac{c_M\kappa(\tmtxPp\tmtxPp^T)\|\be_i^T\mtxP\|{K\zeta}}{(\lambda_K(\tmtxPp\tmtxPp^T))^{2.5}} \epsilon,
	}$
where $c_M$ is a global constant, {and $\kappa(.)$ is the ratio of the largest and smallest nonzero singular values of a matrix.}
  \label{thm:M}
\end{thm}

Algorithm~\ref{algo:inferCone} shows all the steps of our method (\svmcone).
The algorithm requires an estimate of $\delta:=(\zeta+2)\epsilon$, and returns the inferred $\bM$ and near-corners $C$.
When the row-wise error bound $\epsilon$ is unknown, we can start with $\delta=0$ and incrementally increase it until $K$ distinct clusters are found.

\begin{algorithm}[H] 
	\caption{SVM-cone}
	\label{algo:inferCone}
	\begin{algorithmic}[1]
		\REQUIRE $\mtxS \in \R^{n\times m}$, number of corners $K$, estimated distance of corners from hyperplane $\delta$ 
		\ENSURE  Estimated conic combination matrix $\hat{\bM}$ and near-corner set $C$
		\STATE Normalize rows of $\mtxS$ by $\ell_2$ norm to get $\tmtxS$ with rows $\hat{\by}_i^T$
		\STATE Run one-class SVM on $\hat{\by}_i$ to get the normal $\hww$ and distance $\hb$ of the supporting hyperplane
    \STATE Cluster points $\{\hat{\by}_i\mid \hww^T\trowS_i\leq \hb + \delta\}$ that are close to the hyperplane into $K$ clusters
    \STATE Pick one point from each cluster to get near-corner set $C$
    \STATE $\hat{\bM} =\mtxS\tmtxS_C^T(\tmtxS_C\tmtxS_C^T)^{-1}$
	\end{algorithmic}
\end{algorithm}

\section{Applications}
Many network models and topic models have population matrices of the form $\bP=\rho\bGamma \bTheta \bB \bTheta^T \bGamma$.
We have already shown that in such cases, the eigenvector matrix $\bV$ forms an ideal cone (Lemma~\ref{lem:vtheta}), and that Condition~\ref{cond:pop} holds.
It is easy to see that the same holds for $\bV\bV^T$ as well.
This suggests that \svmcone can be applied to the matrix $\hat{\bV}\hat{\bV}^T$, where $\hat{\bV}$ is the empirical top-$K$ eigenvector matrix.
We shall show that this yields {\em per-node error bounds} in estimating community memberships and {\em per-word error bounds} for word-topic distributions.

\subsection{Network models}
Define a ``DCMMSB-type'' model as a model with population matrix $\bP=\rho\bGamma\bTheta\bB\bTheta^T\bGamma$ and an empirical adjacency matrix $\bA$ with $\bA_{ji}=\bA_{ij} \sim \text{Bernoulli}(\bP_{ij})$ for all $i>j$. Assume that rows of $\bTheta$ have unit $\ell_p$ norm, for $p=1$ (DCMMSB) or $p=2$ (OCCAM~\citep{zhang2014detecting}).
Let $\bv_i=\bV^T\be_i$, $\hvv_i=\hv^T\be_i$, $\trowP_i=\bV\bv_i/\|\bV\bv_i\|$, and $\trowS_i=\hv\hvv_i/\|\hv\hvv_i\|$. Denote $\gamma_{\max}=\max_i\bGamma_{ii}$ and $\gamma_{\min}=\min_i\bGamma_{ii}$. 
\sloppy
\begin{thm}[Small row-wise error in Network Models]
	Consider a DCMMSB-type model with $\btheta_i\sim\mathrm{Dirichlet}(\balpha)$, and $\alpha_0:=\balpha^T\bone$.
	If $\bal:= \dfrac{\alpha_0}{\alpha_{\min}}\leq   \dfrac{\min( \sqrt{\frac{n}{27\log n}}, {\frac{\gamma_{\min}^2}{\gamma_{\max}^2}n\rho})}{2(1+\alpha_0)}$, ${\dfrac{\lambda^*(\bB)}{\nu}\geq \dfrac{8(1+\alpha_0)(\log n)^{\xi}}{\gamma_{\min}^2\sqrt{n\rho}}}$ for some constant $\xi>1$, $\kappa(\bTheta^T\bGamma^2 \bTheta)=\Theta(1)$,   and $\alpha_0=O(1)$, then
	\bas{
		\epsilon=\max_i \|\trowP_i-\trowS_i\| = \epsilonMMSB
	}
	with probability at least $1-O(Kn^{-2})$. 
  Here $\lambda^*(\bB)$ is the smallest singular value of $\bB$. 
  \label{thm:row_wise_vvt1}
\end{thm}
\fussy
Similar results for the non-Dirichlet case follow easily as long as $n\rho=\Omega((\log n)^{2\xi})$, $\lambda_K(\bP)=\Omega(\sqrt{n\rho}(\log n)^\xi)$, and $\max_i \|\bV(:,i)\|=O(\sqrt{\rho})$ with high probability.
This shows that the rows of $\hat{\bV}\hat{\bV}^T$ are close to those of $\bV\bV^T$, and the latter forms an ideal cone satisfying Condition~\ref{cond:pop}.
Hence, the conic combination for each node can be recovered by Algorithm~\ref{algo:inferCone} applied to $\hat{\bV}\hat{\bV}^T$.
In fact, we can run the algorithm on $\hat{\bV}$ itself; the output depends only on the SVM dual variables $\bbeta$ (Eq.~\ref{eq:bbeta}), which are the same whether the input is $\hat{\bV}$ or $\hat{\bV}\hat{\bV}^T$.
The output is the same conic combination matrix $\hat{\bM}$ and the same set $C$ of nearly-pure nodes.

For \textbf{identifiability} of $\bTheta$, we need another condition.
We will assume that $\sum\bGamma_{ii}=n$ and all diagonal entries of $\bB$ are equal ({details are provided in \supp}). 
The next theorem shows that \svmcone can be used to consistently infer the parameters of DCMMSB as well as OCCAM~\cite{zhang2014detecting}.
\begin{thm}[Consistent inference of community memberships for each node]\label{thm:theta_bound}
  Consider DCMMSB-type models
  where the conditions of Theorem~\ref{thm:row_wise_vvt1} are satisfied and ${\kappa(\bTheta^T\bGamma^2 \bTheta)=\Theta(1)}$.
  Let $\hat{\bD}$ be a diagonal matrix with entries $\hat{\bD}_{ii} = \sqrt{\be_i^T\tmtxS_C\hv\hat{\bE}\hv^T\tmtxS_C^T\be_i}$.
  Let $\hat{\bTheta} = \hat{\bF}^{-1}\hat{\bM}\hat{\bD}$, where $\hat{\bF}$ is a diagonal matrix with entries $\hat{\bF}_{ii}=\|\be_i^T\hat{\bM}\hat{\bD}\|_1$ (for DCMMSB) and ${\hat{\bF}_{ii}=\|\be_i^T\hat{\bM}\hat{\bD}\|_2}$ (for OCCAM~\cite{zhang2014detecting}).
  Then there exists a permutation matrix $\bpi$ such that
	\bas{
		{\|\be_i^T(\bTheta-\hat{\bTheta}\bpi)\|}
    &= \ThetaErrorDCMMSBtype
	}
	with probability at least $1-O(Kn^{-2})$. 
\end{thm}
\begin{rem} 
	The error bound is small when the clusters are well separated (large $\lambda^*(\bB)$), the network is dense (large $\rho$), there are few blocks (small $K$), and the membership vectors $\bTheta$ are drawn from a balanced Dirichlet distribution (small $\nu$, and hence small $\kappa(P)$), which leads to balanced block sizes.
\end{rem}
\begin{rem}
	For DCMMSB-type models,  $\eta\geq\frac{\gamma_{\min}^2\min_i(\be_i^T\bTheta^T\bone)}{{\lambda_1(\bTheta^T\bGamma^2 \bTheta)}}$. Also, under the conditions of Theorem~\ref{thm:row_wise_vvt1}, $\eta\geq\frac{\gamma_{\min}^2}{3\nu\gamma_{\max}^2}$ with high probability. Proofs are in \supp.
\end{rem}
Observe that these are {\bf per-node error bounds}, as against a simpler bound on $\|\bTheta-\hat{\bTheta}\|$.
Clearly, the same results extend to the special case of 
the Mixed Membership Stochastic Blockmodel~\cite{airoldi2008mixed} and the Stochastic Blockmodel~\cite{holland_stochastic_1983} 
as well (the assumption of equal diagonal entries of $\bB$ is no longer needed, since $\bGamma_{ii}=1$ is enough for parameter identifiability~\cite{mao2017estimating}).

\subsection{Topic Models}
Let $\bT\in\mathbb{R}_{\geq 0}^{V\times K}$ be a matrix of the word to topic probabilities with unit column sum, and let $\bH\in\mathbb{R}_{\geq 0}^{K\times D}$ be the topic to document matrix. 
Then $\W:=\bT\bH$ is the probability matrix for words appearing in documents.
The actual counts of words in documents are assumed to be generated iid as $\w_{ij}\sim \text{Binomial}(N,\W_{ij})$ for $i\in [V], j\in [D]$.

The word co-occurrence probability matrix is given by $\W\W^T/D = \bT (\bH\bH^T/D) \bT^T$.
Setting $\bGamma_{ii}=\|\bT(i,:)\|_1$, $\bTheta=\bGamma^{-1}\bT$, and $\bB=\bH\bH^T/D$, we find that $\W\W^T/D = \bGamma\bTheta\bB\bTheta^T\bGamma$ with $\bTheta\bone=\bone$.
This clearly matches the form of $\bP$ in the DCMMSB model.
Hence, its eigenvector matrix has the desired conic structure with weight matrix $\bM = \bT\bGamma_P^{-1}\bN_P^{-1}$, with the ``pure nodes'' being {\em anchor words} that only occur in a single topic.
We now show that the row-wise error between the empirical and population eigenvector matrices decays with increasing number of documents $D$ and number of words in a document $N$.

\begin{assumption}\label{assumption:topic}
	Let $g_{ik}=\be_i^T\W\W^Te_k$. We assume that when it is not zero, it goes to infinity, in particular, $g_{ik}\geq N\log \gikD$, which gives $D/N\rightarrow \infty$. We also assume that $\lambda_i(\bH\bH^T)=\Theta(D)$,  for $i\in [K]$, and $\kappa(\bT\bT^T)=\Theta(1)$.
\end{assumption}
These assumptions are similar to ones made in other theoretical literature on topic models~\cite{ke2017topic}.

We will construct a matrix $\w_1\w_2^T$, where $\w_1$ and $\w_2$ are obtained by dividing the words in each document uniformly randomly in two equal parts. This ensures that $\uE[\w_1\w_2^T]=\W\W^T$, which in turn helps establishing concentration of empirical singular vectors as shown in the following lemma. For simplicity denote $N_1=N/2$.
\begin{lem}[Small row-wise error in Topic Models]\label{lem:topic_eigen_bound}
\label{lem:topic_eigen_bound1}
Let $\hat{\bV}$ denote the matrix of the top-$K$ singular vectors of $\bU=\w_1\w_2^T/N_1^2$, 
and let the population counterpart of this be $\bV$. 
Let $\bv_i=\bV^T\be_i$, $\hvv_i=\hv^T\be_i$, $\trowP_i=\bV\bv_i/\|\bV\bv_i\|$, and $\trowS_i=\hv\hvv_i/\|\hv\hvv_i\|$. 
Under Assumption~\ref{assumption:topic}, we have:
\bas{
	\epsilon=\max_i \|\trowP_i-\trowS_i\| = \topicRowwiseEigenspaceBoundSimplified
}
with probability at least $1-O(1/D^2)$.
\end{lem}
Thus, Algorithm~\ref{algo:inferCone} run on $\hat{\bV}\hat{\bV}^T$ (or equivalently, just $\hat{\bV}$) can be used to find the conic combination weights $\hat{\bM}\approx\bM$.
Since $\bM$ being the product of $\bT$ with a diagonal matrix where $\bT$ has unit column sum, we can extract $\hat{\bT}=\hat{\bM}\hat{\bD}^{-1}$, where $\hat{\bD}$ is a diagonal matrix with $\hat{\bD}_{ii}=\|\hat{\bM}(:,i)\|_1$.
\begin{thm}[Consistent inference of word-topic probabilities for each word]\label{thm:topic_error_bound}
Under Assumption~\ref{assumption:topic}, there exists a permutation matrix $\bpi$ such that 
\bas{
	 \frac{\|\be_i^T(\hT-\bT\bpi^T)\|}{\|\be_i^T\bT\|}=
		\topicTbounds
}with probability at least $1-O(1/D^2)$.
\end{thm}

\begin{rem}
	For topic models,
	$\eta\geq {\min_i\|\be_i^T\bT\|_1}/{{\lambda_1(\bT^T \bT)}} \geq {\min_i\|\be_i^T\bT\|_1}/{{K}}$. Proofs are in \supp.
\end{rem}

\section{Experiments}
\label{sec:exp}
We ran experiments on simulated and real-world datasets to verify the accuracy and scalability of \svmcone.
We compared \svmcone against several competing baselines.
For {network models}, 
{\bf GeoNMF} detects the corners of a simplex formed by the MMSB model by constructing the graph Laplacian and picking nodes that have large norms in the Laplacian~\cite{mao2017}.
It assumes balanced communities (i.e., the rows of $\bTheta$ are drawn from a Dirichlet with identical community weights).
{\bf SVI} uses stochastic variational inference for MMSB~\cite{gopalan2013efficient}.
{\bf BSNMF}~\cite{BNMF2011} presents a Bayesian approach to Symmetric Nonnegative Matrix Factorization; it can be applied to do inference for MMSB models with $\bB=c\bI$ where $c\in[0,1]$.
{{\bf OCCAM} works on a variant of MMSB where each row of $\bTheta$ has unit $\ell_2$ norm, and the model allows for degree heterogeneity~\cite{zhang2014detecting}.}
{\bf SAAC}~\cite{kaufmann2016spectral} uses alternating optimization on a version of the stochastic blockmodel where each node can be a member of multiple communities, but the membership weight is binary. 
{For topic models, \textbf{RecoverL2} \cite{arora2013practical} uses a combinatorial algorithm to pick anchor words from the word co-occurrence matrix and then recovers the word-topic vectors by optimizing a quadratic loss function. \textbf{TSVD} \cite{bansal2014provable} uses a thresholded SVD based procedure to recover the topics. 
\textbf{GDM} \cite{yurochkin2016geometric} is a geometric algorithm that involves a weighted clustering procedure augmented with geometric corrections.}
We could not obtain the code for~\cite{jin2017estimating,ke2017topic}.

\subsection{Networks with overlapping communities}
In this section, we present experiments on simulated and large real networks.
\subsubsection{Simulations}
We test the recovery of population parameters $(\bTheta, \bB)$ given adjacency matrices $\bA$ generated from the corresponding population matrices $\bP$ ($\bGamma$ are nuisance parameters).
We generate networks with $n=5000$ nodes and $K=3$ communities.
The rows of $\bTheta$ are drawn from Dirichlet$(\bm{\alpha})$ for DCMMSB and OCCAM; for DCMMSB, $\balpha=(1/3,1/3,1/3)$; for OCCAM, $\balpha=(1/6,1/6,1/6)$ and the rows are normalized to have unit $\ell_2$ norm.
We set $\bB_{ii}=1$ and $\bB_{ij}=0.1$ for all $i\neq j$.
The default degree parameters for DCMMSB are as follows: 
for all nodes $i$ that are predominantly in the $j$-th community ($\theta_{ij}>0.5$), 
we set $\bGamma_{ii}$ to $0.3$, $0.5$, and $0.7$ for the 3 respective communities; all other nodes have $\bGamma_{ii}=1$.
For OCCAM, we draw degree parameters from a $\mathrm{Beta}(1,3)$ distribution.
\begin{figure}[!t]
  \centering
	  \begin{subfigure}[b]{0.45\textwidth}
	    \includegraphics[width=\textwidth]{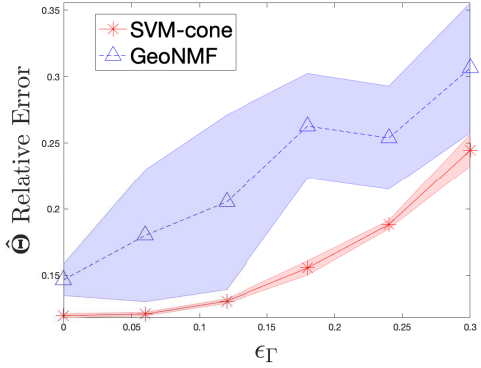}
	    \caption{\scriptsize Varying degree heterogeneity for DCMMSB}
	    \label{sim:gamma}
	  \end{subfigure}
	\begin{subfigure}[b]{0.45\textwidth}
		\includegraphics[width=\textwidth]{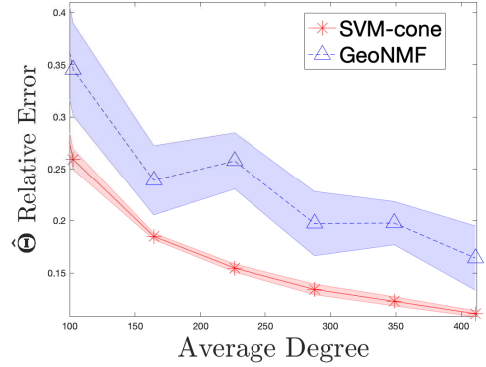}
		\caption{\scriptsize Varying sparsity for DCMMSB}
		\label{sim:rho}
	\end{subfigure}
	\\
	\begin{subfigure}[b]{0.45\textwidth}
		\includegraphics[width=\textwidth]{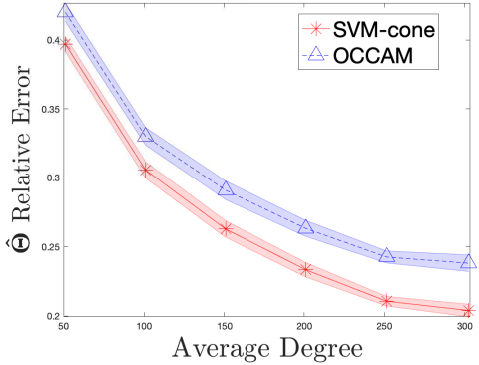}
		\caption{\scriptsize Varying sparsity for OCCAM model}
		\label{sim:occam}
	\end{subfigure}
	\begin{subfigure}[b]{0.45\textwidth}
		\includegraphics[width=\textwidth]{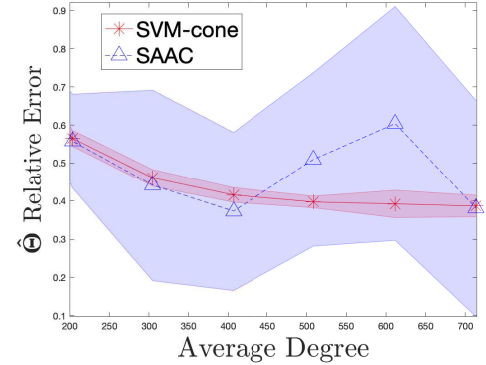}
		\caption{\scriptsize Varying sparsity for OSBM}
		\label{sim:saac}
	\end{subfigure}
  \caption{Relative error in estimation of community memberships: Plots (a) and (b) compare \svmcone against the closest baseline (GeoNMF) on the degree-corrected MMSB model. We then compare against (c) OCCAM and (d) SAAC on networks drawn from their generative models.}
  \label{fig:sim}
\end{figure}

\noindent

\noindent
{\bf Varying degree parameters $\bGamma$:} We set the degree parameters for predominant nodes in the 3 communities as $0.5+\epsilon_\Gamma$, $0.5$, and $0.5+\epsilon_\Gamma$ respectively. Figure~\ref{sim:gamma} shows \svmcone outperforms GeoNMF consistently for all choices of $\epsilon_\Gamma$.

\noindent
{\bf Varying network sparsity $\rho$:}
Figure~\ref{sim:rho} shows the relative error in estimating $\bTheta$ as a function of the network sparsity $\rho$.
Increasing $\rho$ increases the average degree of nodes in the network without affecting the skew induced by their degree parameters $\bGamma$.
As expected, all methods tend to improve with increasing degree.
Our method dominates GeoNMF over the entire range of average degrees.
Figures~\ref{sim:occam} and \ref{sim:saac} show results for networks generated under the models used by OCCAM and SAAC respectively.
\svmcone is comparable or better than these methods even on their generative models.
The smaller error bars on \svmcone show that it is more stable than SAAC. 

\subsubsection{Real-world experiments}
\begin{table*}[!t]
	\centering
	\tablefontsize
	\caption{Network statistics}
	\begin{subtable}{\textwidth}
		\centering
		\caption{DBLP coauthorship networks.}
	\begin{tabular}{|c|c|c|c|c|c|c|}

		\hline
		{\tablefontsize Dataset}  & {\tablefontsize DBLP1}     & {\tablefontsize DBLP2}    & {\tablefontsize DBLP3}   & {\tablefontsize DBLP4}    & {\tablefontsize DBLP5}  \\
		\hline
		\hline
		\tablefontsize $\#$ nodes $n$        & \tablefontsize 30,566    & \tablefontsize 16,817    & \tablefontsize 13,315    & \tablefontsize 25,481    & \tablefontsize  42,351       \\
		\hline
		\tablefontsize $\#$ communities $K$     & \tablefontsize 6         & \tablefontsize 3        & \tablefontsize 3        & \tablefontsize 3      & \tablefontsize   4        \\
		\hline
		\tablefontsize Average Degree                          & \tablefontsize 8.9 &  \tablefontsize 7.6 & \tablefontsize 8.5 &  \tablefontsize 5.2  &  \tablefontsize   6.8       \\
		\hline
		Overlap $\%$       & \tablefontsize 18.2 &  \tablefontsize 14.9 &\tablefontsize 21.1 &\tablefontsize 14.4 &\tablefontsize   18.5    \\
		\hline

	\end{tabular} 
	\label{tab:normal_DBLP}
	\end{subtable}%
	\\
	\vspace{1em}
	\begin{subtable}{\textwidth}
	\centering
	\caption{DBLP bipartite author-paper networks.}

	\begin{tabular}{|c|c|c|c|c|c|c|}

		\hline
		{\tablefontsize Dataset}  & {\tablefontsize DBLP1}     & {\tablefontsize DBLP2}    & {\tablefontsize DBLP3}   & {\tablefontsize DBLP4}    & {\tablefontsize DBLP5}  \\
		\hline
		\hline
		\tablefontsize $\#$ nodes $n$        & \tablefontsize 103,660    & \tablefontsize 50,699    & \tablefontsize 42,288    & \tablefontsize 53,369    & \tablefontsize    81,245       \\
		\hline
		\tablefontsize $\#$ communities $K$     & \tablefontsize 12         & \tablefontsize 
		6        & \tablefontsize 6        & \tablefontsize 6    & \tablefontsize  8        \\
		\hline
		\tablefontsize Average Degree                          & \tablefontsize 3.4 &  \tablefontsize 3.4 & \tablefontsize 3.6 &  \tablefontsize 2.6  & \tablefontsize  3.0       \\
		\hline
		Overlap $\%$       & \tablefontsize 6.3 &  \tablefontsize 5.6 &\tablefontsize 5.7 &\tablefontsize 6.9   & \tablefontsize   9.7    \\
		\hline

	\end{tabular} 
	\label{tab:bipartite_DBLP}
	\end{subtable}%
	\label{table:net_stats}
	\vspace{0.5em}
\end{table*}
\sloppy
We tested \svmcone on large network datasets and word-document datasets.
For networks, we used the $5$ DBLP coauthorship networks\footnote{\url{http://www.cs.utexas.edu/˜xmao/coauthorship}} (used in~\cite{mao2017}, where each ground truth community corresponds to a group of conferences on the same topic. 
We also use bipartite author-paper variants for these $5$ networks. 
See Table~\ref{table:net_stats} for network statistics. 
Following~\cite{mao2017}, we evaluate results by the rank correlation between the predicted vector for {community $i$ against the true vector, averaged over all communities:}
${RC_{avg}(\hat{\bTheta}, \bTheta) = \frac{1}{K}\max_{\sigma}\sum_{i=1}^K RC(\hat{\bTheta}(:,i), \bTheta(:,\sigma(i)))}$, 
where $\sigma$ is a permutation over the $K$ communities.
We have $-1\leq RC_{avg}(\hat{\bTheta}, \bTheta)\leq 1$, with higher numbers implying a better match between $\hat{\bTheta}$ and $\bTheta$.
{We do not use metrics like NMI~\citep{strehl2002cluster} or ExNVI~\citep{zhang2014detecting} that require binary overlapping membership vectors to avoid thresholding issues on real-valued membership vectors.}

\begin{figure}[!t]
	\centering
	\begin{subfigure}[b]{0.49\textwidth}
		\includegraphics[width=\textwidth]{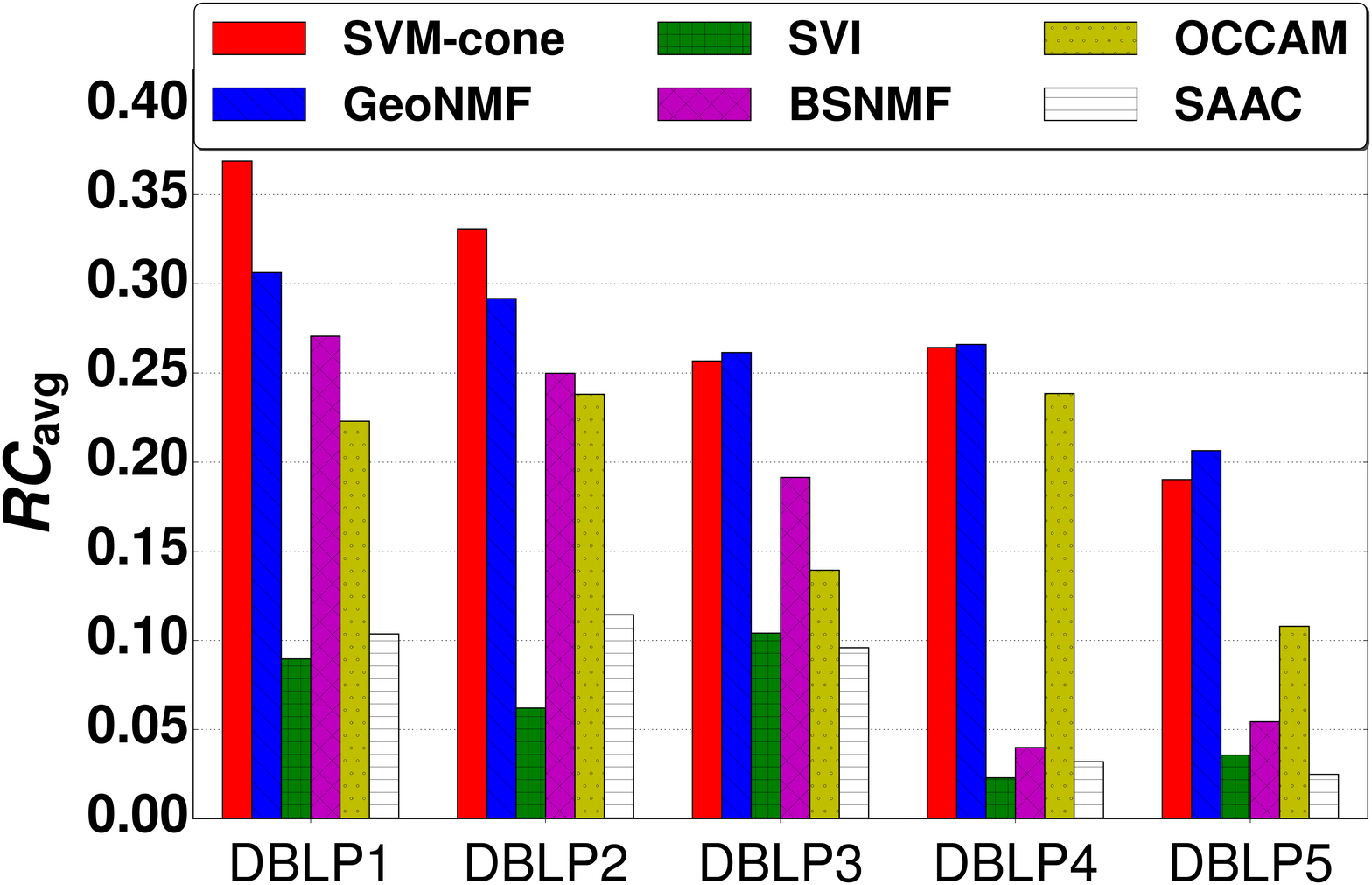}
		\caption{DBLP coauthorship}
		\label{dblp:coauth}
	\end{subfigure}
	\begin{subfigure}[b]{0.49\textwidth}
		\includegraphics[width=\textwidth]{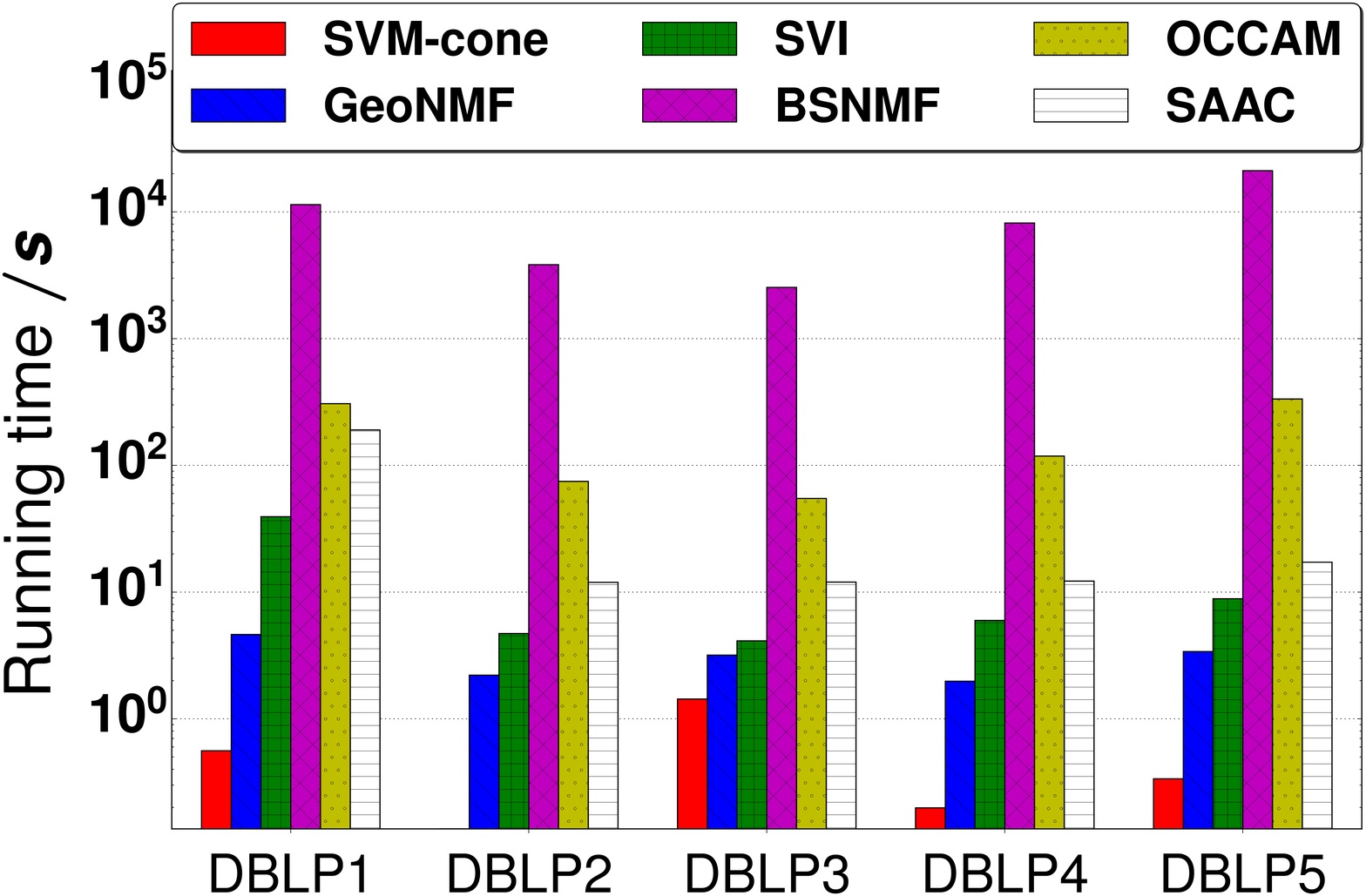}
		\caption{DBLP coauthorship wall-clock time}
		\label{dblp:time}
	\end{subfigure}
	\\
	\begin{subfigure}[b]{0.49\textwidth}
		\includegraphics[width=\textwidth]{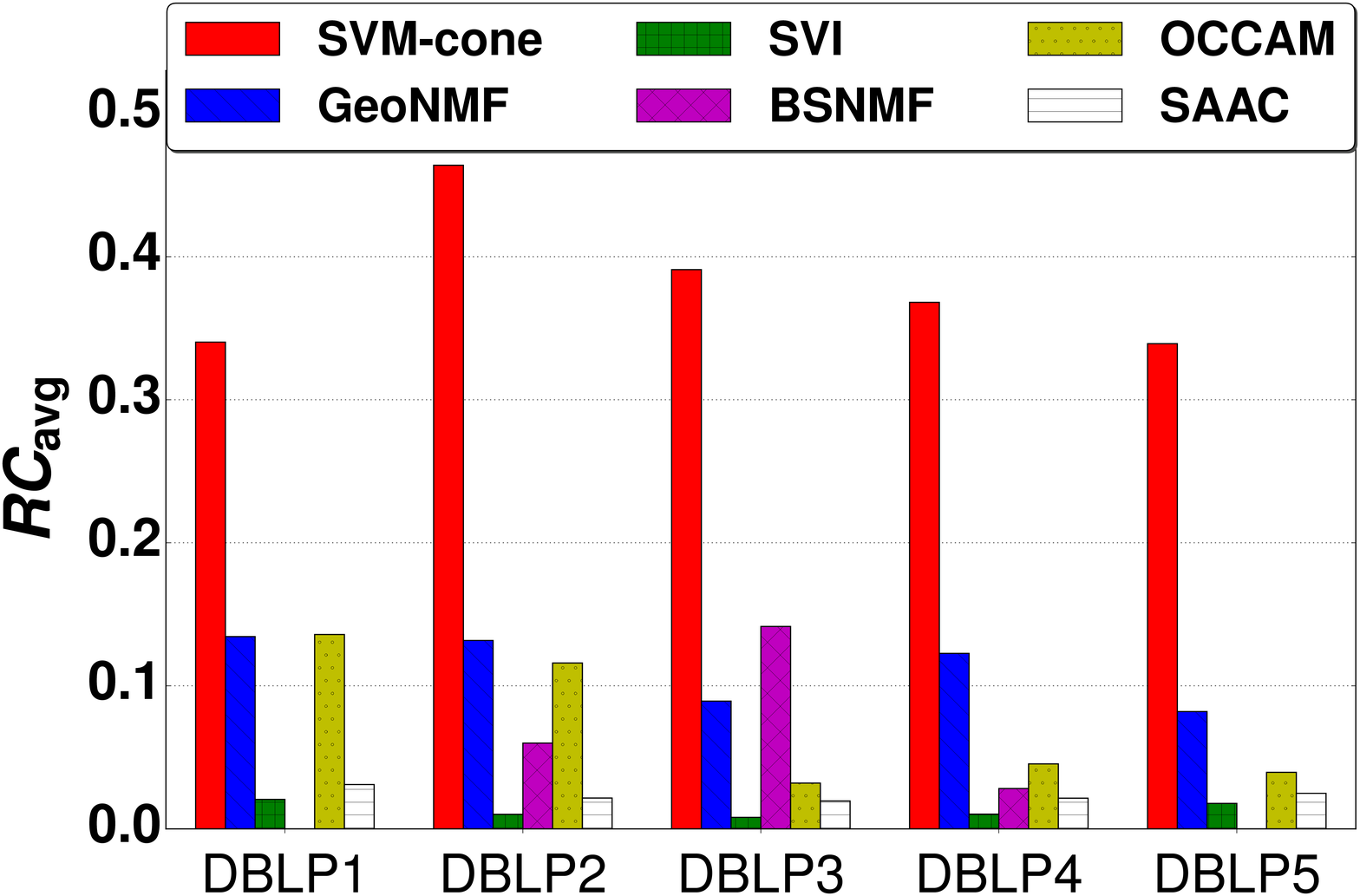}
		\caption{DBLP bipartite}
		\label{dblp:bip}
	\end{subfigure}
	\begin{subfigure}[b]{0.49\textwidth}
		\includegraphics[width=\textwidth]{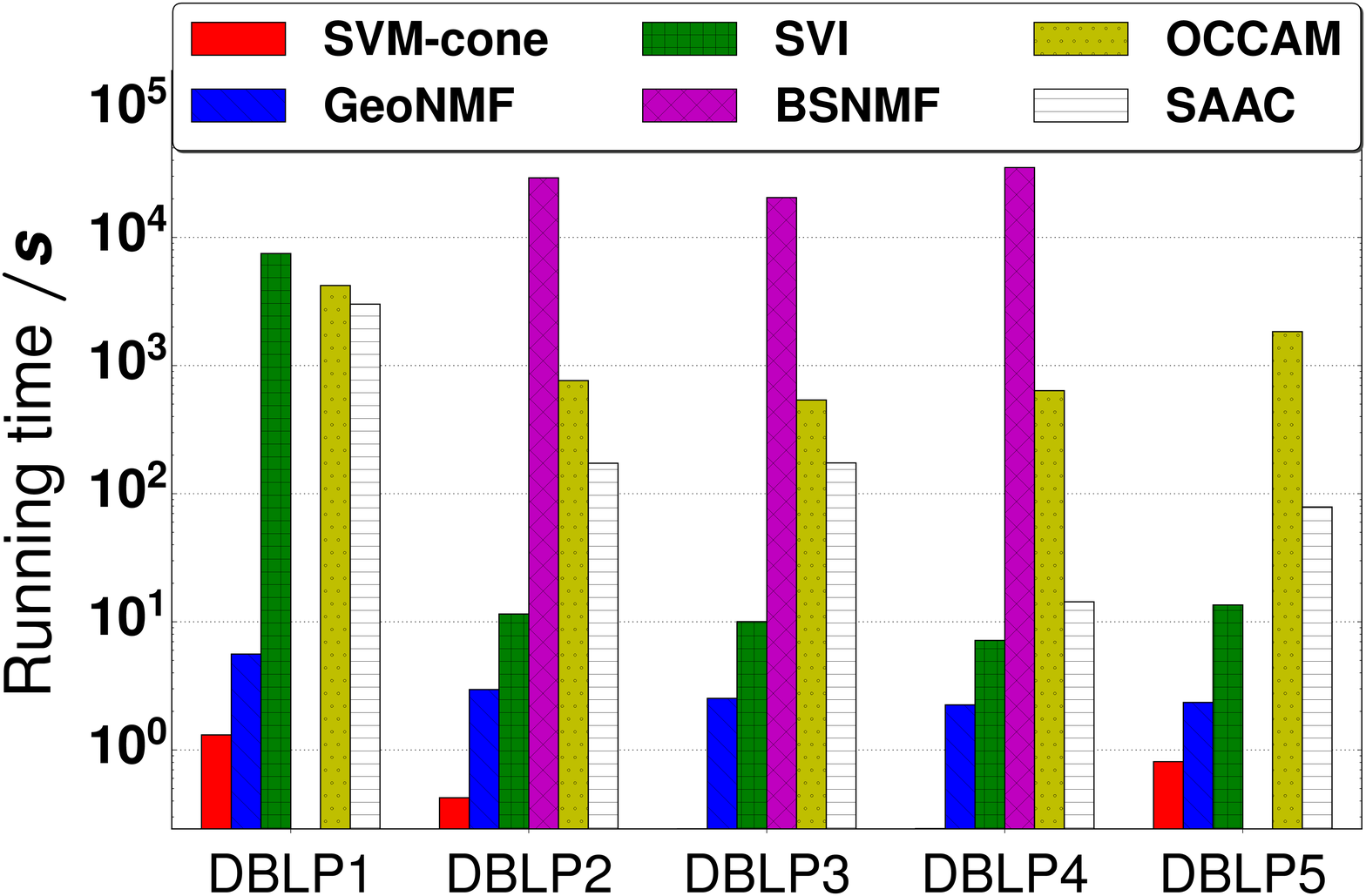}
		\caption{DBLP bipartite wall-clock time}
		\label{dblp:bi_time}
	\end{subfigure}
	\caption{Accuracy of estimated community memberships for (a) the DBLP coauthorship network and (c) the biparite author-paper DBLP network. 
		(b) and (d) The wall-clock time of the competing methods respectively.}
	\label{fig:realworld}
\end{figure}
\fussy
We find that \svmcone outperforms competing baselines on $2$ of the $5$ DBLP coauthorship datasets, and is similar on the remaining three (Figure~\ref{dblp:coauth}).
The closest competitor is GeoNMF~\cite{mao2017}, which assumes that all nodes have the same degree parameter, and the community sizes are balanced.
Both assumptions are reasonable for the dataset, since the number of coauthors (the degree) does not vary significantly among authors, and the communities are formed from conferences where no one conference dominates the others.
The differences between \svmcone and the competition is starker on the bipartite dataset (Figure~\ref{dblp:bip}).
There is severe degree heterogeneity: an author can be connected to many papers, while each paper only has a few authors at best.
Our method is able to accommodate such differences between the nodes, and hence yields much better accuracy than others.

Finally, Figure~\ref{dblp:time} and \ref{dblp:bi_time} shows the wall-clock time for running the various methods on DBLP coauthorship networks and DBLP bipartite author-paper networks respectively. 
Our method is among the fastest.
This is expected; the only computationally intensive step is the one-class SVM and top-$K$ eigen-decomposition (or SVD), for which off-the-shelf efficient and scalable implementations already exist \citep{chang2011libsvm}.
\subsection{Topic Models}
We generate semi-synthetic data following \cite{arora2013practical} and \cite{bansal2014provable} using {\bf NIPS}\footnotemark[1]\footnotetext[1]{\url{https://archive.ics.uci.edu/ml/datasets/Bag+of+Words}}, {\bf New York Times}\footnotemark[1] (NYT), {\bf  PubMed}\footnotemark[1], and {\bf 20NewsGroup}\footnotemark[2]\footnotetext[2]{\url{http://qwone.com/~jason/20Newsgroups/}} (20NG). {Dataset statistics are included in \supp.}
We use Matlab R2018a built-in Gibbs Sampling function for learning topic models to learn the word by topic matrix, which should retain the characteristics of real data distributions. Then we draw the topic-document matrix from Dirichlet with symmetric hyper-parameter 0.01. We set $K=40$ for the first 3 datasets and $K=20$ for 20NG. The word counts matrix is sampled with $N=1000, 300, 100, 200$ respectively, which matches the mean document length of the real datasets. We evaluate the performance of different algorithms using $\ell_1$ reconstruction error $\frac{1}{K}\sum_{i,j}|\bT\bbb{i,j}-\hat{\bT}\bbb{i,\pi(j)}|$, where $\pi(.)$ is a permutation function that matches the topics. 
Table \ref{table:semi_results} shows the $\ell_1$ reconstruction error and wall-clock running time of different algorithms with datasets generated from different number of documents. {Each setting is repeated 5 times, and we report the mean and standard deviation of the results. SVM-cone is much faster than the other methods. Its accuracy is comparable to RecoverL2, and significantly better than TSVD and GDM.}
The \suppNoThe \ also shows the top-10 words of 5 topics learned from SVM-cone for each dataset.
\begin{table}[!ht]
	\caption{$\ell_1$ reconstruction error and wall-clock time on semi-synthetic datasets}
	\label{table:semi_results}
	\begin{center}
		\resizebox{\linewidth}{!}{
		\begin{tabular}{|c|c|c|c|c|c|c|}
			\hline
			Corpus & Documents & & RecoverL2 & TSVD & GDM & SVM-cone\\
			\hline
			\hline
			\multirow{6}{*}{NIPS}&\multirow{2}{*}{20000}&$\ell_1$ Error& {\bf 0.059 \textbf{ ($\pm$} 0.000)}& 0.237  ($\pm$ 0.017)& 0.081 ($\pm$ 0.057)&{ 0.071  ($\pm$ 0.004)}\\
			\cline{3-7} 
			& & Time/$s$ &100.11 ($\pm$ 8.81) & 18.54 ($\pm$ 2.04) & 119.66 ($\pm$ 4.41)& {\bf 5.33 ($\pm$ 0.39)}\\
			\cline{2-7}
			&\multirow{2}{*}{40000}&$\ell_1$ Error& {\bf 0.043 ($\pm$ 0.000)} &  0.250 ($\pm$ 0.045) &0.061 ($\pm$ 0.038)& 0.051 ($\pm$ 0.002)\\
			\cline{3-7}
			& & Time/$s$ & 143.34 ($\pm$ 0.53) & 21.97 ($\pm$ 1.49) & 220.92  ($\pm$ 3.10)& {\bf 9.07 ($\pm$ 0.00)}\\
			\cline{2-7}
			&\multirow{2}{*}{60000}&$\ell_1$ Error& {\bf 0.036 ($\pm$ 0.000)}& 0.269 ($\pm$ 0.064)& 0.059 ($\pm$ 0.038)&{ 0.041 ($\pm$ 0.002)}\\
			\cline{3-7}
			& & Time/$s$ & 247.34 ($\pm$ 20.84) & 35.77 ($\pm$ 3.28) & 406.87  ($\pm$ 36.57) &{\bf 17.63 ($\pm$ 5.29)}\\
			\hline \hline
			\multirow{6}{*}{NYT}&\multirow{2}{*}{20000}&$\ell_1$ Error &{\bf 0.125 ($\pm$ 0.000)}& 0.207 ($\pm$ 0.025)& 0.223($\pm$ 0.008)& 0.131 ($\pm$ 0.003)\\
			\cline{3-7}
			& & Time/$s$ & 78.15 ($\pm$ 7.14) &  25.11 ($\pm$ 6.39) &193.43 ($\pm$ 12.02)&  {\bf 4.51 ($\pm$ 0.70)}\\
			\cline{2-7}
			&\multirow{2}{*}{40000}&$\ell_1$ Error& {\bf 0.103  ($\pm$ 0.000)}& 0.197  ($\pm$ 0.045)&0.216 ($\pm$ 0.010)& 0.106  ($\pm$ 0.001)\\
			\cline{3-7}
			& & Time/$s$ &  140.84  ($\pm$ 15.50)& 50.18  ($\pm$ 13.14)&394.16  ($\pm$ 30.42)&{\bf 8.04 ($\pm$ 1.15)}\\
			\cline{2-7}
			&\multirow{2}{*}{60000}&$\ell_1$ Error& {\bf 0.095 ($\pm$ 0.000)}& 0.166 ($\pm$ 0.028)&0.210 ($\pm$ 0.010)& 0.096 ($\pm$ 0.002)\\
			\cline{3-7}
			& & Time/$s$ & 184.69 ($\pm$ 20.65) & 42.96 ($\pm$ 7.95) & 595.54 ($\pm$ 91.57)&{\bf 11.82 ($\pm$ 1.91)}\\
			\hline \hline
			
			\multirow{6}{*}{PubMed}&\multirow{2}{*}{20000}&$\ell_1$ Error& {\bf 0.163 ($\pm$ 0.000)}& 0.239 ($\pm$ 0.032)&0.277 ($\pm$ 0.051)& 0.181 ($\pm$ 0.002)\\
			\cline{3-7}
			& & Time/$s$ & 54.32 ($\pm$ 5.94) & 15.75 ($\pm$ 2.34) &205.95  ($\pm$ 11.27)&{\bf 2.06 ($\pm$ 0.46)}\\
			\cline{2-7}
			&\multirow{2}{*}{40000}&$\ell_1$ Error& {\bf 0.122 ($\pm$ 0.000)}& 0.255  ($\pm$ 0.018)&0.251  ($\pm$ 0.041)&0.138 ($\pm$ 0.001)\\
			\cline{3-7}
			& & Time/$s$ & 78.99 ($\pm$ 9.99)& 26.44  ($\pm$ 4.49)&459.17 ($\pm$ 30.71)& {\bf 3.73 ($\pm$ 0.37)}\\
			\cline{2-7}
			&\multirow{2}{*}{60000}&$\ell_1$ Error& {\bf 0.098 ($\pm$ 0.000)}& 0.275 ($\pm$ 0.041)&0.269 ($\pm$ 0.052)& 0.114 ($\pm$ 0.001)\\
			\cline{3-7}
			& & Time/$s$ & 98.19 ($\pm$ 15.06) & 24.57 ($\pm$ 4.59) &649.97 ($\pm$ 26.48)& {\bf 5.44 ($\pm$ 0.38)}\\
			\hline \hline
			
			\multirow{6}{*}{20NG}&\multirow{2}{*}{20000}&$\ell_1$ Error& 0.100 ($\pm$ 0.000)& 0.111 ($\pm$ 0.051)&0.137 ($\pm$ 0.071)& {\bf 0.090 ($\pm$ 0.003)}\\
			\cline{3-7}
			& & Time/$s$ & 40.74  ($\pm$ 0.64)&7.51  ($\pm$ 0.42)&102.86 ($\pm$ 4.05)& {\bf 1.85 ($\pm$ 0.26)}\\
			\cline{2-7}
			&\multirow{2}{*}{40000}&$\ell_1$ Error& 0.074 ($\pm$ 0.000)& 0.081 ($\pm$ 0.043)&0.131 ($\pm$ 0.072)& {\bf 0.064 ($\pm$ 0.001)}\\
			\cline{3-7}
			& & Time/$s$ & 94.42 ($\pm$ 9.92)&  16.04 ($\pm$ 2.28)&273.51 ($\pm$ 16.45)& {\bf 4.33 ($\pm$ 0.71)}\\
			\cline{2-7}
			&\multirow{2}{*}{60000}&$\ell_1$ Error& 0.058  ($\pm$ 0.000)& 0.133 ($\pm$ 0.045) & 0.096 ($\pm$ 0.063)&{\bf 0.052 ($\pm$ 0.002)}\\
			\cline{3-7}
			& & Time/$s$ & 142.34  ($\pm$ 20.31)& 23.36  ($\pm$ 5.85)&388.47 ($\pm$ 43.22)& {\bf 5.89 ($\pm$ 0.67)}\\
			\hline
		\end{tabular}
	}
	\end{center}
	\label{tab:multicol}
\end{table}

\section{Conclusions}
\label{sec:conc}

We showed that many distinct models for overlapping clustering can be placed under one general framework, where the data matrix is a noisy version of an ideal matrix and each row is a non-negative weighted sum of ``exemplars.''
In other words, the connection probabilities of one node to others in a network is a non-negative combination of the connection probabilities of $K$ ``pure'' nodes to others in the network.
Each pure node is an examplar of a single community, and we require one pure node from each of the $K$ communities.
This geometrically corresponds to a cone, with the pure nodes being its corners.
This subsumes Mixed-Membership Stochastic Blockmodels and their degree-corrected variants, as well as commonly used topic models.
We showed that a one-class SVM applied to the normalized rows of the data matrix can find both the corners and the weight matrix.
We proved the consistency of our \svmcone algorithm, and used it to develop consistent parameter inference methods for several widely used network and topic models.
Experiments on simulated and large real-world datasets show both the accuracy and the scalability of \svmcone.

\section*{Acknowledgments}
X.M. and P.S. were partially supported by NSF grant DMS 1713082. D.C. was partially supported by a Facebook Faculty Research Award.
\small{
	\bibliographystyle{plainnat}
	\bibliography{references}
}

\newpage
\section*{Appendix}
\appendix
\section{Geometric structure of normalized points from a cone}
\begin{lem}\label{lem:piup}
	Let $\trowP_i=\rowP_i/\|\rowP_i\|$, then $\trowP_i^T=r_i\rowcvxY_i^T\tmtxPp$ for $r_i=\frac{\bbm_i^T \bone}{\|\rowcvxP_i^T\tmtxPp\|}\geq 1$, and $\rowcvxY_i=(\elecvxY_{i1},\elecvxY_{i2},\cdots,\elecvxY_{iK})^T$, $\elecvxY_{ij}=\frac{\elecvxP_{ij}}{\sum_j\elecvxP_{ij}}$. 
	\begin{proof}
		$\trowP_i^T=\frac{\rowP_i^T}{\|\rowP_i\|}
		=\frac{\rowcvxP_i^T\tmtxPp}{\|\rowcvxP_i^T\tmtxPp\|}
		=\frac{\bbm_i^T \bone}{\|\rowcvxP_i^T\tmtxPp\|}\frac{\rowcvxP_i^T}{\bbm_i^T \bone}\tmtxPp
		=r_i\rowcvxY_i^T\tmtxPp$. Clearly $\|\rowcvxP_i^T\tmtxPp\|=\|\sum_j\elecvxP_{ij}\trowP_{I(j)}\|\leq\sum_j\elecvxP_{ij}\|\trowP_{I(j)}\|=\sum_j\elecvxP_{ij}=\bbm_i^T \bone$, so $r_i\geq 1$. 
	\end{proof}
\end{lem}

	\begin{proof}[Proof of Lemma~\ref{lem:vtheta}]
		Since $\rank(\bP)=K$, we have $\bV\bE\bV^T=\bP=\rho\bGamma\bTheta\bB\bTheta^T\bGamma$. W.L.O.G, let $\bTheta(I,:)=\bI$, then 
		$\bV_P\bE\bV^T=\rho\bGamma_P\bB\bTheta^T\bGamma$. Now $\bV\bE=\bP\bV=\rho\bGamma\bTheta\bB\bTheta^T\bGamma\bV=\bGamma\bTheta(\rho\bB\bTheta^T\bGamma)\bV=\bGamma\bTheta(\bGamma_P^{-1}\bV_P\bE\bV^T)\bV=\bGamma \bTheta \bGamma_P^{-1}\bV_P\bE$, right multiplying $\bE^{-1}$ gives $\bV = \bGamma \bTheta \bGamma_P^{-1}\bV_P$. Also consider that  $\bV_P\bE\bV_P^T=\rho\bGamma_P\bB\bGamma_P$, $\bV_P$ is full rank.
	\end{proof}

\section{Identifiability of DCMMSB-type Models}
\begin{lem}
For DCMMSB-type models such that $f(\btheta_i)=1$, $\forall i \in [n]$ for some degree 1 homogeneous function $f$ (e.g., $f(\btheta)=\|\btheta\|_p$), the sufficient conditions for $(\bTheta, \bB, \bGamma)$ to be identifiable up to a permutation of the communities are
(a) there is at least one pure node in each community, (b) $\sum_i \gamma_i=n$, (c) $\bB$ has unit diagonal.
\begin{proof}
	From Lema~\ref{lem:vtheta} we have $\bV = \bGamma \bTheta \bGamma_P^{-1}\bV_P$ and $\bV_P$ is full rank. Suppose two set of parameters $\{\bGamma^{(1)},\bTheta^{(1)},\bB^{(1)}\}$ and $\{\bGamma^{(2)},\bTheta^{(2)},\bB^{(2)}\}$ yield the same $\bP$ (W.L.O.G., we abort $\rho$ in $\bB$) and each has a pure node set $P_1$ and $P_2$ and W.L.O.G., assume the permutation of the communities is fixed, i.e., $\bTheta_{P_1}^{(1)}=\bTheta_{P_2}^{(2)}=\bI$. Then,
	\ba{\label{eq:iden}
		\bGamma^{(1)} \bTheta^{(1)} (\bGamma_{P_1}^{(1)})^{-1}\bV_{P_1} = \bV = \bGamma^{(2)} \bTheta^{(2)} (\bGamma_{P_2}^{(2)})^{-1}\bV_{P_2}. 
	}
	Taking indices $P_1$ and $P_2$ respectively on $\bV$, we have,
	\ba{\label{eq:vp12}
		\bV_{P_1} = \bGamma_{P_1}^{(2)} \bTheta_{P_1}^{(2)} (\bGamma_{P_2}^{(2)})^{-1}\bV_{P_2}\quad\mathrm{and}\quad \bV_{P_2} = \bGamma_{P_2}^{(1)} \bTheta_{P_2}^{(1)} (\bGamma_{P_1}^{(1)})^{-1}\bV_{P_1}.
	}
	Then,
	\ba{\label{eq:identity}
		\bV_{P_1} = &\bGamma_{P_1}^{(2)} \bTheta_{P_1}^{(2)} (\bGamma_{P_2}^{(2)})^{-1}\bGamma_{P_2}^{(1)} \bTheta_{P_2}^{(1)} (\bGamma_{P_1}^{(1)})^{-1}\bV_{P_1}\nonumber\\
		\Longrightarrow\quad\quad \bI=&\bGamma_{P_1}^{(2)} \bTheta_{P_1}^{(2)} (\bGamma_{P_2}^{(2)})^{-1}\bGamma_{P_2}^{(1)} \bTheta_{P_2}^{(1)} (\bGamma_{P_1}^{(1)})^{-1},\quad\text{as $\bV_{P_1}$ is full rank.}
	}
	As $\bGamma_{P_1}^{(2)} \bTheta_{P_1}^{(2)} (\bGamma_{P_2}^{(2)})^{-1}$ and $\bGamma_{P_2}^{(1)} \bTheta_{P_2}^{(1)} (\bGamma_{P_1}^{(1)})^{-1}$ are all nonnegative, using Lemma~1.1 of \citep{minc1988nonnegative}, they are both generalized permutation matrices. Also since $\bGamma_{P_1}^{(2)}$, $(\bGamma_{P_2}^{(2)})^{-1}$ are diagonal matrix, $\bTheta_{P_1}^{(2)}$ must be a permutation matrix as $f(\btheta^{(2)}_i)=1$, $\forall i \in [n]$, and $f$ is homogeneous with degree 1. 
	So nodes in $P_1$ are also pure nodes in $\bTheta^{(2)}$. With same arguments, nodes in $P_2$ are also pure nodes in $\bTheta^{(1)}$. So the pure nodes match up.
	
	Now since $\bV_P\bE\bV_P^T=\bGamma_P\bB\bGamma_P$, we have $\bGamma_{P_1}^{(1)}\bB^{(1)}\bGamma_{P_1}^{(1)}=\bV_{P_1}\bE\bV_{P_1}=\bGamma_{P_1}^{(2)}\bB^{(2)}\bGamma_{P_1}^{(2)}$. As $\bB^{(1)}$ and $\bB^{(2)}$ both have unit diagonal, we must have $\bGamma_{P_1}^{(1)}=c\bGamma_{P_1}^{(2)}$ for $c=\sqrt{\bB^{(2)}_{11}/\bB^{(1)}_{11}}$. Now substituting $P_2$ with $P_1$ in Eq.~\eqref{eq:iden}, and using $\bV_{P_1}$ has full rank, we have,
	\bas{
		\bGamma^{(2)} \bTheta^{(2)} (\bGamma_{P_1}^{(2)})^{-1} = \bGamma^{(1)} \bTheta^{(1)} (\bGamma_{P_1}^{(1)})^{-1} = \bGamma^{(1)} \bTheta^{(1)} (\bGamma_{P_1}^{(2)})^{-1}/c,
	}
	which gives $\bGamma^{(1)} \bTheta^{(1)}=c\bGamma^{(2)} \bTheta^{(2)}$, applying $f(\cdot)$ to rows' transpose on both side,  since $f(\btheta^{(1)}_i)=f(\btheta^{(2)}_i)=1$, $\forall i \in [n]$, and $f$ is homogeneous with degree 1, we have $\bGamma^{(1)} =c\bGamma^{(2)}$.   
	Now as $\bone_n^T\bGamma^{(1)}\bone_n=\bone_n^T\bGamma^{(2)}\bone_n=n$ from condition (b), we must have $c=1$, then 
	$\bGamma^{(1)} =\bGamma^{(2)}$,
	and this immediately gives $\bTheta^{(1)}=\bTheta^{(2)}$. Finally,
	$\bGamma_{P_1}^{(1)}\bB^{(1)}\bGamma_{P_1}^{(1)}=\bGamma_{P_1}^{(2)}\bB^{(2)}\bGamma_{P_1}^{(2)}=\bGamma_{P_1}^{(1)}\bB^{(2)}\bGamma_{P_1}^{(1)}$, and this gives $\bB^{(1)}=\bB^{(2)}$.	
\end{proof}
\end{lem}

\section{Algorithms}
In this section we provide the detailed algorithms for parameter estimations of DCMMSB, OCCAM (Algorithm~\ref{algo:dcmmsb_ocsvm}) and Topic Models (Algorithm~\ref{algo:corner_svm}). These algorithms both reply on the one class SVM (Algorithm~\ref{algo:inferCone}) for finding the corner rays and then use those for parameter estimation, the details of which vary from model to model. Note for Algorithm~\ref{algo:dcmmsb_ocsvm}, step~\ref{step:l1_norm} is to normalize rows of $\bTheta$ by $\ell_1$ norm, if we normalize by $\ell_2$ norm, then it can be used for estimation of OCCAM.

		\begin{algorithm}[H] 
			\caption{{ SVM-cone-DCMMSB}}
			\label{algo:dcmmsb_ocsvm}
			\begin{algorithmic}[1]
				\REQUIRE  Adjacency matrix $\bA \in \R^{n\times n}$, number of communities $K$
				\ENSURE  Estimated degree parameters $\hbGamma$, community memberships $\hbTheta$, and community interaction matrix $\hat{\bB}$
				\STATE Get top-$K$ eigen-decomposition of $\bA$ as $\hv\hE\hv^T$
				\STATE Normalize rows of $\hv$ by $\ell_2$ norm 
				\STATE Use SVM-cone to get pure node set $C$ and estimated $\cvxS$
				\STATE $\hv_C=\hV(C,:)$, get $\hat{\bN}_C$ from row norms of $\hv_C$
				\STATE $\hat{\bD}=\sqrt{\diag(\hat{\bN}_C\hv_C\hE\hv_C^T\hat{\bN}_C)}$
				\STATE $\hat{\bF}=\diag(\cvxS\hat{\bD}\bone_K)$
				\STATE $\hbTheta= \hat{\bF}^{-1}\cvxS\hat{\bD}$ \label{step:l1_norm}
				\STATE $\hbGamma=n\hat{\bF}/(\bone_n^T\hat{\bF}\bone)$, $\hbGamma_C=\hbGamma(C,C)$
				\STATE $\bB=\hbGamma_C^{-1}\hv_C\hE\hv_C\hbGamma_C^{-1}$
				\STATE $\bB=\bB/\max_{i,j}\bB_{ij}$
			\end{algorithmic}
		\end{algorithm}

\begin{algorithm}[H] 
	\caption{{SVM-cone-topic}}
	\label{algo:corner_svm}
	\begin{algorithmic}[1]
		\REQUIRE  Word-document count matrix $\w \in \R^{V\times D}$, number of topics $K$
		\ENSURE  Estimated word-topic matrix $\hat{\bT}$
		\STATE Randomly splitting the words in each document to two halves to get $\w_1$ and $\w_2$
		{\STATE Normalize columns of $\w_1$ and $\w_2$ by $\ell_1$ norm to get $\hat{\bA}_1$ and $\hat{\bA}_2$
		\STATE Get top-$K$ SVD of $\bU={\hat{\bA}_1\hat{\bA}_2^T}$ as $\hv\hE\hv^T$}
		\STATE Normalize rows of $\hv$ by $\ell_2$ row norm
		\STATE Use SVM-cone to get pure node set $C$ and estimated $\cvxS$
		\STATE Normalizing columns of $\hat{\bM}$ by $\ell_1$ norm to get $\hat{\bT}$
	\end{algorithmic}
\end{algorithm}

\newpage
\section{Corner finding with \ocsvm with population inputs}
\begin{lem}\label{lem:ocsvm_popul}
	If $\proj_{\conv(\tmtxPp^T)}(\bzero)$ is an interior point in $\conv(\tmtxPp^T)$, then \ocsvm can find all the $K$ corners with $\elecvxP_{ij}=1$ 
	as support vectors given $\trowP_i$, $i\in[n]$ as inputs. And a sufficient condition for this to hold is $(\tmtxPp\tmtxPp^T)^{-1}\bone>\bzero$. 
	\begin{proof}
		The primal problem of \ocsvm in \citep{scholkopf2001estimating} is
		\bas{
			\min\quad &\frac{1}{2}\|\bw\|^2-b \quad \quad s.t.\quad \bw^T\trowP_i\geq b,\ i\in[n].
		}
		First of all note that $b\geq 0$ because if $b<0$, we can always make $b=0$ to satisfy the condition and decrees the value of the object function. From Lemma~\ref{lem:piup}, we have $\trowP_i^T=r_i\rowcvxY_i^T\tmtxP_P$. As $r_i\geq 1$, 
		if there exists $(\bw,b)$ that $\bw^T\trowP_i\geq b,\ i\in I$, we have $\bw^T\trowP_i=r_i\rowcvxY_i^T\tmtxP_P\bw = r_i\sum_{j}\elecvxY_{ij}\bw^T\trowP_{I(j)} \geq r_i b \geq b,\ i\in[n]$. So we can reduce the problem to using points $i\in I$ as inputs. Furthermore, we consider an equivalent primal problem and its dual:
		\ba{\label{eq:primal_dual}
		\mathrm{Primal:}\quad 	&\max\quad b &\quad \mathrm{Dual:}\quad &\min \quad\frac{1}{2}\sum_{i,j}\beta_i\beta_j\trowP_i^T\trowP_j\\
			&s.t.\quad \|\bw\|\leq 1,\  \bw^T\trowP_i\geq b,\ i\in I &\quad &s.t.\quad  \sum_i\beta_i=1,\ \beta_i\geq0,\ i\in I \nonumber
		}
	The dual problem is basically to find a point in $\conv(\tmtxPp^T)$ that has the minimum norm (closest to origin). Now denote the optimal function value for the dual problem as $L_{\tmtxPp}$ and for any subset $\cS\subset I$, let $L_{\tmtxP_{P(\cS,:)}}$ be the optimal value when we want to find a point in $\conv(\tmtxP_{P(\cS,:)}^T)$ that has the minimum norm.
	
	Let {$\bN\in\R^{n\times n}$ be a diagonal matrix such that $\bN_{ii}=1/\|\rowP_i\|$}, then $\tmtxPp=\bN_P\mtxPp$ is also full rank. 
	If for $\bbeta^*=\arg \min_{\bbeta} L_{\tmtxPp}(\bbeta)$, each coordinate is strictly larger than 0, it is easy to see that $L_{\tmtxPp}>L_{\tmtxP_{P(\cS,:)}}$
	since $\tmtxPp$ is full rank. 
	So a sufficient condition for \ocsvm to find all $K$ corners of $L_{\tmtxPp}$ is $\bbeta^*>\bzero$, which means the closet point to origin in $\conv(\tmtxPp^T)$ is an interior point (also the projection of origin to $\conv(\tmtxPp^T)$ ). Now we will show a sufficient condition for this. 
	
	Suppose the $\bbeta^*>\bzero$. First let us find a hyperplane $(\bw, d)$ that is through columns of $\tmtxPp^T$ with $d< 0$ (since $\tmtxPp$ is full rank, we must have $d\neq 0$). We have $\tmtxPp\bw=d\bone$. 
	Since the distance from origin to hyperplane $(\bw, d)$ is $\frac{|d|}{\|\bw\|}$, $\proj_{\conv(\tmtxPp^T)}(\bzero)$ is an interior point in $\conv(\tmtxPp^T)$, we have
	\ba{\label{eq:ypbeta}
		\tmtxPp^T\bbeta^*=\proj_{\conv(\tmtxPp^T)}(\bzero)=\frac{d}{\|\bw\|}\frac{\bw}{\|\bw\|}
	}
	Then,
	\bas{
		\bw^T\tmtxPp^T\bbeta^*=\frac{d\bw^T\bw}{\|\bw\|^2}=d.
	}
	As $\bw^T\tmtxPp^T=d\bone^T$, we have $d\bone^T\bbeta^*=d$, so $\bone^T\bbeta^*=1$.
	So the only condition left to be satisfied is that $\bbeta^*>\bzero$, using Eq.~\eqref{eq:ypbeta},
	\bas{
		\tmtxPp\tmtxPp^T\bbeta^*=\frac{d\tmtxPp\bw}{\|\bw\|^2}=\frac{d(d\bone)}{\|\bw\|^2},
	}
	so $\bbeta^*=\frac{d^2}{\|\bw\|^2}(\tmtxPp\tmtxPp^T)^{-1}\bone>\bzero$ and all we require is:
	\bas{
		(\tmtxPp\tmtxPp^T)^{-1}\bone>\bzero.
	}
	\end{proof}
\end{lem}

	\begin{proof}[Proof of Theorem~\ref{thm:dcmmsb_condition_true}]
		Using Lemma~\ref{lem:vtheta}, we have:
		\ba{\label{eq:original_vpvpt}
			\bI=\bV^T\bV=\bV_P^T\bGamma_P^{-1}\bTheta^T\bGamma^2 \bTheta \bGamma_P^{-1}\bV_P
			\quad \Longrightarrow \quad
			(\bV_P\bV_P^T)^{-1}=\bGamma_P^{-1}\bTheta^T\bGamma^2 \bTheta \bGamma_P^{-1}.
		}
		Since $\tmtxPp=\bN_P\bV_P$, we have:
		\ba{\label{eq:vpvpT}
			(\tmtxPp\tmtxPp^T)^{-1}=\bN_P^{-1}\bGamma_P^{-1}\bTheta^T\bGamma^2 \bTheta \bGamma_P^{-1}\bN_P^{-1}.
		}
		On the RHS of Eq.~\eqref{eq:vpvpT}, as $\bN_P^{-1}$, $\bGamma_P^{-1}$ and $\bGamma$ are all diagonal matrix with strictly positive diagonal elements, then diagonal of $(\tmtxPp\tmtxPp^T)^{-1}$ must be strictly positive, as the $i$-th element on its diagonal is proportional to $\|\bGamma\bTheta(:,i)\|^2$, and since $\bTheta$ is nonnegative, we can easily get that $(\tmtxPp\tmtxPp^T)^{-1}\bone>0$. So for DCMMSB-type models, it is always true that the closet point in $\conv(\tmtxPp^T)$ to origin is an interior point of $\conv(\tmtxPp^T)$.
	\end{proof}

\section{Corner finding with \ocsvm with empirical inputs}
\begin{lem}\label{lem:b_error}
	Let $\epsilon=\max_i \|\trowP_i-\trowS_i\|$. Denote $(\bw,b)$ and $(\hww,\hb)$ be the optimal solution for the primal problem of \ocsvm in \eqref{eq:primal_dual} with population ($\trowP_1,\trowP_2,\cdots,\trowP_n$) and empirical inputs ($\trowS_1,\trowS_2,\cdots,\trowS_n$) respectively, then $|\hb-b|\leq \epsilon$.
	\begin{proof}
		First we have $\bw^T\trowP_i\geq b$, $\forall i\in[n]$, and $\|\bw^T(\trowS_i-\trowP_i)\|\leq\epsilon$. Then $\bw^T\trowS_i=\bw^T\trowP_i+\bw^T(\trowS_i-\trowP_i)\geq b-\epsilon$. As $(\bw, b-\epsilon)$ is a feasible solution of the primal problem with empirical inputs, by optimality of $\hb$, we have $\hb\geq b-\epsilon$. Similarly we can get $b\geq \hb-\epsilon$, so $|\hb-b|\leq \epsilon$.
	\end{proof}
\end{lem}

\begin{lem}\label{lem:w_bound}
	Let $(\bw,b)$, $(\hww,\hb)$ be the hyperplane of the optimal solution of \ocsvm with population and empirical inputs respectively, then $\|\hww-\bw\|\leq \zeta\epsilon$, for $\zeta=\frac{4}{\eta b^2\sqrt{\lambda_K(\tmtxP_P\tmtxP_P^T)}}\leq \frac{4K}{\eta(\lambda_K(\tmtxP_P\tmtxP_P^T))^{1.5}}$. 
	\begin{proof}
		Let $\beta_l$, $l\in I$ be the solution of the dual problem in Eq.~\eqref{eq:primal_dual} with population inputs, from the construction of this dual problem, we know $\bw=\frac{\sum_{l\in I}\beta_l\trowP_l}{\|\sum_{l\in I}\beta_l\trowP_l\|}$, $\|\sum_{l\in I}\beta_l\trowP_l\| =b$, and $\bbeta:=(\beta_{I(1)},\beta_{I(2)},\cdots,\beta_{I(p)})=b^2(\tmtxPp\tmtxPp^T)^{-1}\bone$, as shown in Lemma~\ref{lem:ocsvm_popul}. So $\bw=\tmtxP_P^T\bbeta/b=b\tmtxP_P^T(\tmtxPp\tmtxPp^T)^{-1}\bone$. From the condition of the primal problem, $\tmtxS_P\hww\geq\hb\bone$, then we have $\tmtxP_P\hww=\tmtxS_P\hww-(\tmtxS_P-\tmtxPp)\hww\geq(\hb-\epsilon)\bone\geq(b-2\epsilon)\bone$. Then there exists a vector $\bc\geq\bzero$ such that $\tmtxP_P\hww=(b-2\epsilon)\bone+\bc$. Now let $\hww=\tmtxP_P^T\bvarphi+\hww_{\bot}$, where $\tmtxP_P\hww_{\bot}=\bzero$. So $\tmtxP_P\hww=\tmtxP_P\tmtxP_P^T\bvarphi=(b-2\epsilon)\bone+\bc$, which gives $\hww=\tmtxP_P^T(\tmtxP_P\tmtxP_P^T)^{-1}((b-2\epsilon)\bone+\bc)+\hww_{\bot}$. Since $\|\hww\|=1$, we have
		\bas{
			1&=\|\hww\|^2=((b-2\epsilon)\bone+\bc)^T(\tmtxP_P\tmtxP_P^T)^{-1}((b-2\epsilon)\bone+\bc)+\|\hww_{\bot}\|^2\\
			&=b^2\bone^T(\tmtxP_P\tmtxP_P^T)^{-1}\bone+2b\bone^T(\tmtxP_P\tmtxP_P^T)^{-1}(\bc-2\epsilon\bone)+(\bc-2\epsilon\bone)^T(\tmtxP_P\tmtxP_P^T)^{-1}(\bc-2\epsilon\bone)+\|\hww_{\bot}\|^2.
		}
		Since $1=\|\bw\|^2=b^2\bone^T(\tmtxP_P\tmtxP_P^T)^{-1}\bone$, we have
		\ba{\label{eq:cross_positive}
			0\leq(\bc-2\epsilon\bone)^T(\tmtxP_P\tmtxP_P^T)^{-1}(\bc-2\epsilon\bone)+\|\hww_{\bot}\|^2
			&=-2b\bone^T(\tmtxP_P\tmtxP_P^T)^{-1}(\bc-2\epsilon\bone)\\
			&=-2b\bone^T(\tmtxP_P\tmtxP_P^T)^{-1}\bc+4b\epsilon\bone^T(\tmtxP_P\tmtxP_P^T)^{-1}\bone\nonumber,
		}
		which uses that $(\tmtxPp\tmtxPp^T)^{-1}$ is positive definite. This gives
		\bas{
			&2b\bone^T(\tmtxP_P\tmtxP_P^T)^{-1}\bc\leq 4b\epsilon\bone^T(\tmtxP_P\tmtxP_P^T)^{-1}\bone=4b\epsilon/b^2  \\
			\Longrightarrow \quad &(\min_i \bone^T(\tmtxP_P\tmtxP_P^T)^{-1}\be_i) \|\bc\|_1\leq \bone^T(\tmtxP_P\tmtxP_P^T)^{-1}\bc \leq 2\epsilon/b^2, 		
		}
		and by Condition~\ref{cond:emp} we know $(\min_i \bone^T(\tmtxP_P\tmtxP_P^T)^{-1}\be_i)\geq\eta$, so $\|\bc\|\leq\|\bc\|_1\leq 2\epsilon/(\eta b^2)$. 
		
		Let $\hat{P}$ be the set of support vectors returned by empirical \ocsvm, and $\hbbeta$ as the optimal solution for the dual problem, then $\hww = \tmtxS_{\hat{P}} \hbbeta/\hb$ and $\sum_{j\in\hat{P}}\hbeta_j=1$. 
		Now we will give an upper bound on $\|\hww_{\bot}\|$. 
		For any ${\boldsymbol v} \in \mathrm{span}(\tmtxPp)$, we have 
		$\|{\hww}_{\bot}\|\leq \|{\hww}-{\boldsymbol v} \|$. Now take ${\boldsymbol v} =\tmtxP_{\hat{P}}^T\hbbeta/\hat{b}$, 
		since all rows of $\tmtxP$ lie in the span of $\tmtxPp$, this choice of ${\boldsymbol v} $ also lies in the span of $\tmtxPp$.
		Thus, 
		\bas{
			\|{\hww}_{\bot}\|\leq \|{\hww}-{\boldsymbol v} \|=\|\tmtxS_{\hat{P}}^T\hbbeta-\tmtxP_{\hat{P}}^T\hbbeta\|/\hb
			=\|\sum_{j\in\hat{P}}\hbeta_j(\trowP_j-\trowS_j)\|/\hb  \leq \epsilon/(b-\epsilon).
}
		Now, we have
		\bas{
			\hww-\bw&=\tmtxP_P^T(\tmtxP_P\tmtxP_P^T)^{-1}((b-2\epsilon)\bone+\bc)+\hww_{\bot}-b\tmtxP_P^T(\tmtxPp\tmtxPp^T)^{-1}\bone=\tmtxP_P^T(\tmtxP_P\tmtxP_P^T)^{-1}(\bc-2\epsilon\bone)+\hww_{\bot},\\
			\|\hww-\bw\|^2&=(\bc-2\epsilon\bone)^T(\tmtxP_P\tmtxP_P^T)^{-1}(\bc-2\epsilon\bone)+\|\hww_{\bot}\|^2
			\leq\bc^T(\tmtxP_P\tmtxP_P^T)^{-1}\bc + 4\epsilon^2/b^2+\epsilon^2/(b-\epsilon)^2\\
			&\leq \|\bc\|^2\lambda_1((\tmtxP_P\tmtxP_P^T)^{-1})+ 4\epsilon^2/b^2+\epsilon^2/(b-\epsilon)^2
			\leq \bbb{\frac{4}{\eta^2b^4\lambda_K(\tmtxP_P\tmtxP_P^T)}+\frac{4}{b^2}+\frac{1}{(b-\epsilon)^2}}\epsilon^2,
		}
		where we use Eq.~\eqref{eq:cross_positive} to get that the cross terms are non-negative for the first inequality.
		First $\frac{4}{\eta^2b^4\lambda_K(\tmtxP_P\tmtxP_P^T)}+\frac{4}{b^2}+\frac{1}{(b-\epsilon)^2}<\frac{4}{\eta^2b^4\lambda_K(\tmtxP_P\tmtxP_P^T)}+\frac{8}{b^2}<\frac{12}{\eta^2b^4\lambda_K(\tmtxP_P\tmtxP_P^T)}$, using $\epsilon<b/2$, $\eta<1$, $b\leq 1$, and $\lambda_K(\tmtxP_P\tmtxP_P^T)<1$. Then by taking $\zeta=\frac{4}{\eta b^2\sqrt{\lambda_K(\tmtxP_P\tmtxP_P^T)}}$, we have $\|\hww-\bw\|\leq \zeta\epsilon$. Furthermore, $\zeta\leq \frac{4K}{\eta(\lambda_K(\tmtxP_P\tmtxP_P^T))^{1.5}}$ by using
		\bas{
			1/b^2=\bone^T(\tmtxP_P\tmtxP_P^T)^{-1}\bone\leq K\lambda_1((\tmtxP_P\tmtxP_P^T)^{-1})=K/\lambda_K(\tmtxP_P\tmtxP_P^T).
		}
	\end{proof}
\end{lem}

\begin{lem}\label{lem:pure_b_bound}
	Let $(\hww,\hb)$ be the hyperplane of the optimal solution of \ocsvm with empirical inputs, then
	$\hb\bone\leq\tmtxSp\hww\leq\hb\bone+(\zeta+2)\epsilon\bone$.
	\begin{proof}
		Using Lemma~\ref{lem:w_bound},
		\bas{
			\tmtxSp\hww=\tmtxPp\hww+(\tmtxSp-\tmtxPp)\hww&\leq \tmtxPp\bw+\tmtxPp(\hww-\bw)+\epsilon\bone\leq b\bone+(\zeta\epsilon+\epsilon)\bone
			\leq \hb\bone+(\zeta+2)\epsilon\bone.
		}
	\end{proof}
\end{lem}

\begin{lem}\label{lem:r_i_bound}
	Let $(\bw,b)$, $(\hww,\hb)$ be the hyperplane of the optimal solution of \ocsvm with population and empirical inputs respectively, 
	and $S$ be the set of nodes selected as support vectors in the optimal solution of the dual problem with empirical inputs. 
	Then for $r_i$ defined in Lemma~\ref{lem:piup}, $r_i-1\leq\frac{1}{b/(2\epsilon)-1}$, $\forall i \in S$. Furthermore, $\forall i\in[n]$, if $\hww^T\trowS_i\leq \hb + (\zeta+2)\epsilon$, then $r_i-1\leq \frac{(\zeta+4)\epsilon}{b-2\epsilon}$.
	\begin{proof}
		First $\forall i \in S$, we have,
		\bas{
			\hb&=\hww^T\trowS_i=\hww^T\trowP_i+\hww^T(\trowS_i-\trowP_i)=r_i\sum_j\elecvxY_{ij}\hww^T\trowP_{I(j)}+\hww^T(\trowS_i-\trowP_i)\\
			&=r_i\sum_j\rowcvxY_{ij}\hww^T\trowS_{I(j)}+r_i\sum_j\rowcvxY_{ij}\hww^T(\trowP_{I(j)}-\trowS_{I(j)})+\hww^T(\trowS_i-\trowP_i)\\
			&\geq r_i\hb - r_i\epsilon -\epsilon.
		}
	This gives
	\bas{
		r_i\leq\frac{\hb+\epsilon}{\hb-\epsilon}\quad\Longrightarrow\quad r_i-1\leq\frac{2\epsilon}{\hb-\epsilon}\leq \frac{2\epsilon}{b-\epsilon-\epsilon}=\frac{1}{b/(2\epsilon)-1},
	}
	where the last step uses $b\geq\hb-\epsilon$ from Lemma~\ref{lem:b_error}. Similarly, for $i\in[n]$ such that $\hww^T\trowS_i\leq \hb + (\zeta+2)\epsilon$, we have $\hb + (\zeta+2)\epsilon \geq r_i\hb - r_i\epsilon -\epsilon$ and this gives $r_i-1\leq \frac{(\zeta+4)\epsilon}{b-2\epsilon}$.
	\end{proof}
\end{lem}

\begin{lem}\label{lem:nearly_pure}
	For $S$ defined in Lemma~\ref{lem:r_i_bound}, $\forall i \in S$, $\exists j\in[K]$ such that for $\elecvxY_{ij}$ defined in Lemma~\ref{lem:piup},  $\rowcvxY_{ij}\geq 1-\epsilon_1$, for $\epsilon_1=\frac{2\epsilon}{b\lambda_K(\tmtxPp\tmtxPp^T)}$. Furthermore, $\forall i\in[n]$, if $\hww^T\trowS_i\leq \hb + (\zeta+2)\epsilon$, then $\exists j\in[K]$, $\elecvxY_{ij}\geq 1-\epsilon_2$, for $\epsilon_2=\frac{(\zeta+4)\epsilon}{(b+(\zeta+2)\epsilon)\lambda_K(\tmtxPp\tmtxPp^T)}< \frac{2\zeta\epsilon}{b\lambda_K(\tmtxPp\tmtxPp^T)}$.
	\begin{proof}
		By Lemma~\ref{lem:r_i_bound} we have 
		$r_i\leq1+\frac{1}{b/(2\epsilon)-1}=\frac{1}{1-2\epsilon/b}$.
		As $\trowP_i=r_i\rowcvxY_i^T\tmtxPp$, we have $1=\|\trowP_i\|=r_i\|\rowcvxY_i^T\tmtxPp\|$, so $\|\rowcvxY_i^T\tmtxPp\|\geq 1-2\epsilon/b$. 
		Let $\trowP_{-k}=\sum_{j\neq k}\frac{\elecvxY_{ij}}{1-\elecvxY_{ik}}\trowP_{I(j)}$, $\forall k\in[K]$. then $\rowcvxY_i^T\tmtxPp=\elecvxY_{ik}\trowP_{I(k)}+(1-\elecvxY_{ik})\trowP_{-k}$. It is easy to see that $\|\trowP_{-k}\|\leq 1$, then
		\bas{
			\|\rowcvxY_i^T\tmtxPp\|^2 &\leq \elecvxY_{ik}^2+(1-\elecvxY_{ik})^2+2\elecvxY_{ik}(1-\elecvxY_{ik})\trowP_{I(k)}^T\trowP_{-k},\\
			\trowP_{I(k)}^T\trowP_{-k}&=\sum_{j\neq k}\frac{\elecvxY_{ij}}{1-\elecvxY_{ik}}\trowP_{I(k)}^T\trowP_{I(j)}\leq \max_{j\neq k} \trowP_{I(k)}^T\trowP_{I(j)} \leq \max_{i\neq l} \trowP_{I(i)}^T\trowP_{I(l)}.
		}
		Using $2\bx_1^T\bx_2=\|\bx_1\|^2+\|\bx_2\|^2-\|\bx_1-\bx_2\|^2$ for any same length vectors $\bx_1$ and $\bx_2$, and
		\bas{
			\|\trowP_{I(i)}-\trowP_{I(l)}\|^2&=\|(\be_i-\be_l)^T\tmtxPp\|^2=(\be_l-\be_j)^T\tmtxPp\tmtxPp^T(\be_i-\be_l)\\
				    	&\geq {2}\min_{\|\bx\|=1}\bx^T\tmtxPp\tmtxPp^T\bx={2}\lambda_K(\tmtxPp\tmtxPp^T),
		}
		we have $\max_{i\neq l} \trowP_{I(i)}^T\trowP_{I(l)}\leq 1-\lambda_K(\tmtxPp\tmtxPp^T)$. Then,
		\bas{
			(1-2\epsilon/b)^2 &\leq \|\rowcvxY_i^T\tmtxPp\|^2 \leq \elecvxY_{ik}^2+(1-\elecvxY_{ik})^2+2\elecvxY_{ik}(1-\elecvxY_{ik})(1-\lambda_K(\tmtxPp\tmtxPp^T))\\
			&= 1-2\elecvxY_{ik}(1-\elecvxY_{ik})\lambda_K(\tmtxPp\tmtxPp^T),
		}
		which gives $\elecvxY_{ik}(1-\elecvxY_{ik})\leq\frac{2\epsilon}{b\lambda_K(\tmtxPp\tmtxPp^T)}:=\epsilon_1$, $\forall k\in[K]$. Since $\sum_k\elecvxY_{ik}=1$, we must have $\exists j\in[K]$, $\rowcvxY_{ij}\geq 1-\epsilon_1$. Similarly, for $i\in[n]$ such that $\hww^T\trowS_i\leq \hb + (\zeta+2)\epsilon$, we have $r_i-1\leq \frac{(\zeta+4)\epsilon}{b-2\epsilon}$ from Lemma~\ref{lem:r_i_bound}, then $\rowcvxY_i^T\tmtxPp=\frac{1}{r_i}\geq 1-\frac{(\zeta+4)\epsilon}{b+(\zeta+2)\epsilon}$, and this gives that $\elecvxY_{ik}(1-\elecvxY_{ik})\leq \frac{(\zeta+4)\epsilon}{(b+(\zeta+2)\epsilon)\lambda_K(\tmtxPp\tmtxPp^T)}:=\epsilon_2< \frac{2\zeta\epsilon}{b\lambda_K(\tmtxPp\tmtxPp^T)}$, using $\zeta\geq4$ and $(\zeta+2)\epsilon\geq 0$.  
		Also since $\sum_k\elecvxY_{ik}=1$, we must have $\exists j\in[K]$, $\rowcvxY_{ij}\geq 1-\epsilon_2$.
	\end{proof}
\end{lem}

\begin{rem}
	Lemma~\ref{lem:nearly_pure} shows that for \ocsvm with empirical inputs, the support vectors selected are all nearly corner points. Lemma~\ref{lem:pure_b_bound} shows that each corner point is closed to the hyperplane $(\hww,\hb)$ selected by \ocsvm by $(\zeta+2)\epsilon$, and then Lemma~\ref{lem:nearly_pure} shows that points close to hyperplane $(\hww,\hb)$ by $(\zeta+2)\epsilon$ are all nearly corner points. So choosing points that are $(\zeta+2)\epsilon$ close to $(\hww,\hb)$ will guarantee us all the $K$ corner points and some nearly corner points.
\end{rem}

\begin{lem}\label{lem:distance}
	Let $S_c=\{i:\hww^T\trowS_i\leq \hb + (\zeta+2)\epsilon\}$, then
	$\forall i,j\in S_c$, for $\epsilon_3=\epsilon+\frac{(\zeta+4)\epsilon}{b-2\epsilon}$, we have $\|\rowcvxY_i-\rowcvxY_j\|\sqrt{\lambda_K(\tmtxPp\tmtxPp^T)}-2\epsilon_3\leq\|\trowS_i-\trowS_j\|\leq \|\rowcvxY_i-\rowcvxY_j\|\sqrt{\lambda_1(\tmtxPp\tmtxPp^T)}+2\epsilon_3$.
	\begin{proof}
		\sloppy
		First we have, 
		${
			\|\trowS_i-\rowcvxY_i^T\tmtxPp\|=\|\trowS_i-r_i\rowcvxY_i^T\tmtxPp+(r_i-1)\rowcvxY_i^T\tmtxPp\|\leq \epsilon+\frac{(\zeta+4)\epsilon}{b-2\epsilon}:=\epsilon_3,
		}$ 
		where last step is by Lemma~\ref{lem:r_i_bound}. This gives $\|(\trowS_i-\trowS_j)-(\rowcvxY_i\tmtxPp-\rowcvxY_j\tmtxPp)\|\leq 2\epsilon_3$, then we have
		${
			\|\rowcvxY_i^T\tmtxPp-\rowcvxY_j^T\tmtxPp\|-2\epsilon_3\leq\|\trowS_i-\trowS_j\|\leq \|\rowcvxY_i^T\tmtxPp-\rowcvxY_j^T\tmtxPp\|+2\epsilon_3
		}$. 
		Combing with
		\bas{
			\|\rowcvxY_i-\rowcvxY_j\|\sqrt{\lambda_K(\tmtxPp\tmtxPp^T)}\leq\|\rowcvxY_i^T\tmtxPp-\rowcvxY_j^T\tmtxPp\|\leq\|\rowcvxY_i-\rowcvxY_j\|\sqrt{\lambda_1(\tmtxPp\tmtxPp^T)},
		}
		we have the result.
		\fussy
	\end{proof}
\end{lem}

\newpage
\begin{lem}\label{lem:p_clusters}
	Let $S_c=\{i:\hww^T\trowS_i\leq \hb + (\zeta+2)\epsilon)\}$, then there exists exact $K$ clusters in $S_c$, given $\epsilon\leq c_{\epsilon}\frac{\eta(\lambda_K(\tmtxPp\tmtxPp^T))^3}{K^{1.5}\sqrt{\kappa(\tmtxPp\tmtxPp^T)}}$, for some constant $c_{\epsilon}$.
	\begin{proof}
		First because $I\in S_c$ from Lemma~\ref{lem:pure_b_bound}, there exists at least $K$ clusters in $S_c$. By Lemma~\ref{lem:nearly_pure}, $\forall i\in S_c$, $\exists k_i\in[K]$, $\elecvxY_{ik_i}\geq 1-\epsilon_2$. If $k_i=k_j$, by Lemma~\ref{lem:distance},
		\bas{
			\|\trowS_i-\trowS_j\|\leq \|\rowcvxY_i-\rowcvxY_j\|\sqrt{\lambda_1(\tmtxPp\tmtxPp^T)}+2\epsilon_3\leq \sqrt{3}\epsilon_2\sqrt{\lambda_1(\tmtxPp\tmtxPp^T)}+2\epsilon_3.
		}
		This means if $j$ is a corner point, $i$ will be close to it, and will be in the same cluster as long as there is enough separation between different clusters. Now we will prove this is true.
		Similarly, if $k_i\neq k_j$,
		\bas{
			\|\trowS_i-\trowS_j\|\geq \|\rowcvxY_i-\rowcvxY_j\|\sqrt{\lambda_K(\tmtxPp\tmtxPp^T)}-2\epsilon_3\geq \sqrt{2}(1-2\epsilon_2)\sqrt{\lambda_K(\tmtxPp\tmtxPp^T)}-2\epsilon_3.
		}
		In order to have enough separation between $p$ clusters, we need 
		\bas{
			\sqrt{2}(1-2\epsilon_2)\sqrt{\lambda_K(\tmtxPp\tmtxPp^T)}-2\epsilon_3&=\sqrt{2}\sqrt{\lambda_K(\tmtxPp\tmtxPp^T)}-2\sqrt{2}\epsilon_2\sqrt{\lambda_K(\tmtxPp\tmtxPp^T)}-2\epsilon_3\\
			&> c'(\sqrt{3}\epsilon_2\sqrt{\lambda_1(\tmtxPp\tmtxPp^T)}+2\epsilon_3),
		} 
		for some constant $c'>2$. This is equivalent to show
		\bas{
			\sqrt{2} > (2\sqrt{2}+\sqrt{3}c'\sqrt{\kappa(\tmtxPp\tmtxPp^T)})\epsilon_2+\frac{2+2c'}{\sqrt{\lambda_K(\tmtxPp\tmtxPp^T)}}\epsilon_3.
		}
		As
		\bas{
			&(2\sqrt{2}+\sqrt{3}c'\sqrt{\kappa(\tmtxPp\tmtxPp^T)})\epsilon_2+\frac{2+2c'}{\sqrt{\lambda_K(\tmtxPp\tmtxPp^T)}}\epsilon_3\\
			\leq& (2\sqrt{2}+\sqrt{3}c'\sqrt{\kappa(\tmtxPp\tmtxPp^T)})\frac{2\zeta\epsilon}{b\lambda_K(\tmtxPp\tmtxPp^T)}+\frac{2+2c'}{\sqrt{\lambda_K(\tmtxPp\tmtxPp^T)}}\left(\epsilon+\frac{(\zeta+4)\epsilon}{b-2\epsilon}\right)\\
			\leq & c_1\frac{\sqrt{\kappa(\tmtxPp\tmtxPp^T)}\zeta\epsilon}{b\lambda_K(\tmtxPp\tmtxPp^T)}+\frac{c_2}{\lambda_K(\tmtxPp\tmtxPp^T)}\frac{\zeta\epsilon}{b}
			\leq c_3\frac{\sqrt{\kappa(\tmtxPp\tmtxPp^T)}\epsilon}{\lambda_K(\tmtxPp\tmtxPp^T)}\frac{4K}{\eta(\lambda_K(\tmtxP_P\tmtxP_P^T))^{1.5}}\frac{\sqrt{K}}{\sqrt{\lambda_K(\tmtxPp\tmtxPp^T)}}\\
			\leq & c_4\frac{K^{1.5}\sqrt{\kappa(\tmtxPp\tmtxPp^T)}}{\eta(\lambda_K(\tmtxPp\tmtxPp^T))^3}\epsilon,
		}
		where $c_i$, $i\in[4]$ are some constants we do not specify and we use $1/b^2\leq K/\lambda_K(\tmtxP_P\tmtxP_P^T)$ in the second last inequality.
		So a sufficient condition for separated clusters is $c_4\frac{K^{1.5}\sqrt{\kappa(\tmtxPp\tmtxPp^T)}}{\eta(\lambda_K(\tmtxPp\tmtxPp^T))^3}\epsilon<\sqrt{2}$, which is
		\bas{
			\epsilon\leq c_{\epsilon}\frac{\eta(\lambda_K(\tmtxPp\tmtxPp^T))^3}{K^{1.5}\sqrt{\kappa(\tmtxPp\tmtxPp^T)}},
		}
		for some constant $c_{\epsilon}$. 
	\end{proof}
\end{lem}

\newpage
\section{Consistency of inferred parameters}
\begin{lem}\label{lem:yp_error}
	For set $C$ returned by Algorithm~\ref{algo:inferCone}, there exits a permutation matrix $\bpi\in \R ^{K\times K}$ that
	${
		\|\tmtxS_{C}-\bpi\tmtxPp\|_F	\leq \epsilon_4, 
	}$ 
	for $\epsilon_4=\frac{c_Y{K}\zeta}{(\lambda_K(\tmtxPp\tmtxPp^T))^{1.5}}\epsilon$ and $c_Y$ is some constant.
	\begin{proof}
		By Lemma~\ref{lem:nearly_pure}, we know that $\forall i \in S_c$, $\exists j\in[K]$ such that $\elecvxY_{ij}\geq 1-\epsilon_2$. Then we have:
		\bas{
			\|\trowS_i-\trowP_{I(j)}\|&\leq \|\trowS_i-\trowP_i\| + \|\trowP_i-\trowP_{I(j)}\|\leq \epsilon + \|r_i\sum_l \elecvxY_{il} \trowP_{I(l)}-r_i\trowP_{I(j)}\|+\|(r_i-1)\trowP_{I(j)}\|\\
			&\leq \epsilon + r_i((1-\elecvxY_{ij})+\|\sum_{l\neq j} \elecvxY_{il} \trowP_{I(l)}\|)+(r_i-1)\\
			&\leq \epsilon + \left(1+ \frac{(\zeta+4)\epsilon}{b-2\epsilon}\right)(2\epsilon_2) + \frac{(\zeta+4)\epsilon}{b-2\epsilon}\tag{by Lemma~\ref{lem:r_i_bound}}\\
			&\leq \left(1+\frac{4\zeta}{b}\right)\epsilon+4\epsilon_2<\frac{c_Y\zeta}{b\lambda_K(\tmtxPp\tmtxPp^T)}\epsilon\leq \frac{c_Y\sqrt{K}\zeta}{(\lambda_K(\tmtxPp\tmtxPp^T))^{1.5}}\epsilon,
		}
		where we use $\epsilon\leq b/(4\zeta)$ and $\zeta\geq4$. And $c_Y$ is a constant. Then $\|\tmtxS_{C}-\bpi\tmtxPp\|_F	\leq \frac{c_Y{K}\zeta}{(\lambda_K(\tmtxPp\tmtxPp^T))^{1.5}}\epsilon$.
	\end{proof}
\end{lem}

\begin{lem}\label{lem:y_error_norm_general}
	Let $\max_i \|\be_i^T(\mtxP-\mtxS)\|=\epsilon_0$, then
	${
		\|\trowP_i-\trowS_i\| \leq \frac{2\epsilon_0}{\|\rowP_i\|}.
	}$
	\begin{proof}
		First note that by definition $\|\|\rowP_i\|-\|\rowS_i\|\|\leq\epsilon_0$, then,
		\bas{
			\|\trowP_i-\trowS_i\|
			&=\left\|\frac{\rowP_i}{\|\rowP_i\|}-\frac{\rowS_i}{\|\rowS_i\|}\right\|
			=\left\|\frac{\|\rowS_i\|\rowP_i-\|\rowP_i\|\rowS_i}{\|\rowP_i\|\|\rowS_i\|}\right\|
			=\left\|\frac{\|\rowS_i\|(\rowP_i-\rowS_i)+(\|\rowS_i\|-\|\rowP_i\|)\rowS_i}{\|\rowP_i\|\|\rowS_i\|}\right\|\\
			&\leq \left\|\frac{\|\rowS_i\|(\rowP_i-\rowS_i)}{\|\rowP_i\|\|\rowS_i\|}\right\|+\left\|\frac{(\|\rowS_i\|-\|\rowP_i\|)\rowS_i}{\|\rowP_i\|\|\rowS_i\|}\right\|
			\leq \left\|\frac{\rowP_i-\rowS_i}{\|\rowP_i\|}\right\|+\left\|\frac{\|\rowS_i\|-\|\rowP_i\|}{\|\rowP_i\|}\right\|
			\leq \frac{2\epsilon_0}{\|\rowP_i\|}.
		}
	\end{proof}
\end{lem}

	\begin{proof}[Proof of Theorem~\ref{thm:M_row_bound}]
		First let us get some important intermediate bounds. Using Weyl's inequality,
		\bas{
			|\sigma_i(\tmtxS_{C})- \sigma_i(\tmtxPp)|&\leq\|\tmtxS_{C}-\bpi\tmtxPp\|\leq \epsilon_4\\
			|\lambda_i(\tmtxS_{C}\tmtxS_{C}^T)- \lambda_i(\tmtxPp\tmtxPp^T)|&= |\sigma_i^2(\tmtxS_{C})- \sigma_i^2(\tmtxPp)| \leq (\sigma_i(\tmtxS_{C})+ \sigma_i(\tmtxPp))\epsilon_4\\
			&\leq (2\sigma_i(\tmtxPp)+\epsilon_4)\epsilon_4.
		}
		Secondly,
		\bas{
			\|(\tmtxS_C\tmtxS_C^T)^{-1}\|=\frac{1}{\lambda_K(\tmtxS_C\tmtxS_C^T)}\leq\frac{1}{\lambda_K(\tmtxPp\tmtxPp^T)-(2\sigma_K(\tmtxPp)+\epsilon_4)\epsilon_4}\leq
			\frac{2}{\lambda_K(\tmtxPp\tmtxPp^T)},
		}
		where we use $(2\sigma_K(\tmtxPp)+\epsilon_4)\epsilon_4<\lambda_K(\tmtxPp\tmtxPp^T)/2$. 
		Then, 
		\bas{
			&\|\bpi(\tmtxPp\tmtxPp^T)^{-1}-(\tmtxS_C\tmtxS_C^T)^{-1}\bpi\|=\|(\bpi\tmtxPp(\bpi\tmtxPp)^T)^{-1}-(\tmtxS_C\tmtxS_C^T)^{-1}\|	\\
			=&\|(\bpi\tmtxPp(\bpi\tmtxPp)^T)^{-1}(\bpi\tmtxPp(\bpi\tmtxPp)^T-\tmtxS_C\tmtxS_C^T)(\tmtxS_C\tmtxS_C^T)^{-1}\|\\
			\leq& \|(\tmtxPp\tmtxPp^T)^{-1}\|\|\bpi\tmtxPp(\bpi\tmtxPp)^T-\tmtxS_C\tmtxS_C^T\|\|(\tmtxS_C\tmtxS_C^T)^{-1}\|\\
			\leq& 2\|(\tmtxPp\tmtxPp^T)^{-1}\|^2(\|\bpi\tmtxPp-\tmtxS_C\|\|(\bpi\tmtxPp)^T\|+\|\tmtxS_C\|\|(\bpi\tmtxPp)^T-\tmtxS_C^T\|)\\
			\leq& 2\|(\tmtxPp\tmtxPp^T)^{-1}\|^2((\|\tmtxPp\|+\|\tmtxS_C\|)\|\tmtxS_{C}-\bpi\tmtxPp\|)\\
			\leq& 2\|(\tmtxPp\tmtxPp^T)^{-1}\|^2(2\|\tmtxPp\|\epsilon_4+\epsilon_4^2).
		}
		Note that $\cvxP=\mtxP\tmtxPp^T(\tmtxPp\tmtxPp^T)^{-1}$. Let $\max_i \|\be_i^T(\mtxP-\mtxS)\|=\epsilon_0$, then,
		\bas{
			&\|\be_i^T(\cvxP - \mtxS\tmtxS_C^T(\tmtxS_C\tmtxS_C^T)^{-1}\bpi)\|=\|\be_i^T(\mtxP\tmtxPp^T(\tmtxPp\tmtxPp^T)^{-1} - \mtxS\tmtxS_C^T(\tmtxS_C\tmtxS_C^T)^{-1}\bpi)\|\\
			=&\|\be_i^T((\mtxP-\mtxS)\tmtxPp^T(\tmtxPp\tmtxPp^T)^{-1})\| +\|\be_i^T(\mtxS(\tmtxPp-\bpi^T\tmtxS_C)^T(\tmtxPp\tmtxPp^T)^{-1})\| \\
			&+ \|\be_i^T(\mtxS\tmtxS_C^T(\bpi(\tmtxPp\tmtxPp^T)^{-1}-(\tmtxS_C\tmtxS_C^T)^{-1}\bpi))\|\\
			\leq& \|\be_i^T(\mtxP-\mtxS)\|\|\tmtxPp\|\|(\tmtxPp\tmtxPp^T)^{-1}\|+\|\be_i^T\mtxS\|\|\tmtxS_C-\bpi\tmtxPp\|\|(\tmtxPp\tmtxPp^T)^{-1}\| \\
			&+ \|\be_i^T\mtxS\|\|\tmtxS_C\|\|\bpi(\tmtxPp\tmtxPp^T)^{-1}-(\tmtxS_C\tmtxS_C^T)^{-1}\bpi\|\\
			\leq& (\|\be_i^T(\mtxP-\mtxS)\|\|\tmtxPp\|+\|\be_i^T\mtxS\|\|\tmtxS_C-\bpi\tmtxPp\|)\|(\tmtxPp\tmtxPp^T)^{-1}\| \\
			&+ 2\|\be_i^T\mtxS\|\|\tmtxS_C\|\|(\tmtxPp\tmtxPp^T)^{-1}\|^2(2\|\tmtxPp\|\epsilon_4+\epsilon_4^2)\\
			\leq& \|(\tmtxPp\tmtxPp^T)^{-1}\|(\|\tmtxPp\|\epsilon_0+13\|\tmtxPp\|^2\|\be_i^T\mtxP\|\|(\tmtxPp\tmtxPp^T)^{-1}\|\epsilon_4)\\
			\leq& \frac{\|\tmtxPp\|\epsilon_0+13\kappa(\tmtxPp\tmtxPp^T)\|\be_i^T\mtxP\|\frac{c_Y{K}\zeta}{(\lambda_K(\tmtxPp\tmtxPp^T))^{1.5}}\epsilon}{\lambda_K(\tmtxPp\tmtxPp^T)} 	\leq \frac{c_M\kappa(\tmtxPp\tmtxPp^T)\|\be_i^T\mtxP\|{K\zeta}}{(\lambda_K(\tmtxPp\tmtxPp^T))^{2.5}} \epsilon:=\epsilon_{M,i}
		}
		where we uses $\epsilon_4\leq \|\tmtxPp\|/2$, $\epsilon_0<\|\be_i^T\mtxP\|\epsilon/2$
		for relaxations.
	\end{proof}

\section{Equivalence of using $\hat{\bV}$ and $\hat{\bV}\hat{\bV}^T$ as input of Algorithm~\ref{algo:inferCone}}
\begin{lem}\label{lem:equivalence}
	For DCMMSB-type models, let $\bu_i=\bU^T\be_i=\bv_i/\|\bv_i\|$, $\trowP_i=\tmtxP^T\be_i=\bV\bv_i/\|\bV\bv_i\|$, $\hat{\bu}_i=\hat{\bU}^T\be_i=\hvv_i/\|\hvv_i\|$, $\trowS_i=\tmtxS^T\be_i=\hv\hvv_i/\|\hv\hvv_i\|$ where $\bV=(\bv_1,\bv_2,\cdots,\bv_n)^T$ and $\hV=(\hvv_1,\hvv_2,\cdots,\hvv_n)^T$ are population and empirical eigenvectors respectively. 
	\ocsvm using rows of $\bU$ (or $\hu$) and rows of $\tmtxP$ (or $\tmtxS$) will return the same solution $\bbeta$. 
	\begin{proof}
		Since $\trowP_i=\bV\bv_i/\|\bV\bv_i\|=\bV\bv_i/\|\bv_i\|=\bV\bu_i$, and $\trowS_i=\hv\hvv_i/\|\hv\hvv_i\|=\hv\hvv_i/\|\hvv_i\|=\hv\hat{\bu}_i$, we have $\trowP_i^T\trowP_j=\bu_i^T\bV^T\bV\bu_j=\bu_i^T\bu_j$ and $\trowS_i^T\trowS_j=\hhu_i^T\hv^T\hv\hhu_j=\hhu_i^T\hhu_j$. It is easy to see that \ocsvm using rows of $\bU$ (or $\hu$) and rows of $\tmtxP$ (or $\tmtxS$) have the same objective function (Eq.~\ref{eq:primal_dual}) and thus will have the same solution of $\beta_i$, $i \in [n]$.
	\end{proof}
\end{lem}

\begin{rem}
	By Lemmas~\ref{lem:equivalence}, \ref{lem:ocsvm_popul}, and Theorem~\ref{thm:dcmmsb_condition_true}, \ocsvm with $\by_i=\bV\bv_i/\|\bV\bv_i\|$, $i\in [n]$ as inputs can find all the $K$ corners corresponding to the pure nodes as support vectors for DCMMSB-type models. Furthermore, as $\hat{\bY}_C=\hat{\bU}_C\hat{\bV}^T$,
	\bas{
		\hat{\bM} =\mtxS\tmtxS_C^T(\tmtxS_C\tmtxS_C^T)^{-1}=\hat{\bV}\hat{\bV}^T\hat{\bV}\hat{\bU}_C^T(\hat{\bU}_C\hat{\bV}^T\hat{\bV}\hat{\bU}_C)^{-1}
		=\hat{\bV}\hat{\bU}_C^T(\hat{\bU}_C\hat{\bU}_C)^{-1},
	}
	which shows that outputs of Algorithm~\ref{algo:inferCone} using $\hat{\bV}$ and $\hat{\bV}\hat{\bV}^T$ as input are same.
\end{rem}

\section{DCMMSB-type models properties}
\begin{lem}\label{lem:v_norm_bound}
	For DCMMSB-type models, if $\|\btheta_i\|_p=1$, 
	for $p=1$ (DCMMSB) or $p=2$ (OCCAM), 
	then we have $\gamma_i/\sqrt{\lambda_1(\bTheta^T\bGamma^2 \bTheta)}\leq\|\bv_i\|\leq\mypsi\gamma_i/\sqrt{\lambda_K(\bTheta^T\bGamma^2 \bTheta)}$, and $\gamma_i/\sqrt{\lambda_1(\bTheta^T\bGamma^2 \bTheta)}\leq\|\bv_i\|\leq \gamma_i/\sqrt{\lambda_K(\bTheta^T\bGamma^2 \bTheta)}$, $\forall i\in I$. 
	\begin{proof}
		Eq.~\eqref{eq:original_vpvpt} gives $((\bGamma_P^{-1}\bV_P)(\bGamma_P^{-1}\bV_P)^T)^{-1}=\bTheta^T\bGamma^2 \bTheta$, 
		then,
		\bas{
			\max_i\|\be_i^T(\bGamma_P^{-1}\bV_P)\|^2
			&=\max_i\be_i^T(\bGamma_P^{-1}\bV_P)(\bGamma_P^{-1}\bV_P)^T\be_i\leq\max_{\|\bx\|=1}\bx^T(\bGamma_P^{-1}\bV_P)(\bGamma_P^{-1}\bV_P)^T\bx\\
			&=\lambda_1((\bGamma_P^{-1}\bV_P)(\bGamma_P^{-1}\bV_P)^T)=1/\lambda_K(\bTheta^T\bGamma^2 \bTheta)\\
			\min_i\|\be_i^T(\bGamma_P^{-1}\bV_P)\|^2
			&=\min_i\be_i^T(\bGamma_P^{-1}\bV_P)(\bGamma_P^{-1}\bV_P)^T\be_i\geq\min_{\|\bx\|=1}\bx^T(\bGamma_P^{-1}\bV_P)(\bGamma_P^{-1}\bV_P)^T\bx\\
			&=\lambda_K((\bGamma_P^{-1}\bV_P)(\bGamma_P^{-1}\bV_P)^T)=1/\lambda_1(\bTheta^T\bGamma^2 \bTheta).
		}
		By Lemma~\ref{lem:vtheta}, $\forall i\in [n]$, if $\|\btheta_i\|_p=1$,  for $p=1$ or $2$,
		\bas{
		 	\|\bv_i\|=\gamma_i\btheta_i^T\bGamma_P^{-1}\bV_P\leq
		 	\gamma_i\max_i\|\btheta_i\|\|\bGamma_P^{-1}\bV_P\|\leq\mypsi\gamma_i/\sqrt{\lambda_K(\bTheta^T\bGamma^2 \bTheta)},
	 	}
 		where we use $\|\btheta_i\|\leq \|\btheta_i\|_p=1$ for $0<p\leq 2$.
		Similarly, 
		\bas{
			\|\bv_i\|\geq\gamma_i\min_i\|\theta_i\|_1\min_i\|\be_i(\bGamma_P^{-1}\bV_P)\|\geq\gamma_i/\sqrt{\lambda_1(\bTheta^T\bGamma^2 \bTheta)}.
		}
		Note that if $\|\btheta_i\|_p=1$, as $\|\btheta_i\|_2\leq K^{1/2-1/p}\|\btheta_i\|_p=K^{1/2-1/p}$, 
		for models with $p>2$, we need to add a model specifically parameter $\psi=K^{1/2-1/p}$ to the upper bound of $\|\bv_i\|$. For simplicity we omit this and only consider cases when $0<p\leq 2$.
	\end{proof}
\end{lem}

\begin{lem}\label{lem:yp_eigen}
	For DCMMSB-type models whose eigenvectors has the form in Lemma~\ref{lem:vtheta}, if using $\mtxP=\bV\bV^T$, $\cvxP=\bGamma\bTheta\bGamma_P^{-1}\bN_P^{-1}$, then:
	\bas{
		\lambda_1(\tmtxPp\tmtxPp^T) \leq \kappa(\bTheta^T\bGamma^2 \bTheta), \ 
		\lambda_K(\tmtxPp\tmtxPp^T) \geq 1/\kappa(\bTheta^T\bGamma^2 \bTheta), \   \text{and} \  \kappa(\tmtxPp\tmtxPp^T)\leq (\mypsi\kappa(\bTheta^T\bGamma^2 \bTheta))^2.
	}
	\begin{proof}
		For DCMMSB-type models, we have $\bV=\bGamma\bTheta\bGamma_P^{-1}\bV_P$, and $(\bV_P\bV_P^T)^{-1}=\bGamma_P^{-1}\bTheta^T\bGamma^2 \bTheta \bGamma_P^{-1}$ by Lemma~\ref{lem:vtheta} and Theorem~\ref{thm:dcmmsb_condition_true} (Eq.~\ref{eq:original_vpvpt}). Note that $\tmtxPp=\bN_P\mtxPp$, then we have 
		\bas{
			\lambda_1(\tmtxPp\tmtxPp^T)&=\lambda_1(\bN_P\mtxPp\mtxPp^T\bN_P)=\lambda_1(\bN_P\bV_P\bV_P^T\bN_P)=\lambda_1(\bN_P\bGamma_P (\bTheta^T\bGamma^2 \bTheta)^{-1} \bGamma_P\bN_P) \\
			&\leq (\lambda_1(\bN_P\bGamma_P))^2\lambda_1((\bTheta^T\bGamma^2 \bTheta)^{-1})\leq(\max_{i\in I} \gamma_i/\|\bv_i\|)^2/\lambda_K(\bTheta^T\bGamma^2 \bTheta)\\
			&\leq \lambda_1(\bTheta^T\bGamma^2 \bTheta)/\lambda_K(\bTheta^T\bGamma^2 \bTheta)=\kappa(\bTheta^T\bGamma^2 \bTheta) \tag{by proof of Lemma~\ref{lem:v_norm_bound}}
		}
		Note that $\bN_{ii}=1/\|\be_i^T\bZ\|=1/\|\be_i^T\bV\bV^T\|=1/\|\be_i^T\bV\|$. Similarly, we have:
		\bas{
			\lambda_K(\tmtxPp\tmtxPp^T)&=\lambda_K(\bN_P\bGamma_P (\bTheta^T\bGamma^2 \bTheta)^{-1} \bGamma_P\bN_P) 
			\geq (\lambda_K(\bN_P\bGamma_P))^2\lambda_K((\bTheta^T\bGamma^2 \bTheta)^{-1})\\
			&\geq(\min_{i\in I} \gamma_i/\|\bv_i\|)^2/\lambda_1(\bTheta^T\bGamma^2 \bTheta)
			\geq \lambda_K(\bTheta^T\bGamma^2 \bTheta)/\lambda_1(\bTheta^T\bGamma^2 \bTheta)\\
			&=1/\kappa(\bTheta^T\bGamma^2 \bTheta) \tag{by proof of Lemma~\ref{lem:v_norm_bound}}
		}
		And finally we have,
		\bas{
			\kappa(\tmtxPp\tmtxPp^T)\leq (\kappa(\bTheta^T\bGamma^2 \bTheta))^2.
		}
	\end{proof}
\end{lem}

\begin{lem}\label{lem:y_error_norm_dcmm}
	For DCMMSB-type models, let $\bv_i=\bV^T\be_i$, $\hvv_i=\hv^T\be_i$, $\bz_i=\bV\bv_i$, $\hat{\bz}_i=\hv\hvv_i$, $\trowP_i=\bV\bv_i/\|\bV\bv_i\|$, and $\trowS_i=\hv\hvv_i/\|\hv\hvv_i\|$, $i\in [n]$. Also let $\epsilon_0=\max_i\|\bz_i-\hat{\bz}_i\|$, then,
	\bas{
		\|\trowP_i-\trowS_i\| \leq \frac{2\epsilon_0}{\|\bv_i\|}\leq \frac{2\epsilon_0 \sqrt{\lambda_1(\bTheta^T\bGamma^2 \bTheta)}}{\gamma_i}.
	}
	\begin{proof}
		From Lemma~\ref{lem:y_error_norm_general}, we have
		\bas{
			\|\trowP_i-\trowS_i\|
			&\leq \frac{2\epsilon_0}{\|\bV\bv_i\|} = \frac{2\epsilon_0}{\|\bv_i\|}\leq \frac{2\epsilon_0 \sqrt{\lambda_1(\bTheta^T\bGamma^2 \bTheta)}}{\gamma_i},
		}
		where the last step uses Lemma~\ref{lem:v_norm_bound}. 
	\end{proof}
\end{lem}

\newpage
\begin{lem}\label{lem:dcmmsb_lapmad_P}
	For DCMMSB-type models, $\lambda^*(\bP)\geq \rho\lambda^*(\bB)\lambda_K(\bTheta^T\Gamma^2\bTheta)$.
	\begin{proof}
		Let $\bX=\bB\bTheta^T\bGamma^2\bTheta\bB$, it easy to see that $\bX$ is full rank and positive definite, then
		\bas{
			\lambda^*(\bP)&=	\rho\lambda^*(\Gamma\bTheta\bB\bTheta^T\bGamma)
			=\rho\sqrt{\lambda_K(\Gamma\bTheta\bB\bTheta^T\bGamma^2\bTheta\bB\bTheta^T\Gamma)}
			=\rho\sqrt{\lambda_K(\Gamma\bTheta\bX\bTheta^T\Gamma)}\\
			&= \rho\sqrt{\lambda_K(\bX^{1/2}\bTheta^T\Gamma^2\bTheta\bX^{1/2})}
			=\rho\sqrt{\lambda_K(\bX\bTheta^T\Gamma^2\bTheta)}
			\geq \rho\sqrt{\lambda_K(\bX)\lambda_K(\bTheta^T\Gamma^2\bTheta)}\\
			&\geq \rho\sqrt{(\lambda_K(\bB))^2(\lambda_K(\bTheta^T\Gamma^2\bTheta))^2}
			=\rho\lambda^*(\bB)\lambda_K(\bTheta^T\Gamma^2\bTheta),
		}
		where we use that $\bL\bL^T$ and $\bL^T\bL$ have the same leading $K$ eigenvalues for a matrix $\bL\in\R^{n\times K}$ with rank $K<n$.
	\end{proof}
\end{lem}

\section{DCMMSB error bounds}
\begin{lem}\label{lem:kappa_theta}
	For DCMMSB-type models, if $\btheta_i\sim\mathrm{Dirichlet}(\balpha)$, let $\alpha_0=\bone_K^T\balpha$, $\alpha_{\max}=\max_i \alpha_i$, $\alpha_{\min}=\min \alpha_i$, $\nu=\alpha_{0}/\alpha_{\min}$, then
	\bas{
		\uP\left(\lambda_1({\bTheta^T\bGamma^2\bTheta})\leq \frac{3\gamma_{\max}^2n\bbb{\alpha_{\max}+\|\balpha\|^2}}{2\alpha_0(1+\alpha_0)}\right)&\geq 1-K\exp\bbb{-\frac{n}{36\nu^2(1+\alpha_0)^2}}\\
		\uP\left(\lambda_K({\bTheta^T\bGamma^2\bTheta})\geq \frac{\gamma_{\min}^2n}{2\nu(1+\alpha_0)}\right)&\geq 1-K\exp\bbb{-\frac{n}{36\nu^2(1+\alpha_0)^2}}\\
		\uP\left(\kappa(\bTheta^T\bGamma^2 \bTheta)\leq 3\frac{\gamma_{\max}^2}{\gamma_{\min}^2}\frac{\alpha_{\max}+\|\balpha\|^2}{\alpha_{\min}}\right) &\geq 1-2K\exp\left(-\frac{n}{36\nu^2(1+\alpha_0)^2}\right)\\
		\uP\left(\lambda^*(\bP) \geq \frac{\gamma_{\min}^2\lambda^*(\bB)}{2\nu(1+\alpha_0)}\rho n\right) &\geq 1-K\exp\left(-\frac{n}{36\nu^2(1+\alpha_0)^2}\right)
	}
	where $\lambda^*(\bP)$ is the $K$-th singular value of $\bP$.
	\begin{proof}
		First note that
		\bas{
			\lambda_1(\bTheta^T\bGamma^2 \bTheta)=\lambda_1(\bGamma\bTheta\bTheta^T\bGamma)\leq (\lambda_1(\bGamma))^2\lambda_1(\bTheta\bTheta^T)=(\lambda_1(\bGamma))^2\lambda_1(\bTheta^T\bTheta).
		}
		Here we use that $\bX\bX^T$ and $\bX^T\bX$ have the same leading $K$ eigenvalues for $\bX\in\R^{n\times K}$ with rank $K<n$. Also, as $\bTheta^T(\bGamma^2-\gamma_{\min}^2\bI) \bTheta$ is positive semidefinite, we have 
		\bas{
			\lambda_K(\bTheta^T\bGamma^2 \bTheta)&=\lambda_K(\bTheta^T(\bGamma^2-\gamma_{\min}^2\bI) \bTheta+\gamma_{\min}^2\bTheta^T\bTheta)
			\geq \lambda_K(\bTheta^T(\bGamma^2-\gamma_{\min}^2\bI)\bTheta)+ \lambda_K(\gamma_{\min}^2\bTheta^T\bTheta)\\
			&\geq  \gamma_{\min}^2 \lambda_K(\bTheta^T\bTheta)		
		}
		By Lemma~A.2 of \citep{mao2017estimating}, 
		\bas{
			\uP\left(\lambda_1({\bTheta^T\bTheta})\leq \frac{3n\bbb{\alpha_{\max}+\|\balpha\|^2}}{2\alpha_0(1+\alpha_0)}\right)&\geq 1-K\exp\bbb{-\frac{n}{36\nu^2(1+\alpha_0)^2}}\\
			\uP\left(\lambda_K({\bTheta^T\bTheta})\geq \frac{n}{2\nu(1+\alpha_0)}\right)&\geq 1-K\exp\bbb{-\frac{n}{36\nu^2(1+\alpha_0)^2}}\\
			\uP\left(\kappa(\bTheta^T\bTheta)  \leq {3\frac{\alpha_{\max}+\|\balpha\|^2}{\alpha_{\min}}}\right)&\geq 1-2K\exp\left(-\frac{n}{36\nu^2(1+\alpha_0)^2}\right)
		}
		So $\kappa(\bTheta^T\bGamma^2 \bTheta)=\frac{\lambda_1(\bTheta^T\bGamma^2 \bTheta)}{\lambda_K(\bTheta^T\bGamma^2 \bTheta)}\leq \frac{\gamma_{\max}^2}{\gamma_{\min}^2}\kappa(\bTheta^T \bTheta)\leq {3\frac{\gamma_{\max}^2}{\gamma_{\min}^2}\frac{\alpha_{\max}+\|\balpha\|^2}{\alpha_{\min}}}$ with high probability.
		Using Lemma~\ref{lem:dcmmsb_lapmad_P}, we have,
		\bas{
			\lambda^*(\bP)
			\geq \rho\lambda^*(\bB)\lambda_K(\bTheta^T\Gamma^2\bTheta)
			\geq  \frac{\gamma_{\min}^2\lambda^*(\bB)}{2\nu(1+\alpha_0)}\rho n,
		}
		with probability at least $1-K\exp\bbb{-\frac{n}{36\nu^2(1+\alpha_0)^2}}$.
	\end{proof}
\end{lem}

\begin{lem}\label{lem:dcmmsb}
	\sloppy
	For DCMMSB-type models, we have $(\tmtxPp\tmtxPp^T)^{-1}\bone\geq\frac{(\min_i\gamma_i)^2}{{\lambda_1(\bTheta^T\bGamma^2 \bTheta)}}\bTheta^T\bone$. Furthermore, if $\btheta_i\sim\mathrm{Dirichlet}(\balpha)$, then ${(\tmtxPp\tmtxPp^T)^{-1}\bone \geq \frac{(\min_i\gamma_i)^2}{{2\lambda_1(\bTheta^T\bGamma^2 \bTheta)}}\frac{n}{\nu}\bone\geq\frac{\gamma_{\min}^2}{3\gamma_{\max}^2}\frac{1}{\nu}\bone}$ with probability larger than $1-1/n^3-K\exp\bbb{-\frac{n}{36\nu^2(1+\alpha_0)^2}}$,  where $\nu = \frac{\sum \alpha_i}{\min \alpha_i}$.
	\fussy
	\begin{proof}
		First note that, for diagonal matrices $\bD\in \mathbb{R}_{\geq 0}^{m\times m}$ and $\bGamma\in \mathbb{R}_{\geq 0}^{n\times n}$ that have strictly positive elements on the diagonal, and some matrices $\bG\in \mathbb{R}_{\geq 0}^{m\times m}$ and $\bH_1\in \mathbb{R}_{\geq 0}^{n\times m}$, $\bH_2\in \mathbb{R}_{\geq 0}^{n\times m}$ we have 
		\ba{
			\bD\bG\bD \bone &\geq (\min_i \bD_{ii})^2 \bG\bone, \label{eq:D_out}\\
			\bH_1^T\bGamma\bH_2 \bone &\geq \min_i \bGamma_{ii} \bH_1^T\bH_2\bone. \label{eq:D_in}
		}
		Eq.~\eqref{eq:D_out} is true because
		\bas{
			\bD\bG\bD \bone- (\min_i \bD_{ii})^2 \bG\bone
			=&\bD\bG\bD \bone-\min_i \bD_{ii} \bG\bD\bone +\min_i \bD_{ii} \bG\bD\bone- (\min_i \bD_{ii})^2 \bG\bone\\
			=&(\bD-\min_i \bD_{ii}\bI)\bG\bD \bone+\min_i \bD_{ii} \bG(\bD-\min_i \bD_{ii}\bI)\bone\geq \bzero,
		}
		where last step follows that $\bD$, 
		$\bG$ and $(\bD-\min_i \bD_{ii}\bI)$ are all non-negative. Eq.~\eqref{eq:D_in} can be proved in a similar way. Now use these on Eq.~\eqref{eq:vpvpT}, we have
		\bas{
			(\tmtxPp\tmtxPp^T)^{-1}\bone&=\bN_P^{-1}\bGamma_P^{-1}\bTheta^T\bGamma^2 \bTheta \bGamma_P^{-1}\bN_P^{-1}\bone\geq\left(\min_i \frac{\|\bv_{I(i)}\|}{\gamma_{I(i)}}\right)^2\bTheta^T\bGamma^2 \bTheta\bone\\
			&\geq \left(\min_i \frac{\|\bv_{I(i)}\|}{\gamma_{I(i)}}\right)^2 (\min_i\gamma_i)^2\bTheta^T\bone
			\geq \frac{(\min_i\gamma_i)^2}{{\lambda_1(\bTheta^T\bGamma^2 \bTheta)}}\bTheta^T\bone,
		}
		where the last step follows Lemma~\ref{lem:v_norm_bound}. 
		By Lemma~C.1. of \citep{mao2017}, we know if rows of $\bTheta$ are from Dirichlet distribution with parameter $\balpha=(\alpha_0,\alpha_2,\cdots,\alpha_K)$, $\alpha_0=\sum_i \alpha_i$,  $\nu=\alpha_0/\min_i\alpha_i$,
		\bas{
			\bTheta^T\bone \geq \frac{n}{\nu}\left(1-O_P\left(\sqrt{\frac{\nu\log n}{n}}\right)\right)\bone
		}
		with probability larger than $1-1/n^3$. Now 
		by Lemma~\ref{lem:kappa_theta}, we have, with probability larger than $1-1/n^3-K\exp\bbb{-\frac{n}{36\nu^2(1+\alpha_0)^2}}$,
		\bas{
			(\tmtxPp\tmtxPp^T)^{-1}\bone
			&\geq \frac{(\min_i\gamma_i)^2}{{\lambda_1(\bTheta^T\bGamma^2 \bTheta)}}\bTheta^T\bone
			\geq \frac{(\min_i\gamma_i)^2}{{\lambda_1(\bTheta^T\bGamma^2 \bTheta)}}\frac{n}{\nu}\left(1-O_P\left(\sqrt{\frac{\nu\log n}{n}}\right)\right)\bone\\
			&\geq \frac{2\gamma_{\min}^2\alpha_0(1+\alpha_0)}{3\gamma_{\max}^2n\bbb{\alpha_{\max}+\|\balpha\|^2}}\frac{n}{\nu}\frac{1}{2}\bone
			=\frac{\gamma_{\min}^2\alpha_{\min}(1+\alpha_0)}{3\gamma_{\max}^2\bbb{\alpha_{\max}+\|\balpha\|^2}}\bone
			\geq \frac{\gamma_{\min}^2}{3\gamma_{\max}^2}\frac{1}{\nu}\bone.
		}
	\end{proof}
\end{lem}\bk
We use a crucial result from \citep{mao2017estimating} that shows row-wise eigenspace concentration for general low rank matrix.
\begin{thm}[Row-wise eigenspace concentration~\citep{mao2017estimating}]\label{thm:row_wise_vvt}
	Suppose $\bP$ has rank $K$, ${\max_{i,j}\bP_{ij}\leq\rho}$. Let $\bA_{ij}=\bA_{ji}\sim \mathrm{Ber}(\bP_{ij})$, $\bV$ and $\hat{\bV}$ are $\bP$ and $\bA$'s top-$K$ eigenvectors respectively. If  $\uP(\max_i\|\bV_{:,i}\|_\infty>\sqrt{\rho})\leq\delta_1$, and for some constant $\xi>1$, $\rho n=\Omega((\log n)^{2\xi})$ and ${\uP(\lambda^*(\bP)<4\sqrt{n\rho}(\log n)^\xi)<\delta_2}$, 
	then for a  fixed $i \in [n]$, with probability at least $1-\delta_1-\delta_2-O(Kn^{-3})$, 
	\bas{
		\| \be_i^T( \hat{\bV}\hat{\bV}^T-\bV\bV^T) \|
		&=O\bbb{\frac{\min\{K,\kappa(\bP)\}\sqrt{Kn\rho}}{\lambda^*(\bP)}}\bbb{(\min\{K,\kappa(\bP)\}+(\log n)^{\xi})\max_i\|\bV_{:,i}\|_\infty+(K+1)n^{-2\xi}}.
	}
\end{thm}

	\begin{proof}[Proof of Theorem~\ref{thm:row_wise_vvt1}]
		First by Lemma~\ref{lem:y_error_norm_dcmm}, 
		\bas{
			\|\trowP_i-\trowS_i\|
			&\leq \frac{2\epsilon_0}{\|\bv_i\|}\leq \frac{2\epsilon_0 \sqrt{\lambda_1(\bTheta^T\bGamma^2 \bTheta)}}{\gamma_i}.
		}
	Also using Lemma~\ref{lem:v_norm_bound},
	\bas{
		\max_j \|\bV_{:,j}\|_\infty \leq \max_i \|\bv_i\|  \leq\mypsi  \max_i \gamma_i/\sqrt{\lambda_K(\bTheta^T\bGamma^2 \bTheta)},
	}
	By Lemma~\ref{lem:kappa_theta}, we have $\max_j \|\bV_{:,j}\|_\infty \leq\mypsi  \max_i \gamma_i/\sqrt{\lambda_K(\bTheta^T\bGamma^2 \bTheta)}$ with probability at least $1-\delta_1$ for $\delta_1\leq K\exp\bbb{-\frac{n}{36\nu^2(1+\alpha_0)^2}}$. Also from the condition of $\nu$, $\max_i \gamma_i/\sqrt{\lambda_K(\bTheta^T\bGamma^2 \bTheta)}\leq \sqrt{\rho}$. Then it is easy to see $\uP(\max_i\|\bV_{:,i}\|_\infty>\sqrt{\rho})\leq\delta_1$.
	Also, from the condition of $\lambda^*(\bB)/\nu$, we have $4\sqrt{n\rho}(\log n)^\xi\leq \frac{\gamma_{\min}^2\lambda^*(\bB)}{2\nu(1+\alpha_0)}\rho n$. Then combined with Lemma~\ref{lem:kappa_theta}, 
	$\uP(\lambda^*(\bP)<4\sqrt{n\rho}(\log n)^\xi)<\delta_2$ is satisfied with $\delta_2\leq K\exp\left(-\frac{n}{36\nu^2(1+\alpha_0)^2}\right)$.
 	Also we have, $$\max_i \gamma_i/\sqrt{\lambda_K(\bTheta^T\bGamma^2 \bTheta)}\geq \max_i \gamma_i/\sqrt{\lambda_1(\bTheta^T\bGamma^2 \bTheta)}\geq\sqrt{2/(3n)}\gg (K+1)n^{-2\xi}$$ with high probability. 
	Then by Theorem~\ref{thm:row_wise_vvt} we have 
	\bas{
		\epsilon_0&=O\bbb{\frac{\min\{K,\kappa(\bP)\}\sqrt{Kn\rho}}{\lambda^*(\bP)}}\bbb{(\min\{K,\kappa(\bP)\}+(\log n)^{\xi})\max_i\|\bV_{:,i}\|_\infty+(K+1)n^{-2\xi}}\\
		&=\tilde{O}\bbb{\frac{\min\{K^2,(\kappa(\bP))^2\}\sqrt{Kn\rho}}{\rho\lambda^*(\bB)\lambda_K(\bTheta^T\bGamma^2\bTheta)}}\frac{\mypsi\gamma_{\max}}{\sqrt{\lambda_K(\bTheta^T\bGamma^2 \bTheta)}}.
	}
	with probability at least $1-\delta_1-\delta_2-O(Kn^{-3})= 1-O(Kn^{-3})$. 
	So,
	\bas{
		\|\trowP_i-\trowS_i\| \leq \frac{2\epsilon_0 \sqrt{\lambda_1(\bTheta^T\bGamma^2 \bTheta)}}{\gamma_i}
		&=\tilde{O}\bbb{\frac{\min\{K^2,(\kappa(\bP))^2\}\sqrt{Kn}}{\sqrt{\rho}\lambda^*(\bB)\lambda_K(\bTheta^T\bGamma^2\bTheta)}}\frac{\mypsi\gamma_{\max}\sqrt{\kappa(\bTheta^T\bGamma^2 \bTheta)}}{\gamma_i}.
	}
	And using Lemma~\ref{lem:kappa_theta},
	\bas{
		\epsilon=\max_i \|\trowP_i-\trowS_i\| 
		&=\yError\\
		&= \tilde{O}\bbb{\frac{\mypsi\gamma_{\max}\min\{K^2,(\kappa(\bP))^2\}\sqrt{\kappa(\bTheta^T\bGamma^2 \bTheta)}K^{0.5}\nu(1+\alpha_0)}{\gamma_{\min}^3\lambda^*(\bB)\sqrt{n\rho}}}
	}
with probability at least $1-O(Kn^{-2})$.
	\end{proof}

	\begin{proof}[Proof of Theorem~\ref{thm:theta_bound}]
		Note that $\bP=\rho\bGamma\bTheta\bB\bTheta^T\bGamma=\bV\bE\bV^T$, we have $\rho\bGamma_P\bB\bGamma_P=\bV_P\bE\bV_P^T$, then $\rho\bN_P\bGamma_P\bB\bGamma_P\bN_P=\bN_P\bV_P\bE\bV_P^T\bN_P=\tmtxPp\bV\bE\bV^T\tmtxPp^T$. As $\bB$ has unit diagonal, let $\bB(i,i)=c^2$, then $c^2\rho\gamma_{I(i)}^2/\|\bv_{I(i)}\|^2=\rho\be_i^T\bN_P\bGamma_P\bB\bGamma_P\bN_P\be_i=\be_i^T\tmtxPp\bV\bE\bV^T\tmtxPp^T\be_i:=d_i^2$. Since our estimation for $c^2\rho\gamma_{I(i)}^2/\|\bv_{I(i)}\|^2$ is $\be_i^T\bpi^T\tmtxS_C\hv\hat{\bE}\hv^T\tmtxS_C^T\bpi\be_i$, and note that $\|\bE\|\leq \max_i \|\be_i^T\bP\|_1=O(\rho n)$, $\|\hat{\bE}\|\leq \|\bE\|+\|\bA-\bP\|=O(\rho n)$ using Weyl's inequality and Theorem~5.2 of \citep{lei2015consistency}, and $\|\bV\bE\bV^T-\hv\hat{\bE}\hv^T\|\leq \lambda_{K+1}(\bA)+\|\bP-\bA\|\leq 2\|\bP-\bA\|=O(\sqrt{\rho n})$. Let   $\hat{d}_i^2=\be_i^T\tmtxS_C\hv\hat{\bE}\hv^T\tmtxS_C^T\be_i$, then we have, 
		\bas{
			&|d_i^2-\hat{d}_{\pi(i)}^2|=\|\be_i^T\tmtxPp\bV\bE\bV^T\tmtxPp^T\be_i-\be_i^T\bpi^T\tmtxS_C\hv\hat{\bE}\hv^T\tmtxS_C^T\bpi\be_i\|\\
			\leq &\|\be_i^T(\tmtxPp-\bpi^T\tmtxS_C)\bV\bE\bV^T\tmtxPp^T\be_i\|+\|\be_i^T\bpi^T\tmtxS_C(\bV\bE\bV^T-\hv\hat{\bE}\hv^T)\tmtxPp^T\be_i\|\\
			&+\|\be_i^T\bpi^T\tmtxS_C\hv\hat{\bE}\hv^T(\tmtxPp^T-\tmtxS_C^T\bpi)\be_i\|\\
			\leq & \|\be_i^T(\tmtxPp-\bpi^T\tmtxS_C)\|\|\bE\|+\|\bV\bE\bV^T-\hv\hat{\bE}\hv^T\|+\|\hat{\bE}\|\|\be_i^T(\tmtxPp-\bpi^T\tmtxS_C)\|\\
			\leq & O(\rho n)\epsilon_4/\sqrt{K}+O(\sqrt{\rho n}).
		}
		Using Lemma~\ref{lem:v_norm_bound}, $c\sqrt{\rho\lambda_K(\bTheta^T\bGamma^2 \bTheta)} \leq d_i\leq c\sqrt{\rho\lambda_1(\bTheta^T\bGamma^2 \bTheta)}$, and by Lemma~\ref{lem:kappa_theta}, $\lambda_1(\bTheta^T\bGamma^2\bTheta)\leq \frac{3\gamma_{\max}^2 n\bbb{\alpha_{\max}+\|\balpha\|^2}}{2\alpha_0(1+\alpha_0)}$, $\lambda_K({\bTheta^T\bGamma^2\bTheta})\geq \frac{\gamma_{\min}^2n}{2\nu(1+\alpha_0)}$, then we have $d_i\geq {c}\sqrt{\frac{\gamma_{\min}^2\rho n}{2\nu(1+\alpha_0)}}$, and $d_i\leq c\sqrt{\frac{3\gamma_{\max}^2 \rho n\bbb{\alpha_{\max}+\|\balpha\|^2}}{2\alpha_0(1+\alpha_0)}}$ with probability at least $1-2K\exp\left(-\frac{n}{36\nu^2(1+\alpha_0)^2}\right)$. Then, using Lemma~\ref{lem:yp_eigen},
		\bas{
			|d_i-\hat{d}_{\pi(i)}|&\leq \frac{O(\rho n)\epsilon_4/\sqrt{K}+O(\sqrt{\rho n})}{\min_j (d_j+\hat{d}_{\pi(j)})}
			\leq \frac{O(\rho n)\epsilon_4/\sqrt{K}+O(\sqrt{\rho n})}{\sqrt{\rho\lambda_K(\bTheta^T\bGamma^2 \bTheta)}}	\\
			&\leq \frac{O(\rho n/\sqrt{K})\frac{c_Y{K}\zeta}{(\lambda_K(\tmtxPp\tmtxPp^T))^{1.5}}\epsilon+O(\sqrt{\rho n})}{\sqrt{\rho\lambda_K(\bTheta^T\bGamma^2 \bTheta)}}
			=O\bbb{\frac{K^{0.5}(\kappa(\bTheta^T\bGamma^2 \bTheta))^{1.5}\zeta \sqrt{\rho}n}{\sqrt{\lambda_K(\bTheta^T\bGamma^2 \bTheta)}}\epsilon}\\
			&= O\bbb{\frac{K^{1.5}(\kappa(\bTheta^T\bGamma^2 \bTheta))^{3} \sqrt{\rho}n}{\eta\sqrt{\lambda_K(\bTheta^T\bGamma^2 \bTheta)}}\epsilon}.
		}
		Let $\bD=\diag({d_1},{d_2},\cdots,{d_K})$ and  $\hat{\bD}=\diag({\hat{d}_1},{\hat{d}_2},\cdots,{\hat{d}_K})$, then $\bD={c}\sqrt{\rho}(\bN_P\bGamma_P)$.
		Now as we estimate ${c}\sqrt{\rho}(\bGamma\bTheta)$ by ${\hat{c}}\sqrt{\hat{\rho}}\hat{\bGamma}\hat{\bTheta}=\hat{\cvxP}\hat{\bD}$, we have
		\bas{
			&\|\be_i^T({c}\sqrt{\rho}\bGamma\bTheta-{\hat{c}}\sqrt{\hat{\rho}}\hat{\bGamma}\hat{\bTheta}\bpi)\|= \|\be_i^T({\cvxP}\bD-\hat{\cvxP}\hat{\bD}\bpi)\|
			\leq  \|\be_i^T(\cvxP-\cvxS\bpi)\bD\|+\|\be_i^T\cvxS\bpi(\bD-\bpi^T\hat{\bD}\bpi)\| \\
			\leq & \|\be_i^T(\cvxP-\cvxS\bpi)\|\|\bD\|+\|\be_i^T\cvxS\|\|\bD-\bpi^T\hat{\bD}\bpi\|
			\leq  \epsilon_{M,i} \max_j d_j + (\|\be_i^T\cvxP\|+\epsilon_{M,i})\max_j|d_j-\hat{d}_{\pi(j)}|\\
			\leq &\epsilon_{M,i} \max_j d_j + ( {\gamma_i}\max_{j\in I}{\|\bv_{j}\|}/{\gamma_{j}}+\epsilon_{M,i})\max_j|d_j-\hat{d}_{\pi(j)}|\\
			\leq & c\sqrt{\rho\lambda_1(\bTheta^T\bGamma^2 \bTheta)}\epsilon_{M,i}+\left(\frac{\gamma_i}{\sqrt{\lambda_K(\bTheta^T\bGamma^2 \bTheta)}}+\epsilon_{M,i}\right)O\bbb{\frac{K^{1.5}(\kappa(\bTheta^T\bGamma^2 \bTheta))^{3} \sqrt{\rho}n}{\eta\sqrt{\lambda_K(\bTheta^T\bGamma^2 \bTheta)}}\epsilon},
		}
		where we use $\|\be_i^T\bM\|=\|\be_i^T\bGamma\bTheta\bGamma_P^{-1}\bN_P^{-1}\|\leq \gamma_i\|\btheta_i\|\max_{j\in I}{\|\bv_{j}\|}/{\gamma_{j}}$ and $\|\btheta_i\|\leq 1$ for DCMMSB and OCCAM for the last inequality.
		As
		\bas{
			\epsilon_{M,i}&=\frac{c_M\kappa(\tmtxPp\tmtxPp^T)\|\be_i^T\mtxP\|{K\zeta}}{(\lambda_K(\tmtxPp\tmtxPp^T))^{2.5}} \epsilon \leq \frac{c_M\kappa(\tmtxPp\tmtxPp^T)\|\be_i^T\mtxP\|{K}}{(\lambda_K(\tmtxPp\tmtxPp^T))^{2.5}}\frac{4K}{\eta(\lambda_K(\tmtxP_P\tmtxP_P^T))^{1.5}}\epsilon\\
			&\leq \frac{c_1\|\be_i^T\mtxP\|(\kappa(\bTheta^T\bGamma^2 \bTheta))^{6}{K^2}}{\eta}\epsilon
			\leq \frac{c_1\mypsi \gamma_{i} (\kappa(\bTheta^T\bGamma^2 \bTheta))^{6}{K^2}}{\eta\sqrt{\lambda_K(\bTheta^T\bGamma^2 \bTheta)}}\epsilon.
		}
		Then
		\bas{
			&\epsilon_5=\|\be_i^T({c}\sqrt{\rho}\bGamma\bTheta-{\hat{c}}\sqrt{\hat{\rho}}\hat{\bGamma}\hat{\bTheta}\bpi)\|\\
			=&c\sqrt{\rho\lambda_1(\bTheta^T\bGamma^2 \bTheta)}\epsilon_{M,i}+\left(\frac{\gamma_i}{\sqrt{\lambda_K(\bTheta^T\bGamma^2 \bTheta)}}+\epsilon_{M,i}\right)O\bbb{\frac{K^{1.5}(\kappa(\bTheta^T\bGamma^2 \bTheta))^{3} \sqrt{\rho}n}{\eta\sqrt{\lambda_K(\bTheta^T\bGamma^2 \bTheta)}}\epsilon}\\
			=& c\sqrt{\rho\lambda_1(\bTheta^T\bGamma^2 \bTheta)}\frac{c_1\mypsi\gamma_{i}(\kappa(\bTheta^T\bGamma^2 \bTheta))^{6}{K^2}}{\eta\sqrt{\lambda_K(\bTheta^T\bGamma^2 \bTheta)}}\epsilon
			+  O\bbb{\gamma_i\frac{K^{1.5}(\kappa(\bTheta^T\bGamma^2 \bTheta))^{3} \sqrt{\rho}n}{\eta\lambda_K(\bTheta^T\bGamma^2 \bTheta)}\epsilon}\\
			=& O\bbb{\max\left\{{\mypsi K^{0.5}(\kappa(\bTheta^T\bGamma^2 \bTheta))^{3.5}},\frac{n}{\lambda_K(\bTheta^T\bGamma^2 \bTheta)}\right\}\frac{\gamma_{i}K^{1.5}(\kappa(\bTheta^T\bGamma^2 \bTheta))^{3} \sqrt{\rho}}{\eta}\epsilon}.
		}
		As $|c\sqrt{\rho}\gamma_i-\hat{c}\sqrt{\hat{\rho}}\hat{\gamma}_i|=\|\be_i^T({c}\sqrt{\rho}\bGamma\bTheta-{\hat{c}}\sqrt{\hat{\rho}}\hat{\bGamma}\hat{\bTheta}\bpi)\bone\|
		\leq\mypsi\sqrt{K}\epsilon_5$, let $\bX_i=\be_i^T{c}\sqrt{\rho}\bGamma\bTheta$ and $\hat{\bX}_i=\be_i^T{\hat{c}}\sqrt{\hat{\rho}}\hat{\bGamma}\hat{\bTheta}\bpi$, then for DCMMSB, $\|{\bX}_i\|_1=c\sqrt{\rho}\gamma_{i}$, $\|\hat{\bX}_i\|\leq \|\hat{\bX}_i\|_1= {\hat{c}}\sqrt{\hat{\rho}}\hat{\gamma}_i$; for OCCAM, $\|{\bX}_i\|=c\sqrt{\rho}\gamma_{i}\|\be_i^T\bTheta\|=c\sqrt{\rho}\gamma_{i}$ and $\|\hat{\bX}_i\|= {\hat{c}}\sqrt{\hat{\rho}}\hat{\gamma}_i$. So we have,
		\bas{
				&\|\be_i^T(\bTheta-\hat{\bTheta}\bpi)\|=\left\|\frac{\bX_i}{c\sqrt{\rho}\gamma_i}-\frac{\hat{\bX}_i}{\hat{c}\sqrt{\hat{\rho}}\hat{\gamma}_i}\right\|\leq 
				\frac{\|\bX_i-\hat{\bX}_i\|}{c\sqrt{\rho}\gamma_i} + \left\|\frac{c\sqrt{\rho}\gamma_i-{\hat{c}}\sqrt{\hat{\rho}}\hat{\gamma}_i}{c\sqrt{\rho}\gamma_i{\hat{c}}\sqrt{\hat{\rho}}\hat{\gamma}_i}\right\|\|\hat{\bX}_i\|\\
				\leq& \frac{\epsilon_5}{c\sqrt{\rho}\gamma_i} + \frac{\mypsi\sqrt{K}}{c\sqrt{\rho}\gamma_i}\epsilon_5
				=O\bbb{\frac{\mypsi\sqrt{K}}{\gamma_{i}\sqrt{\rho}}\epsilon_5}\\
				=& O\bbb{\max\left\{{\mypsi K^{0.5}(\kappa(\bTheta^T\bGamma^2 \bTheta))^{3.5}},\frac{n}{\lambda_K(\bTheta^T\bGamma^2 \bTheta)}\right\}\frac{\mypsi K^{2}(\kappa(\bTheta^T\bGamma^2 \bTheta))^{3} }{\eta}\epsilon}\\
				=& \tilde{O}\bbb{\max\left\{{\mypsi K^{0.5}(\kappa(\bTheta^T\bGamma^2 \bTheta))^{3.5}},\frac{n}{\lambda_K(\bTheta^T\bGamma^2 \bTheta)}\right\}\frac{\gamma_{\max} K^{2.5}\min\{K^2,(\kappa(\bP))^2\}(\kappa(\bTheta^T\bGamma^2 \bTheta))^{3.5} \sqrt{n}}{\gamma_{\min}\eta\lambda^*(\bB)\lambda_K(\bTheta^T\bGamma^2\bTheta)\sqrt{\rho}}}\\
				=& \tilde{O}\bbb{\frac{\gamma_{\max} K^{2.5}\min\{K^2,(\kappa(\bP))^2\} n^{3/2}}{\gamma_{\min}\eta\lambda^*(\bB)\lambda_K^2(\bTheta^T\bGamma^2\bTheta)\sqrt{\rho}}} \tag{when $\kappa(\bTheta^T\bGamma^2 \bTheta)=\Theta(1)$}.
		}
		Note that this bound works for both DCMMSB and OCCAM, and $\lambda_K(\bTheta^T\bGamma^2\bTheta)=\Omega(n)$, so the bound is about $\tilde{O}\bbb{1/\sqrt{\rho n}}$. specifically, for DCMMSB,
		\bas{
			\|\be_i^T(\bTheta-\hat{\bTheta}\bpi)\|
			=& \ThetaErrorDCMMSBtype\\ 
			=& \tilde{O}\bbb{\frac{\gamma_{\max} K^{2.5}\min\{K^2,(\kappa(\bP))^2\} \nu^2(1+\alpha_0)^2 }{\gamma_{\min}^5\eta\lambda^*(\bB)\sqrt{\rho n}}}.
		}
	\end{proof}

\section{Topic model error bounds}
\subsection{Eigenspcae concentration for topic models}
 Consider the following setup similar to~\cite{arora12computing}.
\ba{\label{eq:topic-model}
\w_{ij}\stackrel{iid}{\sim} \mathrm{Binomial}(N,\W_{ij})\qquad \mbox{For $i\in [V], j\in [D]$}
}
Here $\W$ is the probability matrix for words appearing in documents. Furthermore, we have $\W=\bT\bH$, where $\bT$ is the word to topic probabilities {with columns summing to 1} and $\bH$ is the topic to document matrix {with columns summing to 1}. Also note that, 
{$\sum_{i} \|\be^T_i\W\W^T\|_1=D$, }
since the columns of $\W$ sum to one. 
We will construct a matrix $\w_1\w_2^T$, where $\w_1$ and $\w_2$ are obtained by dividing the words in each document uniformly randomly in two equal parts. For simplicity denote $N_1=N/2$.
\bk
Consider the matrix $\bU=\frac{\w_1\w_2^T}{N_1^2}$. We have $\uE[\bU]=\W\W^T$.

\begin{lem}
	For topic models, we have $(\tmtxPp\tmtxPp^T)^{-1}\bone\geq \frac{\min_i\|\be_i^T\bT\|_1}{{\lambda_1(\bT^T \bT)}}\bone \geq \frac{\min_i\|\be_i^T\bT\|_1}{{K}}\bone$, where $\bT$ is the word-topic probability matrix.
	\begin{proof}
		Noting that $\bT=\bGamma\bTheta$ for topic models, where $\gamma_i=\bGamma_{ii}=\|\be_i^T\bT\|_1$. Following the steps of Lemma~\ref{lem:dcmmsb}, we find
		\bas{
			(\tmtxPp\tmtxPp^T)^{-1}\bone
			\geq\frac{\min_i\gamma_i}{{\lambda_1(\bTheta^T\bGamma^2 \bTheta)}}(\bGamma\bTheta)^T\bone
			=\frac{\min_i\gamma_i}{{\lambda_1(\bT^T \bT)}}\bT^T\bone
			=\frac{\min_i\gamma_i}{{\lambda_1(\bT^T \bT)}}\bone
			\geq \frac{\min_i\|\be_i^T\bT\|_1}{{K}}\bone,
		}
		where the last step is true because $\lambda_1(\bT^T \bT)\leq\trace(\bT^T \bT)=\sum_{i}\|\bT\be_i\|^2\leq K$.

		So $\eta\geq\frac{\min_i\|\be_i^T\bT\|_1}{{\lambda_1(\bT^T \bT)}}\geq\frac{\min_i\|\be_i^T\bT\|_1}{{K}}$.
	\end{proof}
\end{lem}
\bk
\begin{lem}\label{topic-frob}
	 Using Eq~\eqref{eq:topic-model}, we see that under Assumption~\ref{assumption:topic}, 
	 \bas{
		\uP\left(\|\bU-\W\W^T\|_F \geq \sqrt{\frac{50D\log \gikD }{N_1}}\right)\leq \frac{2}{(\gikD) ^3}. 
 	}
\end{lem}
\begin{proof}
	Recall that from Assumption~\ref{assumption:topic}, $g_{ik}=\be_i^T\W\W^Te_k$. Let $\bR:=\bU-\W\W^T$.
	\bas{
		\bR_{ik}= \frac{\sum_{j=1}^D \w_1(ij)\w_2(kj)}{N_1^2}-g_{ik}
}
Note that $\uE[\bR_{ij}]=0$, and $\bA_1(ij)\bA_2(kj)/N_1^2$ is bounded by 1. Also $\bA_1(i,j)$ and $\bA_2(i,j)$ are independent. For independent $X:=\w_1(ij)/N_1$, $Y:=\w_1(kj)/N_1$,
\bas{
	\var(XY)&=\var(X)\var(Y)+\var(X)\uE[X]^2+\var(Y)\uE[Y]^2\leq \frac{3\W_1(ij)\W_2(kj)}{N_1}\\
	\var(\bR_{ik})&\leq 3g_{ik}/N_1
}
 When $g_{ik}=0$, $\bU_{ik}=0$. When $g_{ik}>0$, using Bernstein's inequality, we have:
\bas{
\uP\left(|\bR_{ik}|\geq t_{ik}\right)\leq 2\exp\left(-\frac{t_{ik}^2}{2(3g_{ik}/N_1+t_{ik}/3)}\right)
,}
Setting, $t_{ik}=\sqrt{50\log \gikD g_{ik}/N_1}$, we see that,
\bas{
\sum_{i,k}t_{ik}^2 =50\log \gikD\sum_{ik}g_{ik}/N_1=50\log \gikD D/N_1
}
Then, 
\bas{
	\uP\left(\|\bR\|_F^2 \geq \sum_{i,k}t_{ik}^2 \right) \leq V^2 \max_{i,k} \uP\left(|\bR_{ik}|\geq t_{ik}\right) \leq 2V^2/(\gikD)^5 \leq 2/(\gikD)^3.
}
This yields the result.
\end{proof}
\newpage
\begin{lem}\label{lem:topic-op}
	Using Eq~\eqref{eq:topic-model}, we see that, under Assumption~\ref{assumption:topic}, there exists constants $C,r$ such that, 
	$$\uP\left(\|\bU-\W\W^T\|\geq C_r\sqrt{\frac{D\log \gikD}{N}}\right) \leq \frac{2}{(\gikD)^r}.$$
\end{lem}
\begin{proof}
	We use the Matrix Bernstein bound in~\cite{Tropp:2015}. Let $\bS_k:=\dfrac{\w_{1k}\w_{2k}^T}{N_1^2}-\W_k\W_k^T$, where $\bM_k$ is the $k^{th}$ column of matrix $\bM$. Note that $\uE[\bS_k]$ is the $V\times V$ zeros matrix.  We also see that by symmetry of the random splitting, $\uE[\bS_k\bS_k^T]=\uE[\bS_k^T\bS_k]$. 
	
	We will now note some theoretical properties of the $\bS_k$ matrices. Let $\bX$ be a vector of size $V$, such that, $\bX_i\sim \mathrm{Binomial}(N_1,a_i)$. 
	\ba{
		\frac{\uE[\bX^T\bX]}{N_1^2}&=\sum_{i=1}^V\frac{\uE[\bX_i^2]}{N_1^2}=\sum_{i=1}^V\frac{\uE[\bX_i]^2+\var(\bX_i)}{N_1^2}\nonumber\\
		&=\sum_{i=1}^V\frac{N_1^2 a_i^2+N_1a_i(1-a_i)}{N_1^2}=\left(1-\frac{1}{N_1}\right) \|a\|^2+\frac{1}{N_1}\label{eq:e}
}
	Furthermore, let
	\ba{\label{eq:cov}
		\cov(\bX)=\bSigma,
	\qquad \bSigma_{ij}=N_1a_i(1-a_i)1(i=j)
}
	Then,
	\bas{
		\uE[\bS_k\bS_k^T]&=\uE\left[\frac{\w_{1k}\w_{2k}^T\w_{2k}\w_{1k}^T}{N_1^4}-\W_k\W_k^T\W_k\W_k^T\right]\\
	\mbox{(By independence)}\qquad	&=\frac{\uE[\w_{2k}^T\w_{2k}]\uE[\w_{1k}\w_{1k}^T]}{N_1^4}-\|\W_k\|^2 \W_k\W_k^T\\
	\mbox{(By Eq~\eqref{eq:e} and~\eqref{eq:cov})}\qquad&=\left(\frac{1}{N_1}+\|\W_k\|^2(1-\frac{1}{N_1})\right)\left(\frac{\bSigma_k}{N_1^2}+\W_k\W_k^T\right)-\|\W_k\|^2 \W_k\W_k^T\\
	&=\|\W_k\|^2\frac{\bSigma_k}{N_1^2}+\frac{1-\|\W_k\|^2}{N_1}\left(\frac{\bSigma_k}{N_1^2}+\W_k\W_k^T\right)
}
Since $\|\bSigma_k\|\leq N_1\|\W_k\|_1= N_1$, $\|\W\|_F^2\leq D$,
\bas{
v(\bS)=\left\|\sum_k \uE[\bS_k\bS_k^T]\right\|\leq 2\frac{\|\W\|_F^2}{N_1}+\frac{D}{N_1^2}\leq \frac{D}{N_1}\bbb{2+\frac{1}{N_1}}.
}
Furthermore, 
\bas{
\|\bS_k\|\leq \|\W_k\|^2+\frac{\|\w_{1k}\|\|\w_{2k}\|}{N_1^2}\leq 2=:L
}

So the Matrix Bernstein bound gives us:
\bas{
\uP\left(\|\sum_k \bS_k\|\geq t\right)&\leq 2V\exp\left(-\frac{t^2/2}{v(\bS)+Lt/3}\right)=2V\exp\left(-\frac{t^2/2}{3D/N_1+2t/3}\right)
}
Using $t=C_r\sqrt{D\log \gikD/N}$, and using the condition in Assumption~\ref{assumption:topic}, we get the bound.
	\end{proof}

\begin{proof}[Proof of Lemma~\ref{lem:topic_eigen_bound}]
	First note the proof is under Assumption~\ref{assumption:topic}. Let $\bR=\bU-\W\W^T$.
	Using the Davis-Kahan Theorem~\cite{yu2015useful}, we see that there exists an orthogonal matrix $\bO$:
	$$\|\hat{\bV}\bO-\bV\|_F\leq \frac{\sqrt{8}(2\lambda_1(\W\W^T)+\|\bR\|_2)\min(\sqrt{K}\|\bR\|_2,\|\bR\|_F)}{\lambda_K^2(\W\W^T)},$$
	where $\lambda_1$ and $\lambda_K$ are the largest and $K^{th}$ largest singular values (and also eigenvalue) of $\W\W^T$ respectively.
	Thus, 
	\bas{
		\|\hat{\bV}\bO-\bV\|_F&\leq \frac{\sqrt{8}(2\lambda_1(\W\W^T)+\|\bR\|_2)\min(\sqrt{K}\|\bR\|_2,\|\bR\|_F)}{\lambda_K^2(\W\W^T)}\\
		&\leq \sqrt{8}\frac{2\lambda_1(\W\W^T)+C_r\sqrt{D\log \gikD/N}}{\lambda_K(\W\W^T)^2}\sqrt{\frac{D\log \gikD}{N}}\max\left( C_r\sqrt{K},\sqrt{C}\right)\\
		&\leq\frac{\lambda_1(\bH\bH^T)\lambda_1(\bT^T\bT)}{\lambda_K^2(\bH\bH^T)\lambda_K^2(\bT^T\bT)}O_P\bbb{\sqrt{\frac{KD\log\gikD}{N}}}\\
		&=\frac{\kappa(\bH\bH^T)\kappa(\bT^T\bT)}{\lambda_K(\bT^T\bT)}O_P\left(\sqrt{\frac{K\log \gikD}{DN}}\right),
}
where the third inequality follows Lemma~\ref{lem:dcmmsb_lapmad_P} with $\bP=\W\W^T$, $\bGamma\bTheta=\bT$, $\bB=\bH\bH^T$ and $\rho=1$.
Now we bound $\epsilon_0=\max_i\|\bz_i-\hat{\bz}_i\|=\|\be_i^T(\hat{\bV}\hat{\bV}^T-\bV\bV^T)\|$ as:
\bas{
\|\be_i^T(\hat{\bV}\hat{\bV}^T-\bV\bV^T)\|&\leq \|\hat{\bV}\hat{\bV}^T-\bV\bV^T\|_2
\leq  \|(\hat{\bV}\bO-\bV)\bO^T\hat{\bV}^T+\bV(\hat{\bV}\bO-\bV)^T\|\\
&=\topicRowwiseEigenspaceBound.
}
By Lemma~\ref{lem:y_error_norm_dcmm}, $\|\trowP_i-\trowS_i\| \leq \frac{2\epsilon_0}{\|\bv_i\|}\leq \frac{2\epsilon_0 \sqrt{\lambda_1(\bT^T\bT)}}{\|\be_i^T\bT\|_1}.$ So,
\bas{
	\epsilon=\max_i \|\trowP_i-\trowS_i\| = \frac{\kappa(\bH\bH^T)(\kappa(\bT^T\bT))^{1.5}}{\min_j \|\be_j^T\bT\|_1\sqrt{\lambda_K(\bT^T\bT)}}O_P\left(\sqrt{\frac{K\log \gikD}{DN}}\right).
}
\end{proof}

\subsection{Parameter estimation for topic models}
	\begin{proof}[Proof of Theorem~\ref{thm:topic_error_bound}]
		For topic models, $\cvxP=\bT\tmDiagP$, where $\bT=\bGamma\bTheta$, $\tmDiagP=(\bN_P\bGamma_P)^{-1}$, $\gamma_i=\bGamma_{ii}=\|\be_i^T\bT\|_1$. For empirical estimation we have $\cvxS=\hT\tmDiagS$, where $\tmDiagS(i,i)=\|\cvxS(:,i)\|_1$.
		First we have $\forall i \in K$, $\|\bT(:,i)\|_1=1$, then $\|\cvxP(:,i)\|_1=\tmDiagP(i,i)=\|\bv_{I(i)}\|/\gamma_{I(i)}$. Let $\pi$ be the permutation function for permutation matrix $\bpi$  in Theorem~\ref{thm:M_row_bound}, then,
		\bas{
			|\tmDiagP(i,i)-\tmDiagS(\pi(i),\pi(i))|&=|\|\cvxP(:,i)\|_1-\|\cvxS(:,\pi(i))\|_1|
			\leq \|\cvxP(:,i)-\cvxS(:,\pi(i))\|_1\\
			&=\sum_{j=1}^{V} |\cvxP(j,i)-\cvxS(j,\pi(i))|
			\leq \sum_{j=1}^{V} \|\cvxP(j,:)-\cvxS(j,:)\bpi\|_1 \\
			&=\sum_{j=1}^{V} \|\be_j^T\bT\|_1 \frac{\|\cvxP(j,:)-\cvxS(j,:)\bpi\|_1}{\|\be_j^T\bT\|_1}\\
			&\leq K\max_j \frac{\|\cvxP(j,:)-\cvxS(j,:)\bpi\|_1}{\|\be_j^T\bT\|_1}
			\leq K^{1.5}\max_j\frac{\|\cvxP(j,:)-\cvxS(j,:)\bpi\|}{\|\be_j^T\bT\|_1}\\
			&\leq \frac{K^{1.5}\max_j\epsilon_{M,j}}{\min_j \|\be_j^T\bT\|_1}:=\epsilon_D
		}
		Note that $\bT^T\bT=\bTheta^T\bGamma^2 \bTheta$, and 
		from Lemma~\ref{lem:v_norm_bound}, we know
		\bas{
			1/\sqrt{\lambda_1(\bT^T\bT)}\leq\|\bv_i\|/\gamma_i\leq1/\sqrt{\lambda_K(\bT^T\bT)},\  \forall i \in [n]
		} 
		Using Lemma~\ref{lem:yp_eigen}, we have $\lambda_1(\tmtxPp\tmtxPp^T)\leq \kappa(\bTheta\bGamma^2\bTheta)=\kappa(\bT^T\bT)$, $\lambda_K(\tmtxPp\tmtxPp^T)\geq 1/\kappa(\bTheta\bGamma^2\bTheta)=1/\kappa(\bT^T\bT)$, and $\kappa(\tmtxPp\tmtxPp^T)\leq (\kappa(\bTheta\bGamma^2\bTheta))=(\kappa(\bT^T\bT))^2$.
		
		Then the error for each row of $\bT$ is
		\bas{
			\|\be_i^T(\hT-\bT\bpi^T)\|&=\|\be_i^T(\cvxS\tmDiagS^{-1}-\cvxP\tmDiagP^{-1}\bpi^T)\|\\
			&\leq\|\be_i^T(\cvxS-\cvxP\bpi^T)\tmDiagS^{-1}\|+\|\be_i^T\cvxP\bpi^T(\tmDiagS^{-1}-\bpi\tmDiagP^{-1}\bpi^T)\|\\
			&\leq \|\be_i^T(\cvxS-\cvxP\bpi^T)\| \max_j 1/\tmDiagS(j,j)+\|\be_i^T\cvxP\|\max_j \left\|\frac{\tmDiagP(j,j)-\tmDiagS(\pi(j),\pi(j))}{\tmDiagP(j,j)\tmDiagS(\pi(j),\pi(j))} \right\|\\
			&\leq \frac{2\epsilon_{M,i}}{\min_j \tmDiagP(j,j)} + \frac{2\epsilon_D}{(\min_j \tmDiagP(j,j))^2}\|\be_i^T\cvxP\|\\
			&\leq \frac{2\epsilon_{M,i}}{\min_j \|\bv_{I(j)}\|/\gamma_{I(j)}}+\frac{2\epsilon_D}{\min_j \|\bv_{I(j)}\|/\gamma_{I(j)}}\frac{\max_j \tmDiagP(j,j)\|\be_i^T \bT\|}{\min_j \tmDiagP(j,j)}\\
			&\leq 2\sqrt{\lambda_1(\bT^T\bT)}{\epsilon_{M,i}}+2\sqrt{\lambda_1(\bT^T\bT)}\frac{\sqrt{\lambda_1(\bT^T\bT)}}{\sqrt{\lambda_K(\bT^T\bT)}}\|\be_i^T \bT\|\frac{K^{1.5}\max_j\epsilon_{M,j}}{\min_j \|\be_j^T\bT\|_1} \\
			&\leq 4\sqrt{\lambda_1(\bT^T\bT)}\sqrt{\kappa(\bT^T\bT)}\frac{\|\be_i^T \bT\|}{\min_j \|\be_j^T\bT\|_1}K^{1.5}\frac{c_M\kappa(\tmtxPp\tmtxPp^T)\max_j\|\be_j^T\mtxP\|{K\zeta}}{(\lambda_K(\tmtxPp\tmtxPp^T))^{2.5}} \epsilon\\
			&\leq c_1\sqrt{\lambda_1(\bT^T\bT)}(\kappa(\bT^T\bT))^{5.5}K^{2.5}\frac{\gamma_{\max}\|\be_i^T \bT\|}{\min_j \|\be_j^T\bT\|_1}\zeta\epsilon\\
			&\leq \frac{c_2\sqrt{\lambda_1(\bT^T\bT)}(\kappa(\bT^T\bT))^{7}K^{3.5}}{\eta}\frac{\max_j \|\be_j^T \bT\|_1}{\min_j \|\be_j^T\bT\|_1}\|\be_i^T \bT\|\epsilon, \tag{using $\zeta\leq \frac{4K}{\eta(\lambda_K(\tmtxP_P\tmtxP_P^T))^{1.5}}$}
		}
		where we use 
		$\epsilon_D \leq \tmDiagP(j,j)/2$ for relaxation in the 3rd 
		inequality and $c_1$ and $c_2$ are some constants.
		Under Assumption~\ref{assumption:topic}, by Lemma~\ref{lem:topic_eigen_bound}, we have
		\bas{
			\epsilon=\max_i \|\trowP_i-\trowS_i\| = \frac{\kappa(\bH\bH^T)(\kappa(\bT^T\bT))^{1.5}}{\min_j \|\be_j^T\bT\|_1\sqrt{\lambda_K(\bT^T\bT)}}O_P\left(\sqrt{\frac{K\log \gikD }{DN}}\right).
		}
		Then,
		\bas{
			&\frac{\|\be_i^T(\hT-\bT\bpi^T)\|}{\|\be_i^T\bT\|}\leq \frac{c_2\sqrt{\lambda_1(\bT^T\bT)}(\kappa(\bT^T\bT))^{7}K^{3.5}}{\eta}\frac{\max_j \|\be_j^T \bT\|_1}{\min_j \|\be_j^T\bT\|_1} \epsilon\\
			=& \frac{\sqrt{\lambda_1(\bT^T\bT)}(\kappa(\bT^T\bT))^{7}K^{3.5}}{\eta}\frac{\max_j \|\be_j^T \bT\|_1}{\min_j \|\be_j^T\bT\|_1} \frac{\kappa(\bH\bH^T)(\kappa(\bT^T\bT))^{1.5}}{\min_j \|\be_j^T\bT\|_1\sqrt{\lambda_K(\bT^T\bT)}}O_P\left(\sqrt{\frac{K\log \gikD }{DN}}\right)\\
			=& \frac{\max_j \|\be_j^T \bT\|_1}{(\min_j \|\be_j^T\bT\|_1)^2}\frac{\kappa(\bH\bH^T)(\kappa(\bT^T\bT))^{9}}{\eta}O_P\left(K^4\sqrt{\frac{\log \gikD }{DN}}\right)\\
			=& \topicTbounds \tag*{(if $\kappa(\bT^T\bT)=\Theta(1)$ and $\kappa(\bH\bH^T)=\Theta(1)$)}
		}
	\end{proof}

\section{Converting SBMO to DCMMSB}
Since for stochastic blockmodel with overlaps (SBMO) \cite{kaufmann2016spectral}, $\bP=\rho\bZ\bB\bZ$, where rows of $\bZ$ are binary assignments to different communities, we have $\bP=\rho\bZ\bB\bZ=\rho'\bGamma\bTheta\bB\bTheta\bGamma$, where $\gamma_i'=\|\be_i^T\bZ\|_1\in [K]$, $\btheta_i=\be_i^T\bZ/\|\be_i^T\bZ\|_1$, $\gamma_i$ is normalized from $\gamma_i'$ to sum to $n$ for identifiability and $\rho'=(\rho\sum \gamma_i'/n)^2$. We can see each SBMO model is corresponding to an identifiable DCMMSB model, thus we can use \svmcone to recover SBMO model. The way to get binary assignment can be easily done by setting threshold as $1/K$ for each element in $\hat{\bTheta}$.

\section{Closed form rate for known special cases}  
For a Stochastic Blockmodel (SBM) with $K=2$ classes of equal size and standard parameters ($\rho=p$, $\bB_{11}=\bB_{22}=1,\bB_{12}=\bB_{21}=q/p$), our result suggests that as long as $(p-q)/\sqrt{p}=\tilde{\Omega}(1/\sqrt{n})$, SVM-cone will consistently estimate the label of each node uniformly with probability tending to one. This is similar to separation conditions in existing literature for consistent estimation in SBMs, up-to a log factor. \bk
\newpage
\section{Statistics of topic modeling datasets}
\begin{table}[!htbp]
	\vspace{-0.2in}
	\caption{Statistics of topic modeling datasets}
	\vspace{-0.2in} 
	\label{table:topic_statistics}
	\begin{center}
		\begin{tabular}{ |c|c|c|c|c| } 
			\hline
			\small  Corpus & Vocabulary size $V$ & Number of documents $D$ & Total number of words\\ 
			\hline
			\small NIPS\footnotemark[1] & 5002 & 1,491 & 1,589,280\\ 
			\hline
			\small NYTimes\footnotemark[1] & 5004 & 296,784 & 68,876,786\\ 
			\hline
			\small PubMed\footnotemark[1] & 5001 & 7,829,043 & 485,719,597\\ 
			\hline
			\small 20NG\footnotemark[2] & 5000 & 9,540 & 886,043\\ 
			\hline
			\small Enron\footnotemark[1] & 5003 & 29,823 & 4,963,162\\ 
			\hline
			\small KOS\footnotemark[1] & 5001 & 3,412 & 405,190\\ 
			\hline
		\end{tabular}
	\vspace{-0.3in}
	\end{center}
\end{table}

\footnotetext[1]{\url{https://archive.ics.uci.edu/ml/datasets/Bag+of+Words}}
\footnotetext[2]{\url{http://qwone.com/~jason/20Newsgroups/}}
\section{Topics in real data}
\vspace{-0.15in}
\begin{table}[!ht]
	\caption{Top-10 word of 5 topics for different topic modeling datasets}
		\vspace{-0.2in}
	\begin{center}
		\resizebox{\linewidth}{!}{
			\begin{tabular}{|c|c|}
				\hline
				Corpus  & Top-10 words\\
				\hline
				\hline
				\multirow{5}{*}{NIPS}&  algorithm data problem method parameter point vector distribution error space  \\
				\cline{2-2}
				&  neuron output pattern signal circuit visual synaptic unit layer current  \\
				\cline{2-2}
				&  data unit training output image information object recognition pattern point  \\
				\cline{2-2}
				&  unit hidden output layer weight object pattern visual representation connection  \\ 
				\cline{2-2}
				&  error algorithm training weight data parameter method problem vector classifier  \\
				\hline
				\hline
				\multirow{5}{*}{NYT}& con son solo era mayor zzz$\_$mexico director sin fax sector \\
				\cline{2-2}
				& zzz$\_$bush government school campaign show american member country zzz$\_$united$\_$states law \\
				\cline{2-2}
				&  company companies market stock business billion plan money analyst government \\
				\cline{2-2}
				&  team game season play player games run coach win won\\ 
				\cline{2-2}
				&  file sport zzz$\_$los$\_$angeles notebook internet zzz$\_$calif read output web computer \\
				\hline
				\hline
				\multirow{5}{*}{PubMed}& receptor expression gene binding system function region genes dna mechanism \\
				\cline{2-2}
				& concentration strain gene dna system expression region genes test function \\
				\cline{2-2}
				&   tumor gene expression disease genes lesion mutation region dna clinical \\
				\cline{2-2}
				&  rat concentration plasma day serum animal liver drug response administration\\ 
				\cline{2-2}
				&  children disease clinical year test therapy women system diagnosis drug \\
				\hline
				\hline
				\multirow{5}{*}{20NG}& key government car chip state including information cs number long  \\
				\cline{2-2}
				& god jesus bible question things life christian world christ true  \\
				\cline{2-2}
				&   year michael game team cs games win play including car  \\
				\cline{2-2}
				&   drive mb scsi windows card hard disk dos computer drives \\ 
				\cline{2-2}
				&  windows window dos file files program card fax run win  \\
				\hline
				\hline
				\multirow{5}{*}{Enron}& report status changed payment approved approval amount paid due expense  \\
				\cline{2-2}
				& database error operation perform hourahead data file process start message   \\
				\cline{2-2}
				&  power california customer gas order deal list office forward comment  \\
				\cline{2-2}
				&  message contract corp receive offer free send list received click  \\ 
				\cline{2-2}
				&  hourahead final file hour data price process error detected variances  \\
				\hline
				\hline
				\multirow{5}{*}{KOS}& iraq administration military iraqi president american troops bushs officials soldiers   \\
				\cline{2-2}
				&  voting vote senate polls governor electoral voter media voters primary    \\
				\cline{2-2}
				&   percent senate race elections republican party state voters campaign polls    \\
				\cline{2-2}
				&  senate polls governor electoral primary vote ground races voter contact   \\ 
				\cline{2-2}
				&  dean edwards primary clark gephardt lieberman iowa results polls kucinich  \\
				\hline
			\end{tabular}
		}
	\end{center}
\label{tab:multicol}
\end{table}

\end{document}